%% file: main.tex
\documentclass[12pt]{article}

\RequirePackage{amsthm,amsmath,amsfonts,amssymb}
\RequirePackage[numbers, sort]{natbib}
\RequirePackage[colorlinks,citecolor=blue,urlcolor=blue]{hyperref}
\RequirePackage{graphicx}

\usepackage{url}
\usepackage{booktabs}
\usepackage{nicefrac}
\usepackage{microtype}
\usepackage{tcolorbox}
\usepackage{bookmark}
\usepackage{enumitem}
\usepackage{color}
\usepackage{fullpage}
\usepackage{array}
\usepackage{caption}
\usepackage{subcaption}
\usepackage{tikz}
\usepackage{tablefootnote}
\usepackage{multirow}
\usepackage{xcolor}
\usepackage{mathrsfs}
\usepackage{algorithm}
\usepackage{algorithmic}
\usepackage{bm,todonotes}
\usepackage{mathabx}

\def\r{r}
\def\P{\mathbb{P}}

\def\tod{\overset{d}{\to}}
\def\toas{\overset{a.s.}{\to}}
\newcommand{\tabincell}[2]{\begin{tabular}{@{}#1@{}}#2\end{tabular}}
\definecolor{caribbeangreen}{rgb}{0.0, 0.8, 0.6}

\newtheorem{thm}{Theorem}[section]
\newtheorem{lem}{Lemma}[section]
\newtheorem{cor}{Corollary}[section]

\newtheorem{asmp}{Assumption}[section]

\newtheorem{rem}{Remark}[section]
\newtheorem{example}{Example}[section]

\input{tex/math_commands.tex}

\title{
Towards Theoretical Understandings of Robust Markov Decision Processes: Sample Complexity and Asymptotics
}
\author{
Wenhao Yang\thanks{Academy for Advanced Interdisciplinary Studies, Peking University; email: \texttt{yangwenhaosms@pku.edu.cn}. } \\
\and
Liangyu Zhang\thanks{Academy for Advanced Interdisciplinary Studies, Peking University; email: \texttt{zhangliangyu@pku.edu.cn}. } \\
\and
Zhihua Zhang\thanks{School of Mathematical Sciences, Peking University; email: \texttt{zhzhang@math.pku.edu.cn}. } \\
}

\begin{document}



\maketitle
\begin{abstract}
 In this paper, we study  the non-asymptotic and asymptotic performances of the optimal robust policy and value function of robust Markov Decision Processes (MDPs), where the optimal robust policy and value function are estimated from a generative model.
 While prior work focusing on non-asymptotic performances of robust MDPs is restricted in the setting of the KL uncertainty set and $(s,a)$-rectangular assumption, we improve their results and also consider other uncertainty sets, including the $L_1$ and $\chi^2$ balls. Our results show that when we assume $(s,a)$-rectangular on uncertainty sets, the sample complexity is about $\widetilde{\OM}\left(\frac{|\SM|^2|\AM|}{\varepsilon^2\rho^2(1-\gamma)^4}\right)$. In addition, we  extend our results from the $(s,a)$-rectangular assumption to the $s$-rectangular assumption. In this scenario, the sample complexity varies with the choice of uncertainty sets and is generally larger than the case under the $(s,a)$-rectangular assumption. Moreover, we also show that the optimal robust value function is asymptotically normal with a typical rate $\sqrt{n}$ under the $(s,a)$ and $s$-rectangular assumptions from both theoretical and empirical perspectives.
\end{abstract}




\section{Introduction}
\label{sec: intro}
\input{tex/intro.tex}

\section{Preliminaries}
\label{sec: preli}
\input{tex/preliminary.tex}

\section{Non-asymptotic Results}
\label{sec: gen}
\input{tex/gen.tex}

\section{Asymptotic Results}
\label{sec: asymp}
\input{tex/asymp.tex}

\section{Experiments}
\label{sec: exp}
\input{tex/experiment.tex}

\section{Discussion}
\label{sec: conclusion}
\input{tex/conclusion.tex}

\section{Acknowledgments}
The authors thank Xiang Li and Dachao Lin for a discussion related to DRO and some inequalities.

\bibliographystyle{plainnat}
\bibliography{bib/refer.bib}

\newpage
\appendix
\input{tex/appendix}



\end{document}

%% file: tex/math_commands.tex
\def\1{\bm{1}}

\DeclareMathAlphabet{\mathsfit}{\encodingdefault}{\sfdefault}{m}{sl}
\SetMathAlphabet{\mathsfit}{bold}{\encodingdefault}{\sfdefault}{bx}{n}

\def\M{{\bf M}}

\def\0{{\bf 0}}
\def\1{{\bf 1}}

\def\AM{{\mathcal A}}

\def\DM{{\mathcal D}}

\def\NM{{\mathcal N}}
\def\OM{{\mathcal O}}
\def\PM{{\mathcal P}}
\def\SM{{\mathcal S}}
\def\TM{{\mathcal T}}
\def\UM{{\mathcal U}}
\def\VM{{\mathcal V}}

\def\XM{{\mathcal X}}

\def\RB{{\mathbb R}}
\def\EB{{\mathbb E}}
\def\ZB{{\mathbb Z}}
\def\PB{{\mathbb P}}

\def\argmax{\mathop{\rm argmax}}

\def\diag{\mathrm{diag}}



%% file: tex/intro.tex
Reinforcement Learning (RL) is a machine learning paradigm that addresses sequential decision-making problems in an unknown environment. Unlike the supervised learning scenario in which a labeled training dataset is provided, in RL the agent collects information by interacting with the environment through a course of actions.    
In addition to its success in empirical performance \cite{haarnoja2018soft,mnih2015human,mnih2016asynchronous, silver2016mastering}, several works \cite{jin2018q, jin2020provably, duan2020minimax} provide insightful and solid theoretical understandings of RL. RL is typically formulated as the Markov Decision Processes (MDPs) problem \cite{puterman2014markov}. The difficulty of solving an MDP is 
primarily attributable to the inexact knowledge of the reward $R$ and transition probability $P$. To address the challenge, an alternative approach resorts to offline methods, where the agent only has access to a given explorable dataset generated by a strategy.
Many practical deep RL algorithms employ the offline method and achieve state-of-the-art success empirically \cite{mnih2015human,lillicrap2015continuous, fujimoto2019off}. 
However, it often takes incredibly large datasets to make modern RL algorithms work.
The matter of large sample size greatly hinders the application of RL in areas like policy-making, finance, and healthcare, where it is extremely expensive or even impossible to acquire such a large amount of data.
Recently, there are many works focusing on sample efficiency of offline RL from a theoretical perspective.
Some prior works have provided solid results on model-free offline methods \cite{chen2019information, agarwal2020optimistic, duan2021risk} while others have considered model-based approaches \cite{sidford2018near, xie2019towards, yin2020asymptotically,yin2020near}. 
Through these theoretical efforts, 
sample-efficiently learning a near-optimal policy can be guaranteed, i.e., the sample complexity is polynomial in parameters of the underlying MDPs. 

In reality, sometimes the environment used to generate the offline dataset may be different from the real-world MDPs, resulting in suboptimal performance of the policy obtained by RL algorithms.
A well-known example is \textit{the sim-to-real gap} \cite{peng2018sim, zhao2020sim}, which suggests that an RL-based robot controller trained in a simulated environment may perform poorly in the real world.
A similar phenomenon also occurs in application scenarios such as healthcare and finance problems.
For example, we may seek a dynamic treatment regime that would be deployed in hospital \textit{A} using RL algorithms. However, the only available dataset is collected in hospital \textit{B}.
Naively performing RL algorithms with the given dataset and deploying the resulting regime in hospital \textit{A} may cause bad outcomes.
In addition, \citet{mannor2004bias} also showed that the value function might be sensitive to estimation errors of reward and transition probability, which means a small perturbation of reward and transition probability could incur a significant change in the value function.
Then, robust MDPs \cite{iyengar2005robust,nilim2005robust} have been proposed to handle these issues, where the transition probability is allowed to take values in an uncertainty set (or ambiguity set). In this way, the solution of robust MDPs is less sensitive to model estimation errors with a properly chosen uncertainty set $\widehat{\PM}$.


In order to solve the robust MDP problem efficiently, one commonly makes 
the assumption that the uncertainty set $\widehat{\PM}$ is either $(s,a)$-rectangular or $s$-rectangular \cite{iyengar2005robust,nilim2005robust, wiesemann2013robust}, which stand for  the transition probability $P$ taking values independently for each state-action $(s,a)$ pair or each state $s\in\SM$, respectively. Compared with $(s,a)$-rectangular assumption, $s$-rectangular is a more general assumption to alleviate conservative policies and can provide stronger robustness guarantees \cite{wiesemann2013robust}. 
Without these two assumptions, \citet{wiesemann2013robust}  proved that solving robust MDPs could be NP-hard. 
However, under the $(s,a)$-rectangular or $s$-rectangular assumptions, the near-optimal robust policy and value function can be obtained efficiently. With these assumptions, \citet{iyengar2005robust} and \citet{nilim2005robust} proposed multiple choices of uncertainty sets under rectangular assumptions mentioned above, all of which are specific cases of $f$-divergence balls located around the estimated transition probability, including the $L_1$ distance, $\chi^2$ and KL divergence balls. The most widely studied case is the so-called $L_1$ uncertainty set \cite{petrik2014raam, ghavamzadeh2016safe, ho2020partial, behzadian2021optimizing} because it can be solved by the powerful linear programming methods.

In recent years, many works \cite{lim2013reinforcement, goyal2018robust,ho2020partial} have come up with efficient algorithms to solve robust MDPs, obtaining the optimal robust policy and value function.
However, little theory has been developed on the statistical performances of the optimal robust policy and value function. Specifically, two core questions remain open: (a) How many samples are sufficient to guarantee the accuracy of the robust estimators? (b) Is it possible to make statistical inferences from the robust estimators?
In this paper, we figure out both the finite-sample and asymptotic performances of the optimal robust policy and value function in different scenarios and answer these questions conclusively. Specifically, our non-asymptotic results in Section 3 show that
sample-efficient reinforcement learning is possible in robust MDPs, which breaks the misconception that robust MDPs are exponentially hard in terms of effective horizon $(1-\gamma)^{-1}$ \cite{zhou2021finite}. 
And our asymptotic results in Section 4 allow us to make statistical inferences from the robust estimators.



\subsection{Contributions}
\label{subsec: contribution}


Let ${V}_r^\pi(\mu)$ be the robust value function of policy $\pi$ under uncertainty set $\PM$ (unknown) and initial state distribution $\mu$, and $\widehat{V}_r^\pi(\mu)$ be its empirical version under estimated uncertainty set $\widehat{\PM}$.
We  denote by $\widehat{\pi}^*\in\argmax_\pi \widehat{V}_r^\pi(\mu)$  the optimal robust policy, and by $\widehat{V}_r^*(\mu):=\max_\pi \widehat{V}_r^\pi(\mu)$  the optimal robust value function.  
Rather than providing a new efficient algorithm to solve robust MDPs, we take efforts to study the statistical performances of optimal robust value function $\widehat{V}_r^*(\mu)$ and robust policy $\widehat{\pi}^*$ from both finite-sample and asymptotic perspectives. We mainly consider the frequently used data generating approach (i.e., generative models), from which we are able to estimate the transition probability $\widehat{P}$. Moreover, we consider three different uncertainty sets $\PM$: $L_1$, $\chi^2$, and KL balls under both $(s,a)$ and $s$-rectangular assumptions, which are frequently applied in the field of robust MDPs. Although all of the uncertainty sets can be cast into the family of so-called $f$-divergence uncertainty sets, we find it difficult to analyze their finite-sample performance by a general calculation technique. Thus, we analyze the statistical performance of different settings separately and summarize our results in the following parts. 
For practitioners, our sample complexity results indicate how much data is enough for learning a near-optimal policy in a robust MDP, thus guiding the data-collection process.
Our sample complexity results can also serve as theoretical guarantees for the optimality of the learned policy, i.e. with a fixed dataset we may describe the minimum level of optimality for our learned policy.
In addition, our asymptotic results allows practitioners to make statistical inference for optimal robust value functions.
Here are some take-home messages from our results: 
\begin{enumerate}[label=(\alph*)]
\item Sample-efficient results can be guaranteed with $(s,a)$-rectangular or $s$-rectangular assumptions in robust MDPs (upper bound of finite-sample results);
\item Robust MDPs can have a lower sample complexity than original MDPs when the size of uncertainty set is large (upper and lower bounds of finite-sample results);
\item Robust MDPs under $s$-rectangular assumption require more samples than that with $(s,a)$-rectangular assumption (upper bound of finite-sample results);
\item Statistical inference for optimal robust value function is possible (asymptotic results).
\end{enumerate}


\paragraph*{\textbf{Finite-sample results}}
A key criterion of evaluating the finite-sample performance is the following deviation:
\begin{align}
    \label{eq: gap}
    \max_\pi \; V_r^\pi(\mu)- V_r^{\widehat{\pi}^*}(\mu).
\end{align}
In this paper, we use a uniform convergence analysis to control Eqn.~\eqref{eq: gap}:
\begin{align}
    \label{eq: gap_uniform}
    \max_\pi V_r^\pi(\mu)- V_r^{\widehat{\pi}^*}(\mu)\le 2\sup_{\pi\in\Pi}\left|V_r^\pi(\mu)-\widehat{V}_r^\pi(\mu)\right|.
\end{align}
When the dataset is obtained by a generative model, we present the sample complexity of achieving an $\varepsilon$ deviation bound of Eqn.~\eqref{eq: gap} in different settings in Table~\ref{tab: result}. The overall performance among the different uncertainty sets is nearly the same up to some logarithmic factors in the $(s,a)$-rectangular assumption, which is about $\widetilde{\OM}\left(\frac{|\SM|^2|\AM|}{\varepsilon^2\rho^2(1- \gamma)^4}\right)$\footnote{We use $\widetilde{\OM}(\cdot)$ and $\widetilde{\Omega}(\cdot)$ to hide polylogarithmic factors and universal constants.}. Compared to the most related work \cite{zhou2021finite}, which provided an exponential large sample complexity of robust MDPs, we break the misconception that robust MDPs are exponentially harder than original MDPs in terms of $1/(1-\gamma)$. We leave the detailed discussion of comparison in related work section.

\begin{table}[t!]
    \centering
    \scalebox{1.0}{
    \begin{tabular}{ccc}
    \toprule
    Uncertainty Set & $(s,a)$-rectangular (Theorem~\ref{thm: finite-sa}) & $s$-rectangular (Theorem~\ref{thm: finite-s}) \\
    \toprule
    $L_1$ & $\widetilde{\OM}\left(\frac{|\SM|^2|\AM|(2+\rho)^2}{\varepsilon^2\rho^2(1-\gamma)^4}\right)$ & $\widetilde{\OM}\left(\frac{|\SM|^2|\AM|^2(2+\rho)^2}{\varepsilon^2\rho^2(1-\gamma)^4}\right)$\\

    $\chi^2$ & $\widetilde{\OM}\left(\frac{|\SM|^2|\AM|(1+\rho)^2}{\varepsilon^2(\sqrt{1+\rho}-1)^2(1-\gamma)^4}\right)$ & $\widetilde{\OM}\left(\frac{|\SM|^2|\AM|^3(1+\rho)^2}{\varepsilon^2(\sqrt{1+\rho}-1)^2(1-\gamma)^4}\right)$\\

    KL & $\widetilde{\OM}\left(\frac{|\SM|^2|\AM|}{\varepsilon^2\rho^2\underline{p}^2(1-\gamma)^4}\right)$ & $\widetilde{\OM}\left(\frac{|\SM|^2|\AM|^2}{\varepsilon^2\rho^2\underline{p}^2(1-\gamma)^4}\right)$\\
    \bottomrule
    \end{tabular}
    }
    \caption{The sample complexity of achieving $\varepsilon$ deviation bound~\eqref{eq: gap} in the generative model setting (Theorem~\ref{thm: finite-sa}). Here $|\SM|$ and $|\AM|$ are the sizes of the state space and action space, $\gamma \in (0, 1)$ is a discount factor, $\rho$ represents the size of uncertainty set in Examples~\ref{eg: f-set} and~\ref{eg: f-set-s}, and $\underline{p}=\min_{P^*(s'|s,a)>0}P^*(s'|s,a)$.}
    \label{tab: result}
\end{table}

We also derive sample complexity results under $s$-rectangular assumption, whose theoretical properties are never studied before while it is a significant setting in robust MDPs \cite{wiesemann2013robust}. Notably, the sample complexity would enlarge when we assume the uncertainty sets satisfy the $s$-rectangular assumption in Table~\ref{tab: result}. The main difference is caused by the fact that the optimal robust policy is deterministic \cite{nilim2005robust} in the $(s,a)$-rectangular setting while stochastic \cite{wiesemann2013robust} in the $s$-rectangular setting. 
Thus the uniform bound over the class of all possible policies (including stochastic and deterministic policies) could be worse than that over the class of deterministic policies.

We also extend our analysis from estimation by a generative model to estimation by an offline dataset,
which is generated by a given behavior occupancy measure. As long as the concentrability assumption given in \cite{chen2019information} holds, the result of sample complexity only changes by a factor of the concentrability coefficient, which can be referred to Table~\ref{tab: result-off}.

Lastly, we  show that the sample complexity lower bounds of robust MDPs are $\widetilde{\Omega}\left(\frac{|\SM||\AM|(1-\gamma)}{\varepsilon^2}\notag\right. \\ \left.\min\left\{\frac{1}{(1-\gamma)^4},\frac{1}{\rho^4}\right\}\right)$ for the $L_1$ ball and $\widetilde{\Omega}\left(\frac{|\SM||\AM|}{\varepsilon^2(1-\gamma)^2}\min\left\{\frac{1}{1-\gamma},\frac{1}{\rho}\right\}\right)$ for the $\chi^2$ ball, but the lower bound of the KL uncertainty set is still lack of explicit expression. Both the upper and lower bound results imply that the robust MDPs can have a lower sample complexity than original MDPs with a proper size $\rho$ of uncertainty set.

\begin{table}[t!]
    \centering
    \scalebox{1.0}{
    \begin{tabular}{ccc}
    \toprule
    Uncertainty Set & $(s,a)$-rectangular (Theorem~\ref{thm: finite-sa-off}) & $s$-rectangular (Theorem~\ref{thm: finite-s-off}) \\
    \toprule
    $L_1$ & $\widetilde{\OM}\left(\frac{|\SM|(2+\rho)^2}{\nu_{\min}\varepsilon^2\rho^2(1-\gamma)^4}\right)$ & $\widetilde{\OM}\left(\frac{|\SM||\AM|(2+\rho)^2}{\nu_{\min}\varepsilon^2\rho^2(1-\gamma)^4}\right)$\\
    $\chi^2$ & $\widetilde{\OM}\left(\frac{|\SM|(1+\rho)^2}{\nu_{\min}\varepsilon^2(\sqrt{1+\rho}-1)^2(1-\gamma)^4}\right)$ & $\widetilde{\OM}\left(\frac{|\SM||\AM|^2(1+\rho)^2}{\nu_{\min}\varepsilon^2(\sqrt{1+\rho}-1)^2(1-\gamma)^4}\right)$\\
    KL & $\widetilde{\OM}\left(\frac{|\SM|}{\nu_{\min}\varepsilon^2\rho^2\underline{p}^2(1-\gamma)^4}\right)$ & $\widetilde{\OM}\left(\frac{|\SM||\AM|}{\nu_{\min}\varepsilon^2\rho^2\underline{p}^2(1-\gamma)^4}\right)$\\
    \bottomrule
    \end{tabular}
    }
    \caption{The sample complexity of achieving $\varepsilon$ deviation bound~\eqref{eq: gap} in the offline dataset. 
    Here $\nu_{\min}=\min_{s,a,\nu(s,a)>0}\nu(s,a)$}.
    \label{tab: result-off}
\end{table}

\paragraph*{\textbf{Asymptotic results}}
Indeed, the finite-sample results only imply that $\widehat{V}_r^*(\mu)$ is $\widetilde{O}_P(1/\sqrt{n})$, where a logarithmic factor of $n$ exists. It is not sufficient to guarantee the convergence rate of $\widehat{V}_r^*(\mu)$ to be $1/\sqrt{n}$. 
Thus, statistical inference from finite-sample results is inaccurate and we need more precise asymptotic results for more accurate statistical inference.
Our another contribution is showing that $\widehat{V}_r^*(\mu)$ is $\sqrt{n}$-consistent and also asymptotically normal, and then we can derive statistical inference from data directly. We believe our asymptotic results are novel and may open a new approach to statistical inference in robust MDPs.

\paragraph*{\textbf{Empirical studies}}
Finally, we evaluate our theoretical results on simulation experiments. Under the $(s,a)$-rectangular assumption, we follow the classical algorithm Robust Value Iteration \cite{iyengar2005robust}. Under the $s$-rectangular assumption, which is usually more difficult to solve, Bisection Algorithm \cite{ho2018fast} is applied to obtain the near-optimal robust value function. In both settings, our empirical results show that the performance of the near-optimal robust value function is highly correlated with the number of generative samples. In a large sample regime, we also find that the empirical coverage rate (also called  confidence level) of the robust value function is consistent with our theories. We leave more details in Section~\ref{sec: exp}.

\subsection{Related Work}

In this subsection, we summarize prior works on three topics: offline RL, robust MDPs, and distributionally robust optimization (DRO). 

\paragraph*{\textbf{Offline RL}} Two most fundamental problems in offline RL are Off-Policy Evaluation (OPE) and Off-Policy Learning (OPL). 
These two problems assume the agent is unable to interact with the environment but only has access to a given explorable dataset. In terms of OPE whose purpose is to estimate the value function with a given policy, there are mainly three different methods: Direct Method (DM), Importance Sampling (IS) \cite{hirano2003efficient, li2015toward, liu2018breaking, swaminathan2015self}, and Doubly Robust (DR) method \cite{dudik2014doubly, jiang2016doubly, thomas2016data, farajtabar2018more, kallus2020double}. Here we only discuss the most related method DM. For DM, the usual treatment is firstly estimating the reward and transition probability from the offline dataset, and then applying the model estimators to solve the empirical MDP to obtain the value function. \citet{mannor2004bias}  analyzed the bias and variance of the value function estimation by applying frequency estimators of models in tabular MDPs. To tackle large-scale MDP problems, \citet{jong2007model, grunewalder2012modelling}  proposed other methods to estimate the model of dynamics.
\citet{bertsekas1995neuro, dann2014policy, duan2020minimax} then extended the DM method to the setting of value function approximation  by different algorithms, including regression methods. 
It is more challenging to analyze  OPL (or Batch RL)  than OPE, especially  under function approximation settings, because the goal of OPL is to learn the optimal policy from the given dataset. 
When certain assumptions are made, many works have discussed the necessary and sufficient  conditions for an efficient OPL and provided sample-efficient algorithms within different function hypothesis classes \cite{munos2008finite, lazaric2012finite, le2019batch,chen2019information, xie2020batch, wang2020statistical, yin2020near, duan2021risk}.

\paragraph*{\textbf{Robust MDPs\footnote{During the revising process of this manuscript, we noted one very latest paper \cite{panaganti2021sample} appeared online, which only studies the finite-sample results of robust MDPs under the $(s,a)$-rectangular assumption. Compared with their finite-sample results, our corresponding results keep the same as theirs when the $L_1$ uncertainty set is applied. However, our results have a better dependence on $(1-\gamma)^{-1}$ and $\varepsilon$ in cases of both the $\chi^2$ and KL uncertainty sets, whereas
their bound still has an exponential dependence on $(1-\gamma)^{-1}$ when the KL uncertainty set is applied.}}} 
Robust MDPs are related to DM in offline RL. The usual approach to solving robust MDPs is estimating the reward and transition probabilities firstly, and running dynamic programming algorithms to obtain near-optimal solutions \cite{iyengar2005robust,nilim2005robust}. Different from the conventional MDPs \cite{puterman2014markov}, robust MDPs allow transition probability taking values in an uncertainty set \cite{xu2006robustness, mannor2012lightning} and aim to obtain an optimal robust policy that maximizes the worst-case value function.  \citet{xu2009parametric, petrik2012approximate, ghavamzadeh2016safe}  showed that the solutions of robust MDPs are less sensitive to estimation errors. However, the choice of uncertainty sets still matters with the solutions of robust MDPs. \citet{wiesemann2013robust}  concluded that with the $(s,a)/s$-rectangular and convex set assumptions, the computation complexity of obtaining near-optimal solutions is polynomial.

If the uncertainty set is non-rectangular, the problem becomes NP-hard \cite{wiesemann2013robust}. 
With the $(s,a)/s$-rectangular set assumptions, many works have provided efficient learning algorithms to obtain near-optimal solutions in different uncertainty sets \cite{iyengar2005robust, nilim2005robust,wiesemann2013robust, kaufman2013robust, ho2018fast, smirnova2019distributionally, ho2020partial}. In addition, \citet{goyal2018robust} considered a more general assumption called the $r$-rectangular when MDPs have a low dimensional linear representation. And \citet{derman2020distributional} also proposed an extension of robust MDPs  (called distributionally robust MDPs) under a Wasserstein distance. 

There are few works considering the non-asymptotic performances of optimal robust policy as Eqn.~\eqref{eq: gap} states. \citet{si2020distributionally}  considered the asymptotic and non-asymptotic behaviors of the optimal robust solutions in the bandit case when only the KL divergence is applied in the uncertainty set. 
\citet{zhou2021finite} extended the non-asymptotic results of \citet{si2020distributionally} to the infinite horizon RL case.
More importantly, \citet{zhou2021finite}  gave a sample complexity bound $\widetilde{\OM}\left(\frac{C|\SM|^2 }{\nu_{\min}\varepsilon^2\rho^2(1-\gamma)^2}\right)$. However, 
they only considered the settings when the KL divergence is applied in the uncertainty set and the $(s,a)$-rectangular assumption is made, while we consider the settings of the KL ball and other uncertainty sets under both the $(s,a)$ and $s$-rectangular assumptions. In addition,  the result of \citet{zhou2021finite} is exponentially dependent on $\frac{1}{1-\gamma}$, which is hidden in an unspecified parameter $C=\exp(\frac{1}{\beta(1-\gamma)})$. Indeed, the results of \citet{zhou2021finite} gave readers a misconception that robust MDPs are exponentially hard than original MDPs in terms of $1/(1-\gamma)$. In this paper, we break this misconception and prove that robust MDPs can be sample efficient and have lower sample complexity than original MDPs. It is also worth pointing out that an unknown parameter $\beta$ is hidden in $C$, which is an optimal solution for a convex problem and has no explicit expression. In our work, we improve their results to a polynomial and explicit sample complexity bound, which is shown in Tables~\ref{tab: result} and~\ref{tab: result-off}.

\paragraph*{\textbf{Distributionally Robust Optimization (DRO)}} Handling uncertainty sets in robust MDPs is relevant with Distributionally Robust Optimization (DRO), where the objective function is minimized with a worst-case loss function. The core motivation of DRO is to deal with the distribution shift of data using different uncertainty sets. \citet{bertsimas2018data, delage2010distributionally} formulated the uncertainty set  by moment conditions, while \citet{ben2013robust, duchi2021statistics, duchi2016variance, lam2016robust,duchi2018learning} formulated the uncertainty set   by $f$-divergence balls. In addition, \citet{wozabal2012framework, blanchet2019quantifying, gao2016distributionally, lee2017minimax} also considered Wasserstein balls, which is more computationally challenging. The most related work with our results is \citet{duchi2018learning}, which considered the asymptotic and non-asymptotic performances of the empirical minimizer on a population level. However, the result of \citet{duchi2018learning} is mainly built on the supervised learning scenario, while our results are built on robust MDPs. Recently, a line of works \cite{jin2020pessimism, dai2020coindice, xiao2021optimality} has studied the connection between pessimistic RL and DRO.

%% file: tex/preliminary.tex
\paragraph*{Markov Decision Processes} A discounted Markov decision process is defined by a 5-tuple $(\SM,\AM, R, P, \gamma)$, where $\SM$ is the state space and $\AM$ is the action space. In this paper, we assume both $\SM$ and $\AM$ are finite discrete spaces. The reward function satisfies: $R\colon \SM\times\AM\rightarrow[0,1]$,  the transition probability satisfies: $P\colon \SM\times\AM\rightarrow\Delta(\SM)$, where $\Delta(\XM)=\{P\colon \sum_{x\in\XM}P(x)=1, P(x)\ge0\}$ is a set containing all probability measures on a given finite space $\XM$, and $\gamma \in [0, 1)$ is the discount factor. A stationary policy $\pi$ is defined as $\pi\colon \SM\rightarrow\Delta(\AM)$ and the value function of a policy $\pi$ is defined as $V^\pi_P(s)=\EB_{\tau\sim\pi}[\sum_{t=0}^{\infty}\gamma^t R(s_t,a_t)|s_0=s]$, where $\tau\sim\pi$ stands for the trajectory $\tau=(s_0, a_0, s_1, a_1, \ldots)$  generated according to policy $\pi$ and transition probability $P$. Furthermore, if the initial distribution $\mu$ is given, the value function is  $V^\pi_P(\mu)=\EB_{s_0\sim\mu}V^\pi_P(s_0)$. The goal of learning an MDP is to solve the problem $\max_\pi V^\pi_P(s)$ for all $s\in\SM$ or $\max_\pi V^\pi_P(\mu)$. We denote the optimal value $V_P^*(s):=\max_\pi V_P^\pi(s)$.




\paragraph*{Robust Markov Decision Processes} A robust approach to solving MDP is considering the worst MDP case. The robust value function is $V^\pi_{r}(s)=\inf_{P\in\PM}V^\pi_P(s)$, where transition probability $P$ is taken in a given uncertainty set $\PM$. The goal of learning a robust MDP is to solve the problem $\max_\pi \inf_{P\in\PM} V^\pi_P(s)$ for all $s\in\SM$ or $\max_\pi \inf_{P\in\PM} V^\pi_P(\mu)$. We denote the optimal robust value function as $V_r^*(s)=\max_\pi \inf_{P\in\PM} V^\pi_P(s)$.

\paragraph*{Assumptions on Uncertainty Set $\PM$} Even though there are various choices of uncertainty set $\PM$, the existence of a stationary robust optimal policy w.r.t.\ a robust MDP is only guaranteed when some conditions of uncertainty set $\PM$ are satisfied. \citet{iyengar2005robust,nilim2005robust}  proposed the $(s,a)$-rectangular set assumption on uncertainty set $\PM$, which is detailed in Assumption~\ref{asmp: sa-rect}. 

\begin{asmp}[$(s,a)$-rectangular]
    \label{asmp: sa-rect}
    The uncertainty set $\PM$ is called an $(s,a)$-rectangular set if it satisfies:
    \begin{align*}
        \PM=\bigtimes_{(s,a)\in\SM\times\AM}\PM_{s,a},
    \end{align*}
    where $\PM_{s,a}\subseteq \Delta(\SM)$ and ``$\bigtimes$'' represents the Cartesian product.
\end{asmp}
It is shown that the optimal robust policy is stationary and deterministic\footnote{A deterministic policy stands for $\pi(a|s)\in\{0,1\}$ for all $(s,a)\in\SM\times\AM$.} under Assumption~\ref{asmp: sa-rect}. 
In addition, \citet{epstein2003recursive,wiesemann2013robust} proposed an extensive version $s$-rectangular set, which is detailed in Assumption~\ref{asmp: s-rect}. 

\begin{asmp}[$s$-rectangular]
    \label{asmp: s-rect}
    The uncertainty set $\PM$ is called an $s$-rectangular set if it satisfies:
    \begin{align*}
        \PM=\bigtimes_{s\in\SM}\PM_s,
    \end{align*}
    where $\PM_{s}\subseteq \Delta(\SM)^{|\AM|}$ and $\Delta(\SM)^{|\AM|}:=\{(P_a)_{a\in\AM}|P_a\in\Delta(\SM)\text{, for all $a\in\AM$}\}$.
\end{asmp}

It is shown that the optimal robust policy is stationary, while the optimal robust policy could be stochastic\footnote{A stochastic policy stands for $\pi(a|s)\in[0,1]$ for all $(s,a)\in\SM\times\AM$.} instead of deterministic. For a more general uncertainty set, \citet{wiesemann2013robust} mentioned that it could be NP-hard to obtain the optimal robust policy, which could also be non-stationary and stochastic.

\paragraph*{Examples of uncertainty set} Currently, the most frequently used uncertainty sets can all be categorized to the $f$-divergence set as Examples~\ref{eg: f-set} and~\ref{eg: f-set-s} state, where $P^*$ is the center transition probability and $\rho$ determines the size of sets. \citet{iyengar2005robust}  used the $L_1$ uncertainty set when setting $f(t)=|t-1|$. And \citet{nilim2005robust} used the KL uncertainty set when setting $f(t)=t\log t$. In DRO, \citet{duchi2018learning} used a more general form of $f(t)\propto t^k$ where $k>1$, while we only consider $k=2$ in this paper. As we focus on the statistical performances of robust MDPs, we use $\PM_{s,a}(\rho)$, $\PM_{s}(\rho)$ and $\PM$ to represent the uncertainty sets when true transition probability $P^*$ is applied. And we use $\widehat{\PM}_{s,a}(\rho)$, $\widehat{\PM}_s(\rho)$ and $\widehat{\PM}$ to represent the uncertainty sets when estimated transition probability $\widehat{P}$ is applied.


\begin{example}[$f$-divergence under the $(s,a)$-rectangular assumption]
    \label{eg: f-set}
    For each $(s,a)$ pair, we denote the center probability by $P^*(\cdot|s,a)$ and the size of the set by $\rho>0$. The $f$-divergence $(s,a)$-rectangular set is defined by:
    \begin{align*}
        \PM_{s,a}(\rho)=\left\{P\in\Delta(\SM)\Bigg{|}P\ll P^*(\cdot|s,a)\text{\footnotemark},\hspace{2pt}\sum_{s'\in\SM}f\left(\frac{P(s')}{P^*(s'|s,a)}\right)P^*(s'|s,a)\le\rho\right\}.
    \end{align*}
    \footnotetext{For any two probabilities $P,Q$ supporting on a finite set $\XM$, $P\ll Q$ stands for $P$ is absolutely continuous w.r.t. $Q$, which means for any $x\in\XM$, $Q(x)=0$ implies $P(x)=0$.}
\end{example}

\begin{example}[$f$-divergence under the $s$-rectangular assumption]
    \label{eg: f-set-s}
    For each $s\in\SM$, we denote the center probability by $P^*(\cdot|s,a)$ and the size of the set by $\rho>0$. The $f$-divergence $s$-rectangular set is defined by:
    \begin{align*}
        \PM_{s}(\rho)=\left\{P\in\Delta(\SM)^{|\AM|}\Bigg{|}P(\cdot|a)\ll P^*(\cdot|s,a),\hspace{2pt}\sum_{s'\in\SM, a\in\AM}f\left(\frac{P(s'|a)}{P^*(s'|s,a)}\right)P^*(s'|s,a)\le|\AM|\rho \right\}.
    \end{align*}
\end{example}

 \paragraph*{Connection with Non-robust MDPs} In our settings (Examples~\ref{eg: f-set} and \ref{eg: f-set-s}), the parameter $\rho$ controls the difference between robust value function $V_r^\pi$ and non-robust value function $V^\pi$. Intuitively, we would expect a small difference for a small $\rho$, which is quantified by the following theorem.
\begin{thm}
    \label{thm: non-robust-difference}
    If there exists a monotonically increasing and concave function $h(t)\colon \RB_+\to\RB_+$ such that for any probability distributions $P, Q\in\Delta(\SM)$ with $P\ll Q$:
    \begin{align}
        \label{cond: f_h}
        \sum_{s\in\SM} |P(s)-Q(s)|\le h\left(\sum_{s\in\SM} f\left(\frac{P(s)}{Q(s)}\right)Q(s)\right),
    \end{align}
    
    then for any fixed policy $\pi$, we have:
    \begin{align*}
        \left\|V_r^\pi - V^\pi_{P^*}\right\|_\infty\le
        \begin{cases}
            \frac{\gamma h(\rho)}{(1-\gamma)^2},\quad & \mbox{if } \; \text{Example~\ref{eg: f-set} is applied}, \\
            \frac{\gamma |\AM|h(\rho)}{(1-\gamma)^2}, \quad & \mbox{if } \;  \text{Example~\ref{eg: f-set-s} is applied}.
        \end{cases}
    \end{align*}
    Specifically, if we use $f(t)=|t-1|$, $(t-1)^2$, and $t\log t$, respectively, then $h(t)=t$, $\sqrt{t}$, and $\sqrt{2t}$, respectively\footnote{The specific result is obtained by Cauchy-Schwarz inequality and Pinsker's inequality, see \citet{sason2016f} for details.}.
\end{thm}

\paragraph*{Performance gap of Robust MDPs} We usually do not have access to the true transition probability $P^*$ but an unbiased estimated transition probability $\widehat{P}$ can be obtained from a dataset. The empirical optimal robust policy is given by $\widehat{\pi}^*=\argmax_\pi \widehat{V}_r^\pi(\mu)$, where $\widehat{V}_r^\pi(\mu)=\inf_{P\in\widehat{\PM}}V^\pi_P(\mu)$. To test the performance of empirical solution $\widehat{\pi}^*$, we evaluate it by the following performance gap:
\begin{align}
    \label{eq:dev-bound}
    \max_\pi V_r^\pi(\mu)-V_r^{\widehat{\pi}^*}(\mu).
\end{align}
Following the uniform convergence argument in statistical learning theory \cite{mohri2018foundations, hastie2009elements}, we can bound this gap by a uniform excess risk \cite{mohri2018foundations} as Lemma~\ref{lem: uni-dev-pi} states. For any fixed policy $\pi$, we note that $\widehat{V}_r^\pi$ is a fixed point of robust Bellman operator $\widehat{\TM}_r^\pi$, which is similar to the non-robust case \cite{puterman2014markov}. Thus, we can further bound the uniform excess risk as Lemma~\ref{lem: uni-dev-v} states. Thus, as long as $\widehat{\TM}_r^\pi V$ approximates $\TM_r^\pi V$ with enough samples for fixed $V\in\VM$ and $\pi\in\Pi$, we can bound the supreme of the uniform excess risks by union bound over $\Pi$ and $\VM$.
\begin{lem}
    \label{lem: uni-dev-pi}
    Denote $\widehat{\pi}^*=\argmax_{\pi\in\Pi}\widehat{V}_r^\pi(\mu)$, where $\widehat{V}_r^\pi(\mu)=\inf_{P\in\widehat{\PM}}V^\pi_P(\mu)$ and $\widehat{\PM}$ is the uncertainty set with $\widehat{P}$  applied. Then the following inequality holds:
    \begin{align*}
        0\le \max_\pi V_r^\pi(\mu)-V_r^{\widehat{\pi}}(\mu)\le 2\sup_{\pi\in\Pi}\left|V^\pi_r(\mu)-\widehat{V}^\pi_r(\mu)\right|,
    \end{align*}
    where $\Pi=\Delta(\AM)^{|\SM|}$ contains all probabilities on simplex $\Delta(\AM)$ for each $s\in\SM$.
\end{lem}

\begin{lem}
    \label{lem: uni-dev-v}
    Denoting $V_r^\pi=\left(V_r^\pi(s)\right)_{s\in\SM}$ and $\widehat{V}_r^\pi=\left(\widehat{V}_r^\pi(s)\right)_{s\in\SM}$, we have
    \begin{align*}
        \left\| V_r^\pi-\widehat{V}_r^\pi\right\|_\infty\le\frac{1}{1-\gamma}\sup_{V\in\VM}\left\|\TM_r^\pi V- \widehat{\TM}_r^\pi V\right\|_\infty,
    \end{align*}
    where $\TM^\pi_r V= R^\pi+\gamma\inf_{P\in\PM}P^\pi V $, $\widehat{\TM}^\pi_r V= R^\pi+\gamma\inf_{P\in\widehat{\PM}}P^\pi V $ for any $V\in\VM:=\left[0,\frac{1}{1-\gamma}\right]^{|\SM|}$ and $R^\pi(s):=\sum_{a\in\AM}R(s,a)\pi(a|s)$, $P^\pi(s'|s):=\sum_{a\in\AM}P(s'|s,a)\pi(a|s)$.
\end{lem}
\begin{rem}
    \label{rem: deter-r}
    In this paper we consider $\widehat{\TM}_r^\pi V = R^\pi +\gamma\inf_{P\in\widehat{\PM}}P^\pi V$ with a deterministic reward. If $R(s,a)$ is a bounded random variable for each $(s,a)\in\SM\times\AM$, we could easily obtain that $\sup_{s,a}|\widehat{R}(s,a)-\EB R(s,a)|\le \widetilde{\OM}\left(n^{-1/2}\right)$ with high probability by Hoeffding's inequality, which is much smaller than the statistical error incurred by estimation of transition probability.
\end{rem}

Thus, our ultimate goal is to evaluate the supreme of $\left\|\TM_r^\pi V- \widehat{\TM}_r^\pi V\right\|_\infty$ over $\VM$ and $\Pi$.
To do so, we need to estimate the sizes of $\VM$ and $\Pi$ to apply concentration inequalities over the whole sets. 
Noting that $\VM$ is an infinite subset of $\RB^{|\SM|}$, we apply Lemma~\ref{lem: eps-v} to discretize the value space $\VM$ and bound the performance gap. 
To discretize the policy set $\Pi$, we consider two cases. 
When the $(s,a)$-rectangular assumption holds, the optimal robust policy is deterministic, leading to the policy class being finite (i.e., $|\Pi|=|\AM|^{|\SM|}$).
However, when the $s$-rectangular assumption holds, the optimal policy may be stochastic instead of deterministic, which means the policy class is infinite. 
Thus, we need Lemma~\ref{lem: eps-pi} to help us control the deviation. 
We also prove that the covering numbers of $\VM$ and $\Pi$ are bounded as Lemma~\ref{lem: num-eps-pi} states, which is useful to bound the supreme value over $\VM_\varepsilon$ and $\Pi_\varepsilon$.

\begin{lem}
    \label{lem: eps-v}
    Let $\VM_\varepsilon:=\NM(V,\|\cdot\|_\infty,\varepsilon)$ denote the smallest $\varepsilon$-net of $\VM$ w.r.t. norm $\|\cdot\|_{\infty}$, which satisfies: $\forall V\in\VM$ there exists a $V_0\in\VM_\varepsilon$ such that $\|V-V_0\|_\infty\le\varepsilon$. Then, we have:
    \begin{align*}
        \sup_{V\in\VM}\left\|\TM_r^\pi V-\widehat{\TM}_r^\pi V\right\|_\infty\le2\gamma\varepsilon+\sup_{V\in\VM_\varepsilon}\left\|\TM_r^\pi V-\widehat{\TM}_r^\pi V\right\|_\infty.
    \end{align*}
\end{lem}

\begin{lem}
	\label{lem: eps-pi}
    Let $\Pi_\varepsilon:=\NM(\Pi,\|\cdot\|_1, \varepsilon)$ denote the smallest $\varepsilon$-net of $\Pi$ w.r.t. norm $\|\cdot\|_1$, which satisfies: $\forall \pi\in\Pi$ there exists a $\pi_0\in\Pi_\varepsilon$ such that $\|\pi(\cdot|s)-\pi_0(\cdot|s)\|_1\le\varepsilon$ for all $s\in\SM$. Then, we have:
    \begin{align*}
        \sup_{\pi\in\Pi,V\in\VM}\left\|\TM_r^\pi V-\widehat{\TM}_r^\pi V\right\|_\infty\le\frac{2\gamma\varepsilon}{1-\gamma}+\sup_{\pi\in\Pi_\varepsilon,V\in\VM}\left\|\TM_r^\pi V-\widehat{\TM}_r^\pi V\right\|_\infty.
    \end{align*}
\end{lem}

\begin{lem}
    \label{lem: num-eps-pi}
    The cardinalities of $\VM_\varepsilon$ in Lemma~\ref{lem: eps-v} and $\Pi_\varepsilon$ in Lemma~\ref{lem: eps-pi} can be respectively bounded by:
    \begin{align*}
        \left|\VM_\varepsilon\right|\le\left(1+\frac{1}{(1-\gamma)\varepsilon}\right)^{|\SM|}\; \mbox{ and } \; \hspace{4pt}\left|\Pi_\varepsilon\right|\le\left(1+\frac{4}{\varepsilon}\right)^{|\SM||\AM|}.
    \end{align*}
\end{lem}
\begin{rem}
    We give a high-level idea on the construction of $\VM_{\varepsilon}$ and $\Pi_{\varepsilon}$ in Lemma~\ref{lem: num-eps-pi}. 
    For $\VM_{\varepsilon}$, we can just divide $[0,1/(1-\gamma)]$ by a factor $\varepsilon$ at each dimension $s\in\SM$. For $\Pi_{\varepsilon}$, we can use $L_1$ balls in $\RB^{|\AM|-1}$ with size $\varepsilon$ to cover the entire policy space.
\end{rem}

%% file: tex/gen.tex
In this section we assume there is access to a generative model such that for any given pair $(s,a)\in\SM\times\AM$, it is able to return an arbitrary value of next states $s'$ following probability $P^*(\cdot|s,a)$. Thus, according to the generated samples, we can construct the empirical estimation of transition probability $P^*$ by:
\begin{align}
    \label{eq: gen}
    \widehat{P}(s'|s,a)=\frac{1}{n}\sum_{k=1}^n \mathbf{1}(X_k^{s,a}=s'),
\end{align}
where $\{X_k^{s,a}\}_{k=1}^n$ are i.i.d.\ samples generated from $P^*(\cdot|s,a)$. Thus, $\widehat{P}$ is an unbiased estimator of $P^*$. With the generative model, our non-asymptotic results are stated in the following theorems. In our proof, as the dual problems differ for different choices of $f$, it is unlikely to obtain a unified concentration result covering all the three settings ($L_1$, $\chi^2$, and KL cases). 
Before presenting theoretical results, the proof sketch is as follows. 
\begin{itemize}
    \item Firstly, for any fixed $\pi\in\Pi$ and $V\in\VM$, we calculate the dual forms of $\TM_r^\pi V (s)$ and $\widehat{\TM}_r^\pi V (s)$ for all $s\in\SM$ for the different uncertainty sets.
    \item Secondly, we bound the concentration error $\left\|\TM_r^\pi V - \widehat{\TM}_r^\pi V\right\|_\infty$ for fixed $\pi\in\Pi$ and $V\in\VM$ from the dual forms. 
    \item Next, as $\left\|\TM_r^\pi V - \widehat{\TM}_r^\pi V\right\|_\infty$ is Lipschitz w.r.t.\ $V\in\VM$ in norm $\|\cdot\|_\infty$, we can derive a union bound over $V\in\VM$ by Lemma~\ref{lem: eps-v}.
    \item Finally, under the $(s,a)$-rectangular assumption, the optimal robust policy is deterministic. Thus, we can derive a union bound over the deterministic policy class, which is finite and satisfies $|\Pi|=|\AM|^{|\SM|}$. However, when we consider the $s$-rectangular assumption, the optimal robust policy may be stochastic, which leads to the policy class $\Pi$ to be infinitely large. According to Lemma~\ref{lem: eps-pi}, we can also derive a union bound over $\Pi$ by taking an $\varepsilon$-net of $\Pi$.
\end{itemize}

\begin{rem}
    We can also extend our non-asymptotic results in this section to the setting with an offline dataset, which can be referred to in Supplementary Appendix 9.
\end{rem}

\subsection{Results with the $(s,a)$-rectangular assumption}

Taking $f(t)=|t-1|$, $f(t)=(t-1)^2$, and $f(t)=t\log t$ in Example~\ref{eg: f-set}, respectively, we have the following results when the $(s,a)$-rectangular assumption holds.

\begin{thm}
    \label{thm: finite-sa}
    Under the $(s,a)$-rectangular assumption, the following results hold:
    \begin{enumerate}[label=(\alph*), ref=\thethm (\alph*)]
        \item \label{thm: l1-union} If $f(t)=|t-1|$ in Example~\ref{eg: f-set} ($L_1$ balls), then with probability $1-\delta$:
        \begin{align*}
            \max_{\pi} V^\pi_r(\mu)- V_r^{\widehat{\pi}}(\mu)\le\frac{2(2+\rho)\gamma\sqrt{|\SM|}}{\rho(1-\gamma)^2\sqrt{2n}}\left(2+\sqrt{\log\frac{4|\SM||\AM|^2 [1+2(2+\rho)\sqrt{2n} ]^2}{\delta(2+\rho)}}\right).            
        \end{align*}
        \item \label{thm: chi2-uni} If $f(t)=(t-1)^2$ in Example~\ref{eg: f-set} ($\chi^2$ balls), then with probability $1-\delta$:
        \begin{align*}
            \max_\pi V_r^\pi(\mu)-V_r^{\widehat{\pi}}(\mu)\le\frac{2C^2(\rho)\gamma\sqrt{|\SM|}}{(C(\rho) {-} 1)(1 {-} \gamma)^2\sqrt{n}}\left(4+\sqrt{2\log\frac{2|\SM||\AM|^2 [1+(C(\rho) {+} 3)\sqrt{n} ]^2}{\delta C^2(\rho)}}\right),
        \end{align*}
        where $C(\rho)=\sqrt{1+\rho}$.
        \item \label{thm: kl-uni} If $f(t)=t\log t$ in Example~\ref{eg: f-set} (KL balls), then with probability $1-\delta$:
        \begin{align*}
            \max_\pi V_r^\pi(\mu)-V_r^{\widehat{\pi}}(\mu)\le\frac{4\gamma\sqrt{|\SM|}}{\rho(1-\gamma)^2\underline{p}\sqrt{n}}\left(1+\sqrt{\log\frac{2|\SM|^2|\AM|^2 [1+\rho\underline{p}\sqrt{n}] }{\delta}}\right),
        \end{align*}
        where $\underline{p}=\min_{P^*(s'|s,a)>0}P^*(s'|s,a)$.
    \end{enumerate}
\end{thm}

As a brief sum-up, in order to achieve an $\varepsilon$ performance gap, the number of generated samples should be $n_{\text{tot}}:=n\times|\SM||\AM|=\widetilde{\OM}\left(\frac{|\SM|^2|\AM|}{\varepsilon^2\rho^2(1-\gamma)^4}\right)$ in all the cases under the $(s,a)$-rectangular assumption, up to some logarithmic factors. 

And notably, in Theorem~\ref{thm: kl-uni}, an additional factor $1/\underline{p}$ occurs in the upper bound, which seems unavoidable. From a high-level point of view, the core step in Theorem~\ref{thm: kl-uni} can be expressed by bounding the deviation $|\log \frac{1}{n}\sum_i X_i-\log\mu|$ with $\EB X_i=\mu$, where the factor $1/\mu$ plays a significant role on the sample complexity. Fortunately, compared to \citet{zhou2021finite}, whose finite-sample result in the KL setting is exponentially dependent on $\frac{1}{1-\gamma}$, our result in the KL setting is only polynomial dependent on $\frac{1}{1-\gamma}$.

\subsection{Results with the $s$-rectangular assumption}

Taking $f(t)=|t-1|$, $f(t)=(t-1)^2$, and $f(t)=t\log t$ in Example~\ref{eg: f-set-s}, respectively, we have the following results when the $s$-rectangular assumption holds.

\begin{thm}
    \label{thm: finite-s}
    Under the $s$-rectangular assumption, the following results hold:
    \begin{enumerate}[label=(\alph*), ref=\thethm (\alph*)]
        \item \label{thm: l1-union-s} If $f(t)=|t-1|$ in Example~\ref{eg: f-set} ($L_1$ balls), then with probability $1-\delta$:
        \begin{align*}
            \max_\pi V_r^\pi(\mu)-V_r^{\widehat{\pi}}(\mu)\le\frac{2\gamma(2+\rho)\sqrt{|\SM||\AM|}}{\rho(1-\gamma)^2\sqrt{2n}}\Bigg{(}4+\sqrt{\log\frac{2|\SM|(1+2\sqrt{2n}(\rho+4))^3}{\delta}}\Bigg{)}.          
        \end{align*}
        \item \label{thm: chi2-s-union} If $f(t)=(t-1)^2$ in Example~\ref{eg: f-set} ($\chi^2$ balls), then with probability $1-\delta$:
        \begin{align*}
            \max_\pi V_r^\pi(\mu)-V^{\widehat{\pi}}_r(\mu)\le\frac{2\gamma C^2(\rho)\sqrt{|\SM||\AM|^2}}{(C(\rho)-1)(1-\gamma)^2\sqrt{n}}\left(6+\sqrt{2\log\frac{2|\SM|(1+8\sqrt{n}C(\rho))^3}{\delta}}\right),
        \end{align*}
        where $C(\rho)=\sqrt{1+\rho}$.
        \item \label{thm: kl-s-union} If $f(t)=t\log t$ in Example~\ref{eg: f-set} (KL balls), then with probability $1-\delta$:
        \begin{align*}
            \max_\pi V_r^\pi(\mu)-V^{\widehat{\pi}}_r(\mu)\le\frac{4\gamma\sqrt{|\SM||\AM|}}{\rho\underline{p}(1-\gamma)^2\sqrt{n}}\left(2+\sqrt{2\log\frac{2|\SM|^2|\AM|\left(1+4\rho\underline{p}\sqrt{n}\right)}{\delta}}\right),
        \end{align*}
        where $\underline{p}=\min_{P^*(s'|s,a)>0}P^*(s'|s,a)$.
    \end{enumerate}
\end{thm}

Under the $s$-rectangular assumption, the optimal robust policy can be stochastic. In this case, the policy class $\Pi$ is infinitely large. By controlling the deviation through Lemma~\ref{lem: eps-pi}, there could be an amplification in the statistical error. In the cases of both the $L_1$ and KL balls, the total sample complexity to achieve an $\varepsilon$ performance gap is $n_{\text{tot}}=\widetilde{\OM}\left(\frac{|\SM|^2|\AM|^2}{\varepsilon^2\rho^2(1-\gamma)^4}\right)$. But in the case of the $\chi^2$ balls, the total sample complexity is $n_{\text{tot}}=\widetilde{\OM}\left(\frac{|\SM|^2|\AM|^3}{\varepsilon^2\rho^2(1-\gamma)^4}\right)$, which is larger than others and caused by the specific dual solution of $\TM_r^\pi V$.

\subsection{Discussion on $\rho$}
\label{subsec: discussrho}
All of the results under the $(s,a)$ and $s$-rectangular assumptions suggest that the sample complexity would be unbounded when $\rho\to0$. To illustrate this phenomenon, we consider a simple distributionally robust optimization problem:
\begin{align}
    \inf_Q &\sum_{i=1}^{|\XM|} Q_iV_i,\label{eq:dro_primal}\\
    \text{s.t. } &\sum_{i=1}^{|\XM|} f\left(\frac{Q_i}{P_i}\right)P_i\le\rho,\label{eq:dro_constraint}\\
    &\sum_{i=1}^{|\XM|} Q_i = 1, \; Q_i\ge 0 \hspace{2pt}, \; \forall i=1, \ldots, |\XM|.
\end{align}
Here we assume $P\in \Delta(\XM)$ and  $P_i>0$ for all $i$. In addition, $f$ is a convex function such that $f(1)=0$ and $V_i\in[0,M]$ for all $i$. We denote the optimal value of the above problem~\eqref{eq:dro_primal} as $g(P,\rho)$. Now if we have an unbiased estimator $\widehat{P}$ of $P$, we would like to know the absolute error between $g(P,\rho)$ and $g(\widehat{P},\rho)$. However, we cannot apply concentration inequality to $g(\widehat{P},\rho)$ directly as the randomness is hidden in the constraint~\eqref{eq:dro_constraint}. Fortunately, we can write the dual problem of $g(P,\rho)$ and prove the strong duality. In this case, the randomness is displayed in the dual objective~\eqref{eq:dro_dual}, where $f^*(y)=-\inf_{x\ge0}f(x)-x y$ is the conjugate function of $f$.
We denote the dual objective~\eqref{eq:dro_dual} as $d(P,\rho)$,
That is, 
\begin{align}
    \sup_{\lambda\ge0,\beta\in\RB}-\sum_{i=1}^{|\XM|}\lambda P_i f^*\left(-\frac{V_i+\beta}{\lambda}\right)-\lambda\rho-\beta\label{eq:dro_dual}.
\end{align}

To control the error $\left|d(\widehat{P},\rho)-d(P,\rho)\right|$, we have to determine the range of dual variables $\lambda$ and $\beta$ based on the specific choice of $f$. Then we can apply concentration inequalities uniformly over the range of dual variables. However, the range of dual variables will enlarge to infinity when $\rho$ goes to zero. In this case, the uniform concentration inequalities will suffer error amplification, which leads to infinite sample complexity. Alternatively, by Theorem~\ref{thm: non-robust-difference} (setting $\gamma=0$ in this case) and Hoeffding's inequality \cite{wainwright2019high}, we can upper bound the primal error by:
\begin{align*}
    \left|g(\widehat{P},\rho)-g(P,\rho)\right|&\le\left|g(\widehat{P},\rho)-g(\widehat{P},0)\right|+\left|g(P,\rho)-g(P,0)\right|+\left|g(\widehat{P},0)-g(P,0)\right|\\
    &\le \OM\left(h(\rho)\right)+\left|g(\widehat{P},0)-g(P,0)\right|\\
    &= \OM\left(h(\rho)\right)+\widetilde{\OM}\left(\frac{1}{\sqrt{ n}}\right).
\end{align*}

In this case, when $\rho$ approaches zero, we can apply the non-robust results instead. 
When we extend the analysis to the setting of robust MDPs  with  $\rho\to0$, we can alternatively upper bound the performance gap by:
\begin{align*}
    \left\|V_r^*-V_r^{\widehat{\pi}}\right\|_{\infty}&\le2\sup_{\pi\in\Pi}\left\|V_r^\pi-\widehat{V}_r^\pi\right\|_\infty\\
    &\le2\sup_{\pi\in\Pi}\left\|V_r^\pi-V_{P^*}^\pi\right\|_\infty+2\sup_{\pi\in\Pi}\left\|V_{P^*}^\pi-V_{\widehat{P}}^\pi\right\|_\infty+2\sup_{\pi\in\Pi}\left\|V_{\widehat{P}}^\pi-\widehat{V}_r^\pi\right\|_\infty\\
    &\le \OM\left(\frac{h(\rho)}{(1-\gamma)^2}\right)+2\sup_{\pi\in\Pi}\left\|V_{P^*}^\pi-V_{\widehat{P}}^\pi\right\|_\infty\\
    &= \OM\left(\frac{h(\rho)}{(1-\gamma)^2}\right)+\widetilde{\OM}\left(\sqrt{\frac{|\SM|}{(1-\gamma)^4 n}}\right),
\end{align*}
where the first inequality is due to Lemma~\ref{lem: uni-dev-pi}, the second inequality holds by error decomposition, the third inequality holds by Theorem~\ref{thm: non-robust-difference}, and the last equality holds by sample complexity of non-robus MDPs with a generative model \cite{azar2013minimax, agarwal2020model}.
In other words, we should not expect robustness when $\rho\to0$, which also coincides with the theoretical results of the lower bound in the next part.

\subsection{Lower Bound}
\label{sec: lb}
\input{tex/lb.tex}

%% file: tex/lb.tex
To complement our non-asymptotic analysis, here we provide the lower bound results of robust MDPs with a generative model. The MDP we construct in Theorem~\ref{thm: lb} is a classic 2-state MDP with only one action, which is frequently analyzed in \citep{azar2013minimax,duan2021risk}. The details can be found in Supplementary Appendix~\ref{apd: lb}.
\begin{thm}[Lower bound]
    \label{thm: lb}
    There exists a class of robust MDPs with a $f$-divergence uncertainty set, such that for every $(\varepsilon,\delta)$-correct robust RL algorithm $\mathscr{A}$, the total number of generated samples needs to be at least:
    \begin{align*}
        \widetilde{\Omega}\left(\frac{|\SM||\AM|(g'(p))^2 p(1-p)}{\varepsilon^2(1-\gamma g(p))^4}\right),
    \end{align*}
    where $p\in(0,1)$, $g(p)=\inf_{D_f(q\|p)\le\rho}q$ and $D_f(q\|p)=pf(\frac{q}{p})+(1-p)f(\frac{1-q}{1-p})$.
\end{thm}

In Theorem~\ref{thm: lb}, the parameter $p$ can take arbitrary values in $(0,1)$ while we always set $p$ close to $1$. Next, we give the exact lower bounds in the following corollaries when the $L_1$ and $\chi^2$ uncertainty sets are considered. However, when we consider the KL uncertainty set, there is no explicit form of lower bound by the fact that there is no closed-form expression of $g(p)$ when $f(t)=t\log t$.
\begin{cor}[Lower bound for the $L_1$ case]
    \label{cor: lb_l1}
    Given that $f(t)=|t-1|$ and $p=\frac{2\gamma-1}{\gamma}$ in Theorem~\ref{thm: lb}, the lower bound of sample complexity is:
    \begin{align*}
        \widetilde{\Omega}\left(\frac{|\SM||\AM|(1-\gamma)}{\varepsilon^2}\min\left\{\frac{1}{(1-\gamma)^4}, \frac{1}{\rho^4}\right\}\right).
    \end{align*}
\end{cor}

\begin{cor}[Lower bound for the $\chi^2$ case]
    \label{cor: lb_chi2}
    Given that $f(t)=(t-1)^2$ and $p=\frac{2\gamma-1}{\gamma}$ in Theorem~\ref{thm: lb}, the lower bound of sample complexity is:
    \begin{align*}
        \widetilde{\Omega}\left(\frac{|\SM||\AM|}{\varepsilon^2(1-\gamma)^2}\min\left\{\frac{1}{1-\gamma}, \frac{1}{\rho}\right\}\right).
    \end{align*}
\end{cor}

From Corollaries~\ref{cor: lb_l1} and \ref{cor: lb_chi2}, we observe that when $\rho\le(1-\gamma)$, the lower bound is exactly $\widetilde{\Omega}\left(\frac{|\SM||\AM|}{\varepsilon^2(1-\gamma)^3}\right)$, which coincides with the lower bound of classic MDPs with a generative model \citep{azar2013minimax}. When $\rho>(1-\gamma)$, the lower bounds are $\widetilde{\Omega}\left(\frac{|\SM||\AM|(1-\gamma)}{\varepsilon^2\rho^4}\right)$ for the $L_1$ case and $\widetilde{\Omega}\left(\frac{|\SM||\AM|}{\varepsilon^2\rho(1-\gamma)^2}\right)$ for the $\chi^2$ case. 

It is worth noting that a gap exists between the upper bound and lower bound. It is because we obtain the upper bound via a uniform convergence analysis over the whole value space $\VM$ and policy space $\Pi$. If we are able to find the local deviation bound near the optimal robust value function $V_r^*$, the upper bound can be tighter and the gap may also be closed. Unfortunately, we have no additional information of $V_r^*$ except for a robust Bellman equation $V_r^*=\TM_r V_r^*$, which is insufficient to perform a precise local analysis. We think it is an important work to close the gap and we leave it to subsequent works.

%% file: tex/asymp.tex
From the theoretical results of Section~\ref{sec: gen}, we obtain that the statistical convergence rate of robust MDPs is $\widetilde{O}_p(1/\sqrt{n})$\footnote{For any two random variable sequences $\{X_n\}_{n\ge1}$ and $\{Y_n\}_{n\ge1}$, $X_n=o_P(Y_n)$ stands for $X_n/Y_n$ converges to zero in probability as $n$ goes infinity, and $X_n=O_P(Y_n)$ stands for $X_n/Y_n$ is bounded in probability. See \cite{van2000asymptotic} for more details.}. In this section we investigate the asymptotic properties of robust MDPs. Specifically, in the context of robust MDPs, we show that the robust value function $\widehat V_r^\pi(\mu)$ (given policy $\pi$) and the optimal robust value function $\widehat{V}_r^*(\mu)$ are $\sqrt{n}$-consistent and asymptotically normal in both the $(s,a)/s$-rectangular settings. Before presenting our results, we first give a high-level idea about how we prove empirical robust value function to be asymptotically normal.

\begin{itemize}
    \item Firstly, for any fixed policy $\pi\in\Pi$, we prove the empirical robust Bellman noise is asymptotically normal with a variance matrix $\Lambda^\pi\in\RB^{|\SM|\times|\SM|}$:
    \begin{align*}
        \sqrt{n}\left(\widehat{\TM}_r^\pi V_r^\pi-\TM_r^\pi V_r^\pi\right)\tod\NM(0,\;\Lambda^\pi).
    \end{align*}
    \item Noting that $\widehat{V}_r^\pi=\widehat{\TM}_r^\pi\widehat{V}_r^\pi$, we prove that there exists a matrix $\widehat{\M}^\pi\in\RB^{|\SM|\times|\SM|}$, which is the derivative of the operator $I-\widehat{\TM}_r^\pi$ at the point $V_r^\pi$, such that:
    \begin{align*}
        \sqrt{n}\left(\widehat{\TM}_r^\pi V_r^\pi-\TM_r^\pi V_r^\pi\right)&=\sqrt{n}\left(\widehat{\TM}_r^\pi V_r^\pi-\widehat{\TM}_r^\pi\widehat{V}_r^\pi\right)-\sqrt{n}\left(\TM_r^\pi V_r^\pi-\widehat{\TM}_r^\pi\widehat{V}_r^\pi\right)\\
        &=\sqrt{n}\left(\widehat{\TM}_r^\pi V_r^\pi-\widehat{\TM}_r^\pi\widehat{V}_r^\pi\right)-\sqrt{n}\left(V_r^\pi-\widehat{V}_r^\pi\right)\\
        &=-\widehat{\M}^\pi\cdot\sqrt{n}\left(V_r^\pi-\widehat{V}_r^\pi\right)+o_P(\sqrt{n}\|V_r^\pi-\widehat{V}_r^\pi\|).
    \end{align*}
    \item Because the LHS above is asymptotically normal,  we can prove that $\sqrt{n}\left(V_r^\pi {-} \widehat{V}_r^\pi\right)=O_P(1)$. By proving that $\widehat{\M}^\pi$ is consistent to $\M^\pi$, we obtain the final result:
    \begin{align*}
        \sqrt{n}\left(V_r^\pi-\widehat{V}_r^\pi\right)\tod\NM\left(0,\;(\M^\pi)^{-1}\Lambda^\pi(\M^\pi)^{-\top}\right).
    \end{align*}
    \item Also, we can extend the results to the optimal case and leave the discussion of this result in Sections~\ref{subsec: asymp_sa} and~\ref{subsec: asymp_s}:
    \begin{align*}
        \sqrt{n}\left(\max_\pi V_r^\pi-\max_\pi \widehat{V}_r^\pi\right)\tod\NM\left(0,\;(\M^{\pi^*})^{-1}\Lambda^{\pi^*}(\M^{\pi^*})^{-\top}\right).
    \end{align*}
\end{itemize}

\subsection{Results with the $(s,a)$-rectangular assumption}
\label{subsec: asymp_sa}

We consider the asymptotic behaviors of robust value function under $(s,a)$-rectangular assumptions.    From Section~\ref{sec: gen}, we can deduce that the estimator $\widehat{V}_r^\pi(\mu)$ converges to $V_r^\pi(\mu)$ almost surely (also converges in probability) for a given policy $\pi$, which can be seen in Supplementary Appendix 11. Furthermore, $\widehat{V}_r^\pi(\mu)$ is also asymptotically normal with rate $\sqrt{n}$ by the following results. 


\begin{thm}
    \label{thm: normal-sa} Without loss of generality, we  assume $V_r^\pi(s_1)<\cdots<V_r^\pi(s_{|\SM|})$.
    Under the $(s,a)$-rectangular assumption,   we have that for any fixed policy $\pi\in\Pi$,
    \begin{align*}
        \sqrt{n}\left(\widehat{V}_r^\pi(\mu)-V_r^\pi(\mu)\right)\tod\NM\left(0,\; \mu^\top (\M^\pi)^{-1}\Lambda^\pi (\M^\pi)^{-\top}\mu\right),
    \end{align*}
    where $\Lambda^\pi$ is the asymptotic variance of empirical robust Bellman error, and $\M^\pi$ is the derivative of the operator $I-\TM_r^\pi$ at the point $V_r^\pi$.
    
    Specifically, $\Lambda^\pi=\diag\{\sigma_1^2(\pi), \ldots, \sigma^2_{|\SM|}(\pi)\}$ where $\sigma^2_s(\pi)=\gamma^2\sum_{a\in\AM}\pi(a|s)^2\sigma^2(P^*(\cdot|s,a), V_r^\pi)$.  And 
    
    \begin{enumerate}[label=(\alph*), ref=\thethm (\alph*)]
        \item \label{thm: l1_sa_normal} If $f(t)=|t-1|$ in Example~\ref{eg: f-set} ($L_1$ balls), then the $(i,j)$th element of $\M^{\pi}$ is:
        \begin{align*}
            \M^\pi(i,j)&=\mathbf{1}\{i=j\}-\gamma\sum_{a\in\AM}\pi(a|s_i)\Bigg[P(s_j|s_i,a)\mathbf{1}\{j<K(P(\cdot|s_i,a))\}\\
            & \quad -\left(\sum_{k< K(P(\cdot|s_i,a))}P(s_k|s_i,a)-(1-\frac{\rho}{2})\right)\mathbf{1}\{j=K(P(\cdot|s_i,a))\}+\frac{\rho}{2}\mathbf{1}\{j=1\}\Bigg],
        \end{align*}
        where $K(P):=\min\{l\in\ZB_+|\sum_{k\le l}P(s_k)>1-\rho/2\}$ for any $P\in\Delta(\SM)$. And  $\sigma^2(P^*(\cdot|s,a), V_r^\pi)=(b_{s,a}^{\pi})^\top\Sigma_{s,a} b_{s,a}^\pi$ where
        \begin{align*}
            \Sigma_{s,a}(i,j)&=-P^*(s_i|s,a)P^*(s_j|s,a)+P^*(s_i|s,a)\mathbf{1}\{i=j\},\\
            b_{s,a}^{\pi}(i)&=-(\eta^*(P^*(\cdot|s,a), V_r^\pi)-V_r^\pi(s_i))_+,
        \end{align*}
       and $\eta^*$ is the dual solution of $\TM_r^\pi V_r^\pi$.
        \item \label{thm: chi2_sa_normal} If $f(t)=(t-1)^2$ in Example~\ref{eg: f-set} ($\chi^2$ balls), then the $(i,j)$th element of $\M^{\pi}$ is:
        \begin{align*}
            \M^\pi(i,j)=&\mathbf{1}\{i=j\}\\
            &-\gamma\sum_a \pi(a|s_i)C(\rho)\frac{P^*(s_j|s_i,a)(\eta^*(P^*(\cdot|s_i,a),V_r^\pi)-V_r^\pi(s_j))_+}{\sqrt{\sum_{\tilde s\in\SM}P^*(\tilde s|s_i,a)(\eta^*(P^*(\cdot|s_i,a),V_r^\pi)-V_r^\pi(\tilde s))_+^2}},
        \end{align*}
        where $C(\rho)=\sqrt{1+\rho}$. And $\sigma^2(P^*(\cdot|s,a), V_r^\pi)=(b_{s,a}^{\pi})^\top\Sigma_{s,a} b_{s,a}^{\pi}$ where
        \begin{align*}
            \Sigma_{s,a}(i,j)&=-P^*(s_i|s,a)P^*(s_j|s,a)+P^*(s_i|s,a)\mathbf{1}\{i=j\},\\
            b_{s,a}^\pi(i)&=-C(\rho)\frac{(\eta^*(P^*(\cdot|s,a), V_r^\pi)-V_r^\pi(s_i))_+^2}{2\sqrt{\sum_{s'\in\SM}P^*(s'|s,a)(\eta^*(P^*(\cdot|s,a),V_r^\pi)-V_r^\pi(s'))_+^2}},
        \end{align*}
        and $\eta^*$ is the dual solution of $\TM_r^\pi V_r^\pi$.
        \item \label{thm: kl_sa_normal} If $f(t)=t\log t$ in Example~\ref{eg: f-set} (KL balls), then the $(i,j)$th element of $\M^{\pi}$ is:
        \begin{align*}
            \M^\pi(i,j)=\mathbf{1}\{i=j\}-\gamma\sum_a \pi(a|s_i)\frac{P^*(s_j|s_i,a)\exp(-\frac{V_r^\pi(s_j)}{\lambda^*(P^*(\cdot|s_i,a),V^\pi_r)})}{\sum_{\tilde s\in\SM} P^*(\tilde s|s_i,a)\exp(-\frac{V_r^\pi(\tilde s)}{\lambda^*( P^*(\cdot|s_i,a),V_r^\pi)})}.
        \end{align*}
        And 
         $\sigma^2(P^*(\cdot|s,a), V_r^\pi)=(b_{s,a}^\pi)^\top\Sigma_{s,a} b_{s,a}^\pi$ where
        \begin{align*}
            \Sigma_{s,a}(i,j)&=-P^*(s_i|s,a)P^*(s_j|s,a)+P^*(s_i|s,a)\mathbf{1}\{i=j\},\\
            b_{s,a}^\pi(i)&=-\frac{\lambda^*(P^*(\cdot|s,a),V_r^\pi)\exp(-\frac{V_r^\pi(s_i)}{\lambda^*(P^*(\cdot|s,a),V_r^\pi)})}{\sum_{s'\in\SM}P^*(s'|s,a)\exp(-\frac{V_r^\pi(s')}{\lambda^*(P(\cdot|s,a),V_r^\pi(s'))})},
        \end{align*}
        and $\lambda^*$ is the dual solution of $\TM_r^\pi V_r^\pi$.
    \end{enumerate}
\end{thm}

Notably, the asymptotic variance is determined by robust value function $V_r^\pi(\mu)$, $P^*$ and optimal dual variables $\lambda^*$, $\eta^*$ of robust optimization problem $\TM_r^\pi V_r^\pi$. To estimate the asymptotic variance, we can substitute these variables with consistent estimators $\widehat{V}_r^\pi$, $\widehat{P}$, and $\widehat{\lambda}^*$, $\widehat{\eta}^*$, where $\widehat{\lambda}^*$ and $\widehat{\eta}^*$ are dual solutions of problem $\widehat{\TM}_r^\pi\widehat{V}_r^\pi$.
Thus, an asymptotic confidence interval for a given policy $\pi$ can be given by Slutsky's lemma.

We now give the asymptotic results of $\max_\pi \widehat{V}^\pi_r(\mu)$. Prior to that, we define the robust Q-value function $Q_r^\pi(s,a)$ as:
\begin{align*}
    Q_r^\pi(s,a)=R(s,a)+\gamma\inf_{P\in\PM_{s,a}(\rho)}P^\top V^\pi_r.
\end{align*}
Under some mild assumptions, we show that the asymptotic normality of $\max_\pi \widehat{V}_r^\pi(\mu)$ still holds in the following corollary.
\begin{cor}
    \label{cor: sup_asymp}
    Assuming $\min_{s, a_1\not=a_2}\left|Q_r^*(s,a_1)-Q_r^*(s,a_2)\right|>0$, we have:
    \begin{align*}
        \sqrt{n}\left(\max_\pi\widehat{V}_r^\pi(\mu)-\max_\pi V_r^\pi(\mu)\right)\tod\NM\left(0, \; \mu^\top(\M^{\pi^*})^{-1}\Lambda^{\pi^*}(\M^{\pi^*})^{-\top}\mu\right),
    \end{align*}
    where $\pi^*\in\argmax_\pi V_r^\pi(\mu)$, where $\M^{\pi^*}$ and $\Lambda^{\pi^*}$ are defined in Theorem~\ref{thm: normal-sa}.
\end{cor}

Here we give a high-level idea on why Corollary~\ref{cor: sup_asymp} holds. When sample size $n$ is large, we can prove that $\widehat{Q}_r^*(s,a)$ approximates $Q_r^*(s,a)$ for each $(s,a)$ pair. In addition, as the $(s,a)$-rectangular assumption holds, we also know that $\widehat{\pi}^*(s)=\argmax_a \widehat{Q}_r^*(s,a)$ and $\pi^*(s)=\argmax_a Q_r^*(s,a)$ are both deterministic policies. Thus, by Assumption $\min_{s, a_1\not=a_2}\left|Q_r^*(s,a_1)- Q_r^*(s,a_2)\right|>0$, we conclude that $\widehat{\pi}^*=\pi^*$ when sample size $n$ is large. Thus, we can safely consider $\max_\pi \widehat{V}_r^\pi(\mu)=\widehat{V}_r^{\pi^*}(\mu)$ in an asymptotic regime and obtain Corollary~\ref{cor: sup_asymp} by applying $\pi=\pi^*$ in Theorems~\ref{thm: l1_sa_normal},~\ref{thm: chi2_sa_normal} and~\ref{thm: kl_sa_normal}.

\subsection{Results with the $s$-rectangular assumption}
\label{subsec: asymp_s}

We extend the asymptotic results from the $(s,a)$-rectangular assumption to the $s$-rectangular assumption. Unfortunately, when the $L_1$ uncertainty set is applied, the asymptotic behavior is not guaranteed by the fact that the Bellman operator $\TM_r$ is neither differentiable nor affine w.r.t.\ $V$. Instead, the asymptotic normality still holds when either the $\chi^2$ uncertainty set or the KL uncertainty set is applied, which is presented as follows. 

\begin{thm}
    \label{thm: normal-s}
    Without loss of generality, we assume $V_r^\pi(s_1)<\cdots<V_r^\pi(s_{|\SM|})$. Under the $s$-rectangular assumption,  
    we have that for any fixed policy $\pi\in\Pi$,
    \begin{align*}
        \sqrt{n}\left(\widehat{V}_r^\pi(\mu)-V_r^\pi(\mu)\right)\tod\NM\left(0,\; \mu^\top (\M^\pi)^{-1}\Lambda^\pi (\M^\pi)^{-\top}\mu\right),
    \end{align*}
      where $\Lambda^\pi$ is the asymptotic variance of empirical robust Bellman error, and $\M^\pi$ is the derivative of the operator $I-\TM_r^\pi$ at the point $V_r^\pi$.
    
    Specifically,  $\Lambda^\pi=\diag\{\sigma_1^2(\pi), \ldots, \sigma^2_{|\SM|}(\pi)\}$ where $\sigma^2_s(\pi)=\gamma^2\sigma^2(\pi, P^*(\cdot|s,\cdot), V_r^\pi)$. And 
    
    \begin{enumerate}[label=(\alph*), ref=\thethm (\alph*)]
        \item \label{thm: chi2_s_normal} If $f(t)=(t-1)^2$ in Example~\ref{eg: f-set} ($\chi^2$ balls), then the $(i,j)$th element of $\M^{\pi}$ is:
        \begin{align*}
            \M^\pi(i,j)=&\mathbf{1}\{i=j\}\\
            &-\gamma\sqrt{|\AM|}C(\rho)\frac{\sum_a \pi(a|s_i)P^*(s_j|s_i,a)(\eta^*_a(P^*(\cdot|s_i,\cdot),V_r^\pi)-\pi(a|s_i)V_r^\pi(s_j))_+}{\sqrt{\sum_{\tilde s,a}P^*(\tilde s|s_i,a)(\eta^*_a(P^*(\cdot|s_i,\cdot),V_r^\pi)-\pi(a|s_i)V_r^\pi(\tilde s))_+^2}},
        \end{align*}
        where $C(\rho)=\sqrt{1+\rho}$. And $\sigma^2(\pi, P^*(\cdot|s,\cdot), V_r^\pi)=(b_{s}^{\pi})^\top\Sigma_{s} b_{s}^{\pi}$ where for two pairs $(s_i, a_k)$ and $(s_j, a_l)$:
        \begin{align*}
            \Sigma_s((s_i,a_k),(s_j,a_l))&=\left(-P^*(s_i|s, a_k)\cdot P(s_j|s, a_l)+P(s_i|s, a_k)\mathbf{1}\{ s_i=s_j\}\right)\mathbf{1}\{a_k=a_l\},\\
            b_{s}^\pi(s_i, a_k)&=-   \frac{\sqrt{|\AM|}C(\rho)  (\eta_{a_k}^*(P^*(\cdot|s, \cdot), V_r^\pi)-\pi(a_k|s)V_r^\pi( s_i))_+^2}{2\sqrt{\sum_{s'\in\SM,a'\in\AM}P(s'|s, a')(\eta^*_{a'}(P^*(\cdot|s,\cdot), V_r^\pi)-\pi(a'|s)V_r^\pi(s'))_+^2}},
        \end{align*}
        and $\eta^*$ is the dual solution of $\TM_r^\pi V_r^\pi$.
        \item \label{thm: kl_s_normal} If $f(t)=t\log t$ in Example~\ref{eg: f-set} (KL balls), then the $(i,j)$th element of $\M^{\pi}$ is:
        \begin{align*}
            \M^\pi(i,j)=\mathbf{1}\{i=j\}-\gamma\sum_a \frac{\pi(a|s_i)P^*(s_j|s_i,a)\exp(-\frac{\pi(a|s_i)V_r^\pi(s_j)}{\lambda^*(P^*(\cdot|s_i,\cdot),V^\pi_r)})}{\sum_{\tilde s\in\SM} P^*(\tilde s|s_i,a)\exp(-\frac{\pi(a|s_i)V_r^\pi(\tilde s)}{\lambda^*( P^*(\cdot|s_i,\cdot),V_r^\pi)})},
        \end{align*}
         And $\sigma^2(P^*(\cdot|s,\cdot), V_r^\pi)=(b_{s}^\pi)^\top\Sigma_{s} b_{s}^\pi$ where for two pairs $(s_i, a_k)$ and $(s_j, a_l)$:
        \begin{align*}
            \Sigma_s((s_i,a_k),(s_j,a_l))&=\left(-P^*(s_i|s, a_k)\cdot P(s_j|s, a_l)+P(s_i|s, a_k)\mathbf{1}\{ s_i=s_j\}\right)\mathbf{1}\{a_k=a_l\},\\
            b_s( s_i,  a_k)&=-\frac{\lambda^*(P^*(\cdot|s,\cdot),V_r^\pi)\exp\left(-\frac{\pi(a_k|s)V_r^\pi(s_i)}{\lambda^*(P^*(\cdot|s, \cdot),V_r^\pi)}\right)}{\sum_{s'\in\SM}P(s'|a_k, s)\exp\left(-\frac{\pi(a_k|s)V_r^\pi(s')}{\lambda^*(P(\cdot|s,\cdot),V_r^\pi)}\right)},
        \end{align*}
        and $\lambda^*$ is the dual solution of $\TM_r^\pi V_r^\pi$.
    \end{enumerate}
\end{thm}

However, different from the $(s,a)$-rectangular setting, the optimal policies $\widehat{\pi}^*\in\argmax_{\pi}\widehat{V}_r^\pi$ and $\pi^*\in\argmax_{\pi}V_r^\pi$ could be stochastic. Thus, we can only obtain $\widehat{\pi}^*\toas \pi^*$ in the $s$-rectangular setting (Theorem~\ref{thm: sup_pi_consistent}) and can not just set $\pi=\pi^*$ in Theorems~\ref{thm: chi2_s_normal} and \ref{thm: kl_s_normal}. Fortunately, we could still obtain a result of asymptotic normality  of $\widehat{V}_r^*$ in Corollary~\ref{cor: sup_asymp_s}, where the details can be found in Supplementary Appendix 11.

\begin{thm}
    \label{thm: sup_pi_consistent}
    Assuming $\pi^*\in\argmax_\pi V_r^\pi$ is unique, we have that $\widehat{V}_r^*\toas V_r^*$ and $\widehat{\pi}^*\toas \pi^*$.
\end{thm}

\begin{cor}
    \label{cor: sup_asymp_s}
    Assuming $\pi^*$ is unique, we have:
    \begin{align*}
        \sqrt{n}\left(\max_\pi\widehat{V}_r^\pi(\mu)-\max_\pi V_r^\pi(\mu)\right)\tod\NM\left(0, \; \mu^\top(\M^{\pi^*})^{-1}\Lambda^{\pi^*}(\M^{\pi^*})^{-\top}\mu\right),
    \end{align*}
    where $\pi^*\in\argmax_\pi V_r^\pi(\mu)$,  where $\M^{\pi^*}$ and $\Lambda^{\pi^*}$ are defined in Theorem~\ref{thm: normal-sa}.
\end{cor}


\subsection{Interpretation of asymptotic variance} 
\label{subsec: discussasympvar}

The asymptotic variances of robust MDPs under different assumptions all share a similar expression $\mu^\top(\M^{\pi})^{-1}\Lambda^{\pi}(\M^{\pi})^{-\top}\mu$, where $\M^\pi$ and $\Lambda^\pi$ are determined by the choice of uncertainty sets. The term $\Lambda^\pi$ is actually the asymptotic variance of $\sqrt{n}\left(\widehat{\TM}_r^\pi V_r^\pi-\TM_r^\pi V_r^\pi\right)$, which reduces to the variance of empirical Bellman noise~\cite{li2021polyak} in the setting of  non-robust MDPs. Recall that $\M^\pi$ is the derivative of 
operator $I - \TM_r^\pi$ at the point $V_r^\pi$, which reduces to the matrix $I-\gamma P^\pi$ in the non-robust MDPs setting. In other words, the asymptotic variances in robust MDPs share a similar structure with the asymptotic variance in non-robust MDPs~\cite{jiang2016doubly, li2021polyak}. Besides, the asymptotic results imply the empirical robust value function converges to the true robust value function with a typical rate $\sqrt{n}$. Thus, a direct application of our asymptotic results is to construct a confidence interval for the robust value function, as long as we have a good estimation of asymptotic variances. Indeed, the asymptotic variances have explicit forms in our results (Theorems~\ref{thm: normal-sa} and \ref{thm: normal-s}). Thus, to construct a confidence interval from the dataset, we can plug empirical estimator $\widehat{P}$ and robust value function $\widehat{V}_r^*$ into the asymptotic variance. We leave the details in the experiments section.

%% file: tex/experiment.tex
To evaluate the statistical performance of robust MDPs, we conduct several numerical experiments in this section. We choose randomly generated MDPs as experiment environments. Under the $(s,a)$-rectangular setting, we run the classic algorithm Robust Value Iteration (RVI) \citep{iyengar2005robust} on random MDPs to show that we can obtain a near-optimal value function $\widehat{V}^*_r$ and policy $\widehat{\pi}\in\argmax\widehat{V}_r^\pi$ efficiently. Under the $s$-rectangular setting, we run the Bisection algorithm, which was proposed by \citet{ho2018fast}. The details of environments and algorithms are given in Supplementary Appendix 12.

\subsection{Convergence Guarantees}
\label{subsec: converge}

We first investigate the convergence performance of RVI on a random MDP, where $|\SM|=20, |\AM|=10$, $\gamma=0.9$. We leave the generation mechanism details in Supplementary Appendix 12. For every choice of $f\in\{L_1, \chi^2, \text{KL}\}$, $\rho\in\{0.1, 0.5, 1.0\}$ and $n\in\{10, 50, 100, 500, 1000\}$, we run RVI independently for $1000$ times and draw average performances in Fig~\ref{fig: converge_1}. In Fig~\ref{fig: converge_1}, the x-axis stands for the number of iteration steps and the y-axis stands for estimation error $\|V_t-V_r^*\|_\infty$, where the $V_t$ come from RVI. 

\begin{figure}[t!]
    \centering
    \begin{tabular}{cc}
        \includegraphics[width=.5\columnwidth]{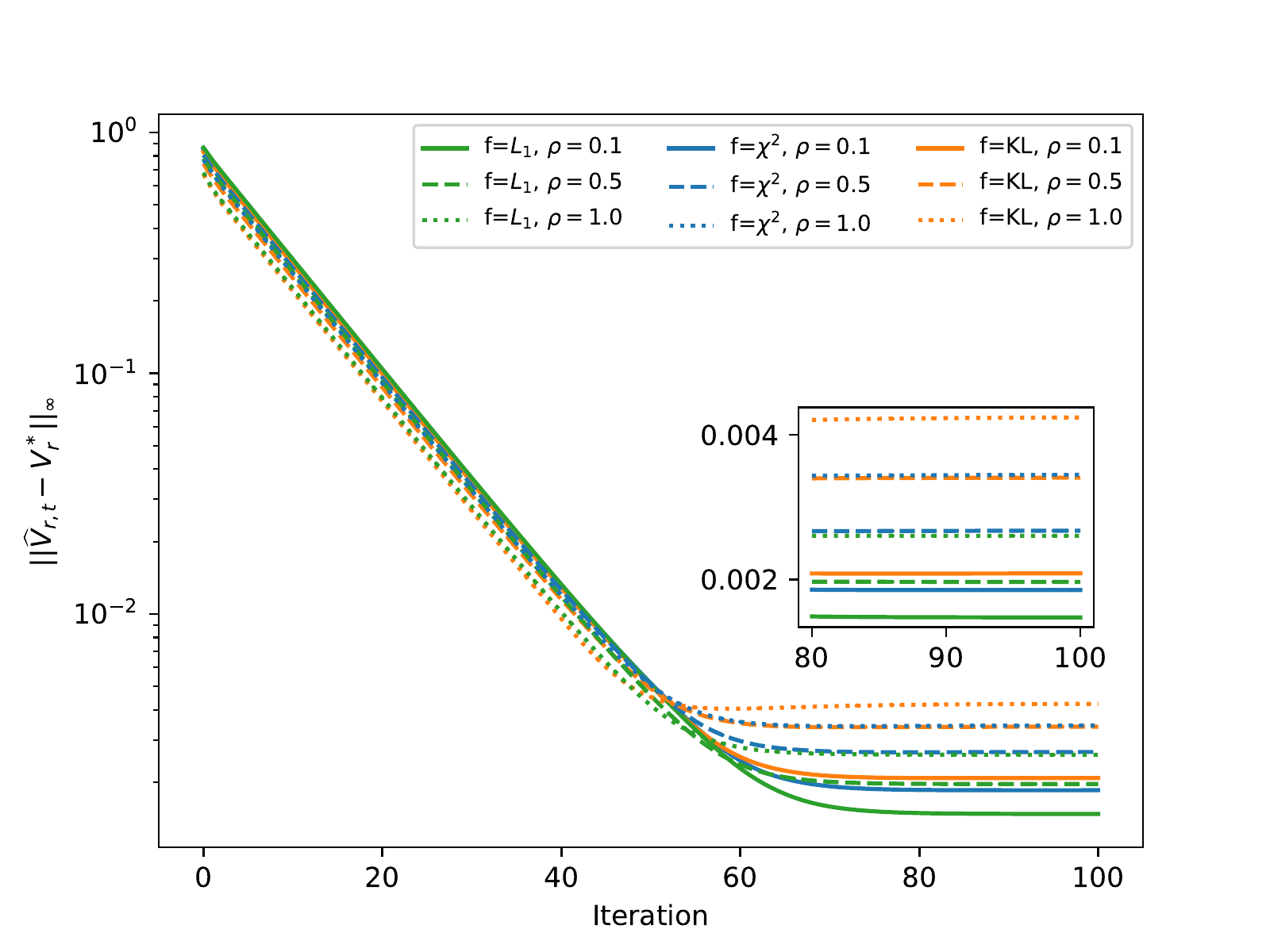}&
        \includegraphics[width=.5\columnwidth]{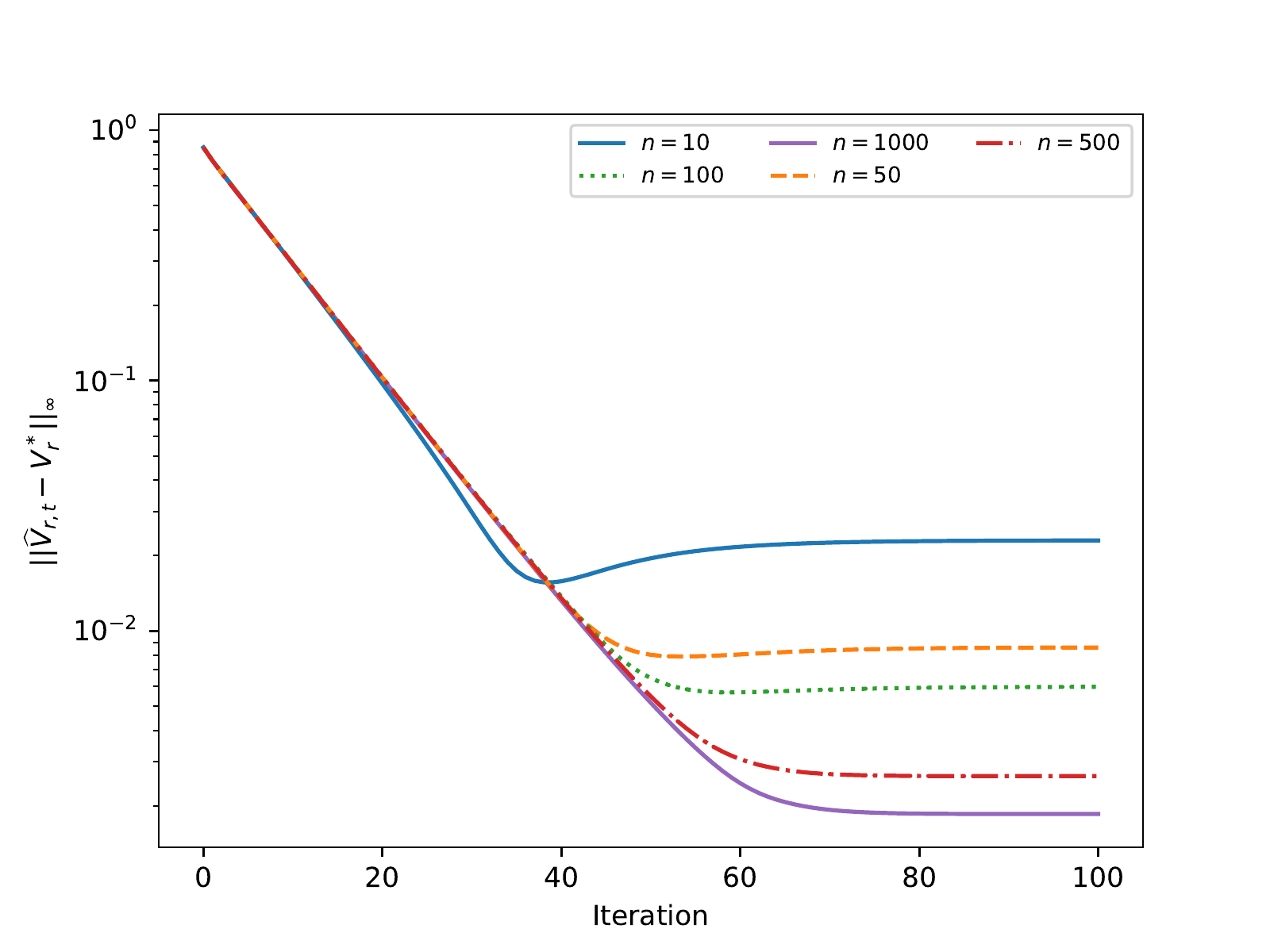}\\
        (a) $n=1000$ &(b) $f=\chi^2$, $\rho=0.1$
    \end{tabular}
    \caption{Convergence results of RVI under $(s,a)$-rectangular settings. (a): Results of all cases with $n=1000$. (b): Results when $f=\chi^2$ and $\rho=0.1$.}
    \label{fig: converge_1}
\end{figure}

In Fig~\ref{fig: converge_1}(a), we show the convergence results with all the cases. It can be observed that the convergence rate is linear at the first stage and then becomes stable at a certain error level for all the settings. Indeed, $V_t$ converges to $\widehat{V}_r^*$ at linear rate and there exist statistical errors between $\widehat{V}_r^*$ and $V_r^*$. Thus, the first stage in Fig~\ref{fig: converge_1}(a) is due to linear convergence rate of $\|V_t-\widehat{V}_r^*\|_\infty$ and the second stage is due to statistical error $\|\widehat{V}_r^*-V_r^*\|_\infty$. In Fig~\ref{fig: converge_1}(b), we set $f=\chi^2$,  $\rho=0.1$ and $n\in\{10, 50, 100, 500, 1000\}$. It is worth noting that the final performance is correlated with the choice of $n$. In fact, the statistical error $\|\widehat{P}_{s,a}-P_{s,a}\|_1$ is correlated with $n$. When $n$ is small, it is no wonder that the final performance is bad.

In addition, we run the Bisection algorithm on another random MDP with $|\SM|=|\AM|=5$, $\gamma=0.9$ under the $s$-rectangular setting. We choose a smaller MDP than the $(s,a)$-rectangular case by the fact that $s$-rectangular problems are more difficult to deal with. For every choice of $f\in\{\chi^2, \text{KL}\}$, $\rho\in\{0.05, 0.1, 0.5\}$ and $n\in\{10, 50, 100, 500, 1000\}$, we also run the Bisection algorithm independently for 1000 times and draw average performances in Fig~\ref{fig: converge_2}. In Fig~\ref{fig: converge_2}(a), we show the results with all the cases. In comparison with the $(s,a)$-rectangular setting, the final statistical errors vary less among the different choices of $f$. In Fig~\ref{fig: converge_2}(b), we choose $f=\chi^2$,  $\rho=0.1$ and $n\in\{10, 50, 100, 500, 1000\}$. From Fig~\ref{fig: converge_2}, it is easily  observed that the convergence performances are similar with the $s$-rectangular settings.
\begin{figure}[t!]
    \centering
    \begin{tabular}{ccc}
        \includegraphics[width=.5\columnwidth]{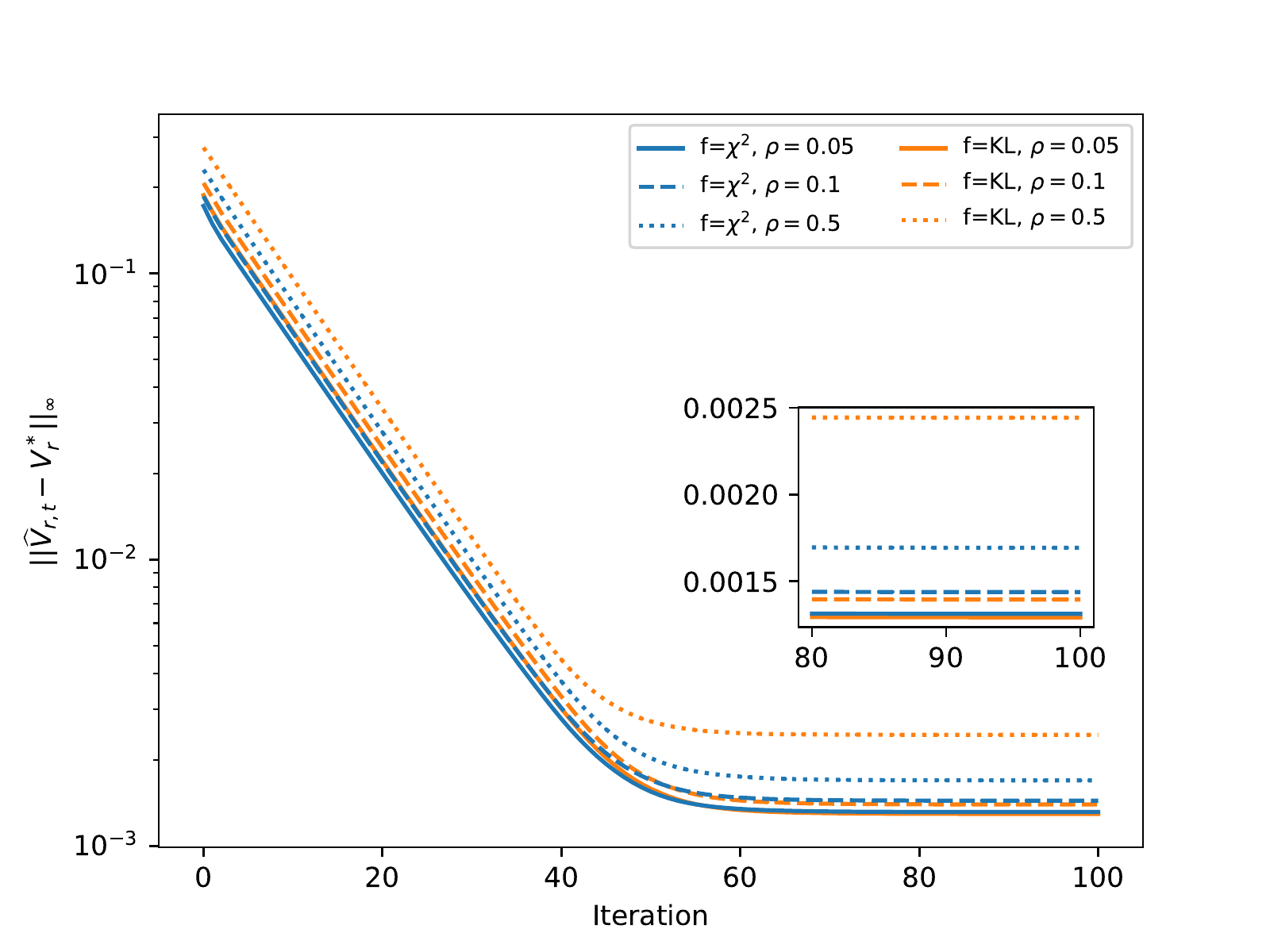}&
        \includegraphics[width=.5\columnwidth]{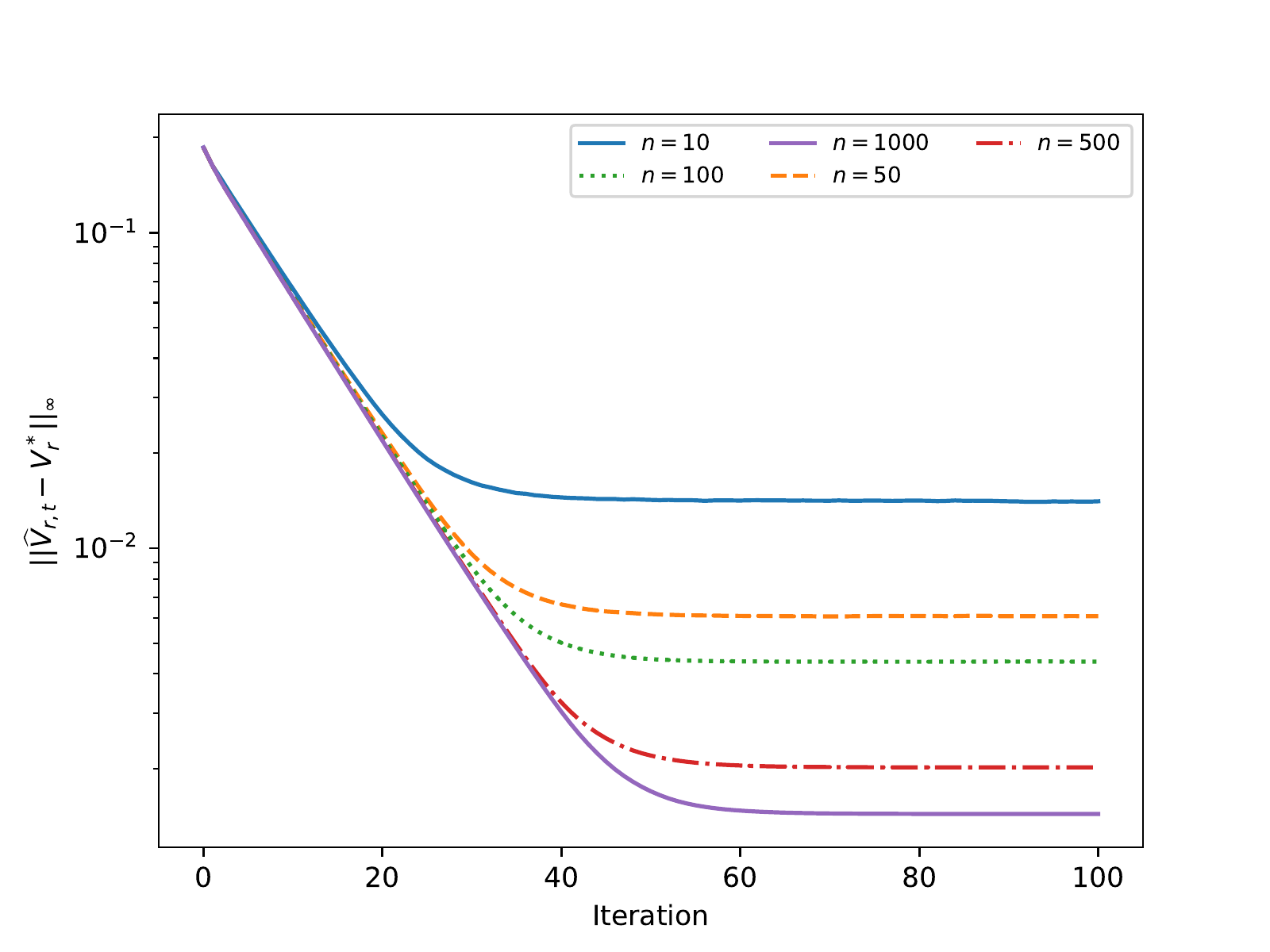}\\
        (a) $f\in\{\chi^2$, KL$\}$ &(b) $n\in\{10, 50, 100, 500, 1000\}$
    \end{tabular}
    \caption{Convergence results of Bisection Algorithm under $s$-rectangular settings. (a): Results of all cases with $n=1000$. (b): Results when $f=\chi^2$ and $\rho=0.1$.}
    \label{fig: converge_2}
\end{figure}

\subsection{Asymptotics}

Next, we follow the theoretical results from Section~\ref{sec: asymp} to make inference on $\widehat{V}_T(\mu)$ empirically under the same settings as in Section~\ref{subsec: converge}. First of all, based on RVI under the $(s,a)$-rectangular setting, we estimate $\widehat{\Lambda}$ and $\widehat{\M}^{\pi_T}$ with $\widehat{V}_T$ and obtain $\widehat{\sigma}=\sqrt{\mu^\top (\widehat{\M}^{\pi_T})^{-1}\widehat{\Lambda}(\widehat{\M}^{\pi_T})^{-\top}\mu}$, where $\Lambda$ and $\M^\pi$ are defined in Section~\ref{sec: asymp}. Then we are able to construct a confidence interval $\text{CI}_n(p)=[\widehat{V}_T - z_{p}\frac{\widehat{\sigma}}{\sqrt{n}}, \widehat{V}_T + z_{p}\frac{\widehat{\sigma}}{\sqrt{n}}]$, where $z_p$ is the $p$-quantile of the standard normal distribution $\NM(0,1)$. By the fact that $\widehat{\M}^\pi$ and $\widehat{\Lambda}$ are consistent (refer to the detailed proofs in Section~\ref{apd: asymp}),  we can safely say $\lim_{n\rightarrow\infty}\PB(V_r^*(\mu)\in\text{CI}_n(p))=1-2(1-p)$. To evaluate our theory, we test the empirical coverage rate in Table~\ref{tab: coverage} and Figure~\ref{fig: ci}, where we set $p=0.975$ and $z_p = 1.96$. We observe that the empirical coverage rate approximates the desired true coverage rate and the length of confidence interval decreases as the number of samples increases in all the cases. Interestingly, it seems that the length of confidence interval increases as $\rho$ increases.

\begin{figure}[t!]
    \centering
    \begin{tabular}{ccc}
        \includegraphics[width=.32\columnwidth]{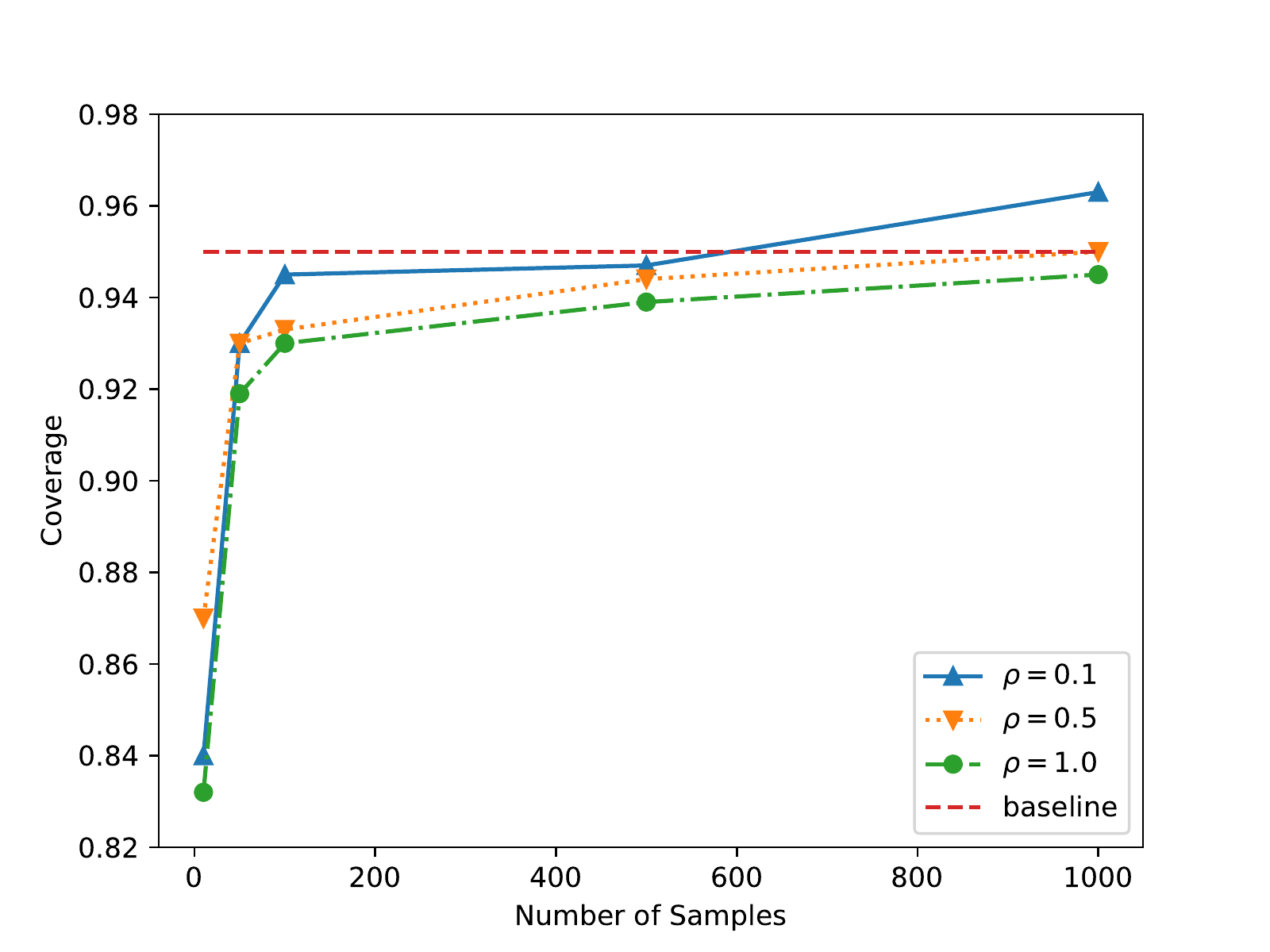}&
        \includegraphics[width=.32\columnwidth]{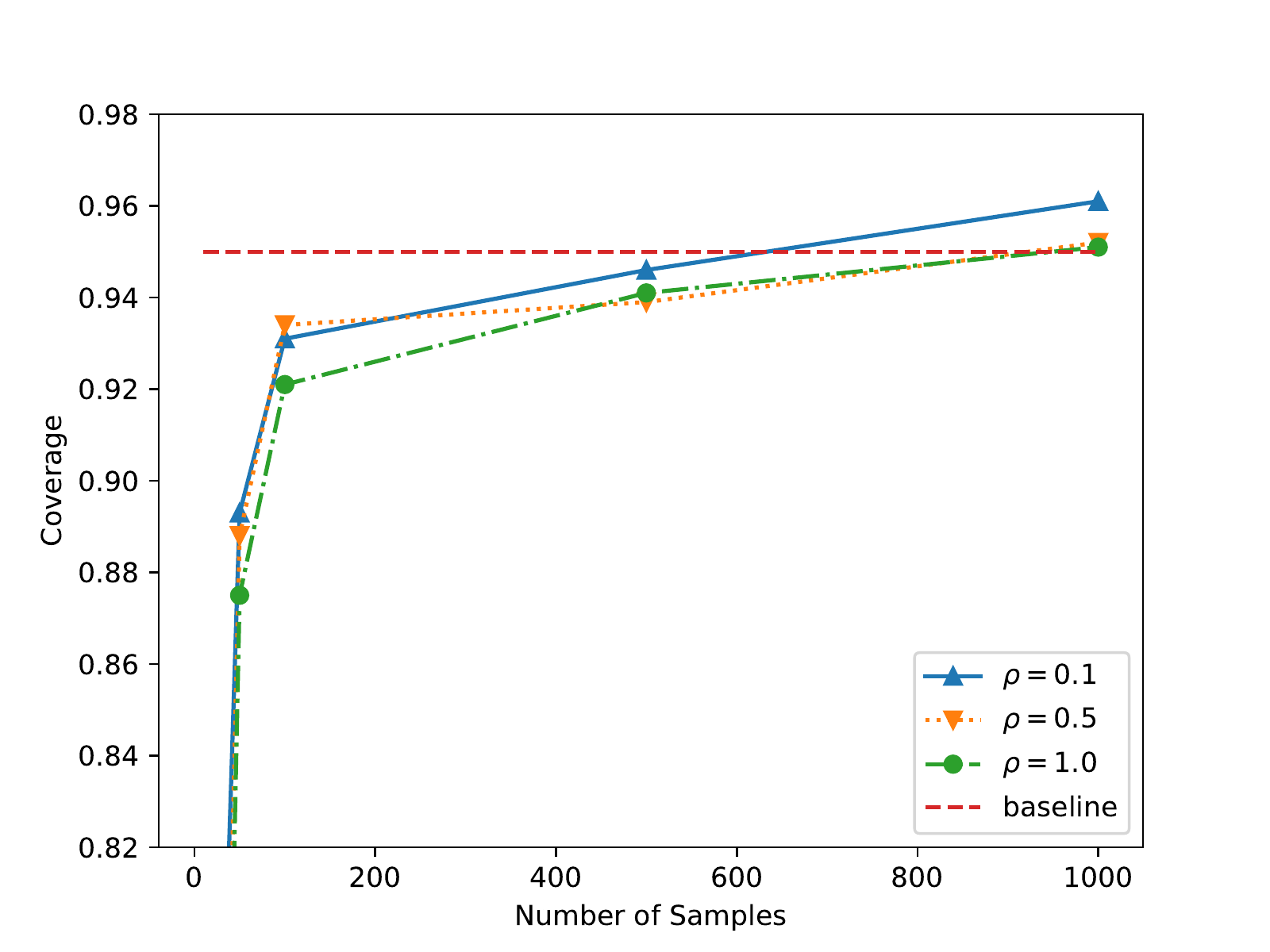}&
        \includegraphics[width=.32\columnwidth]{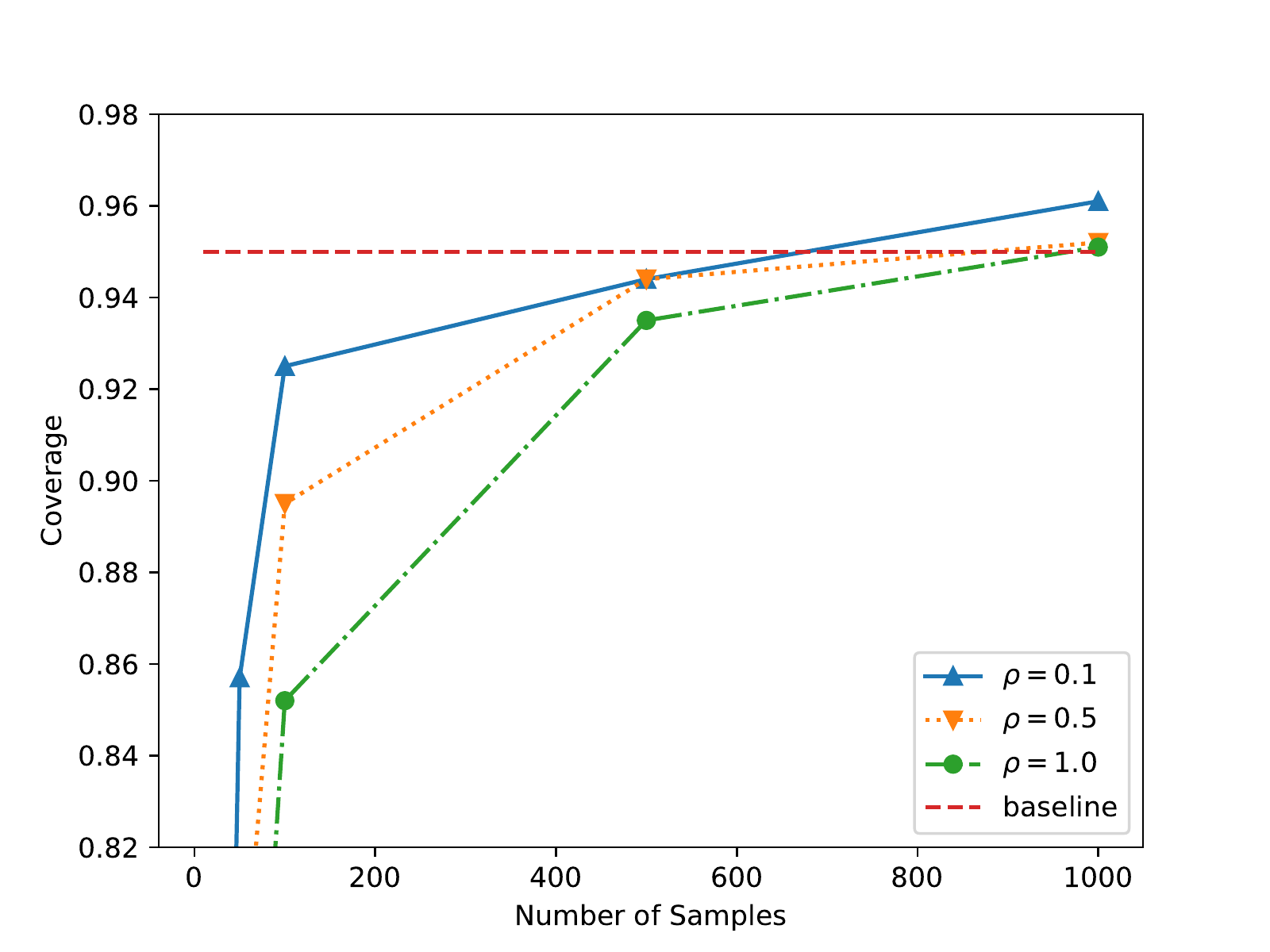}\\
        (a) CR ($f=L_1$) &(b) CR ($f=\chi^2$) & (c) CR ($f=\text{KL}$)\\
        \includegraphics[width=.32\columnwidth]{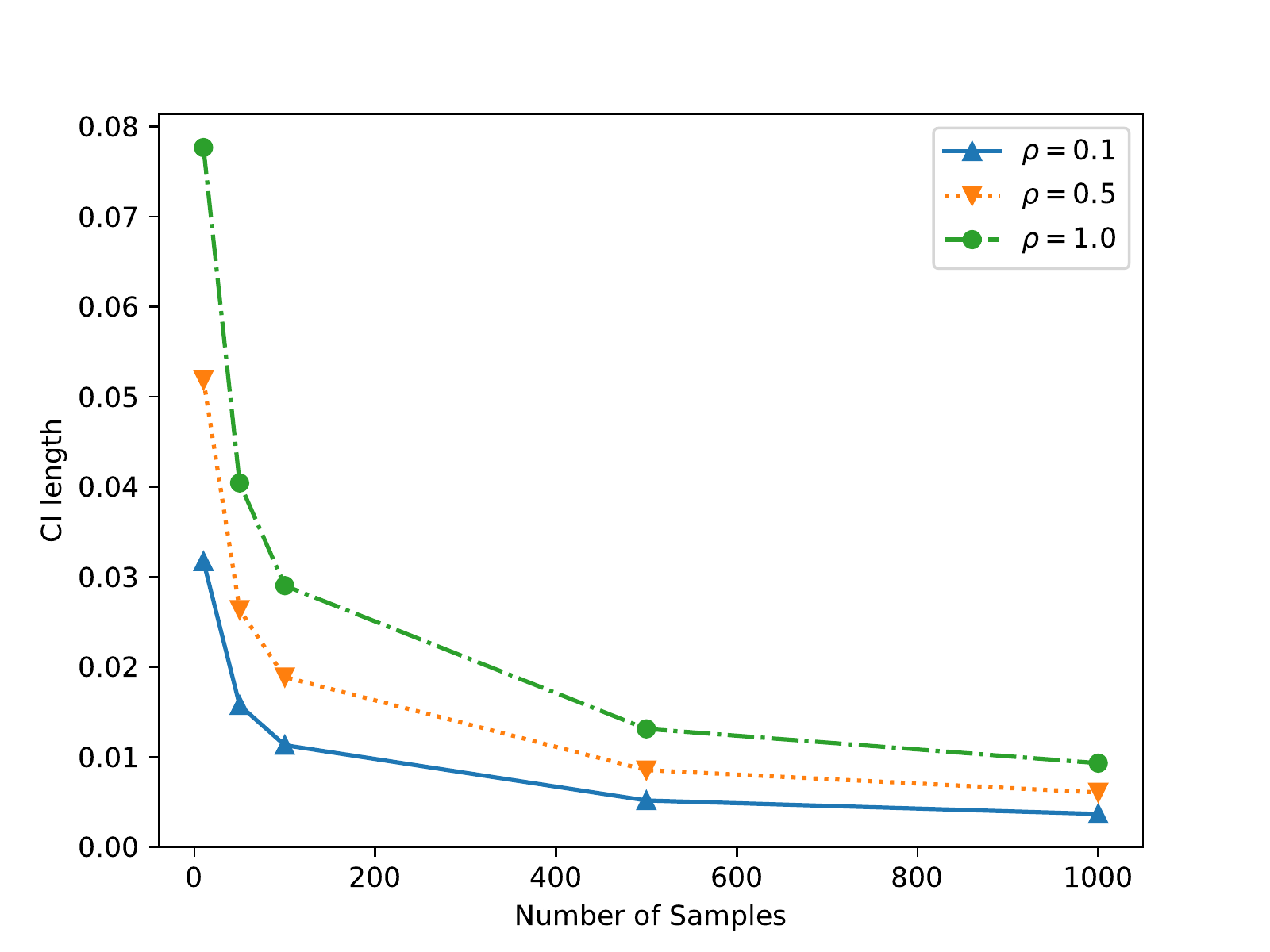}&
        \includegraphics[width=.32\columnwidth]{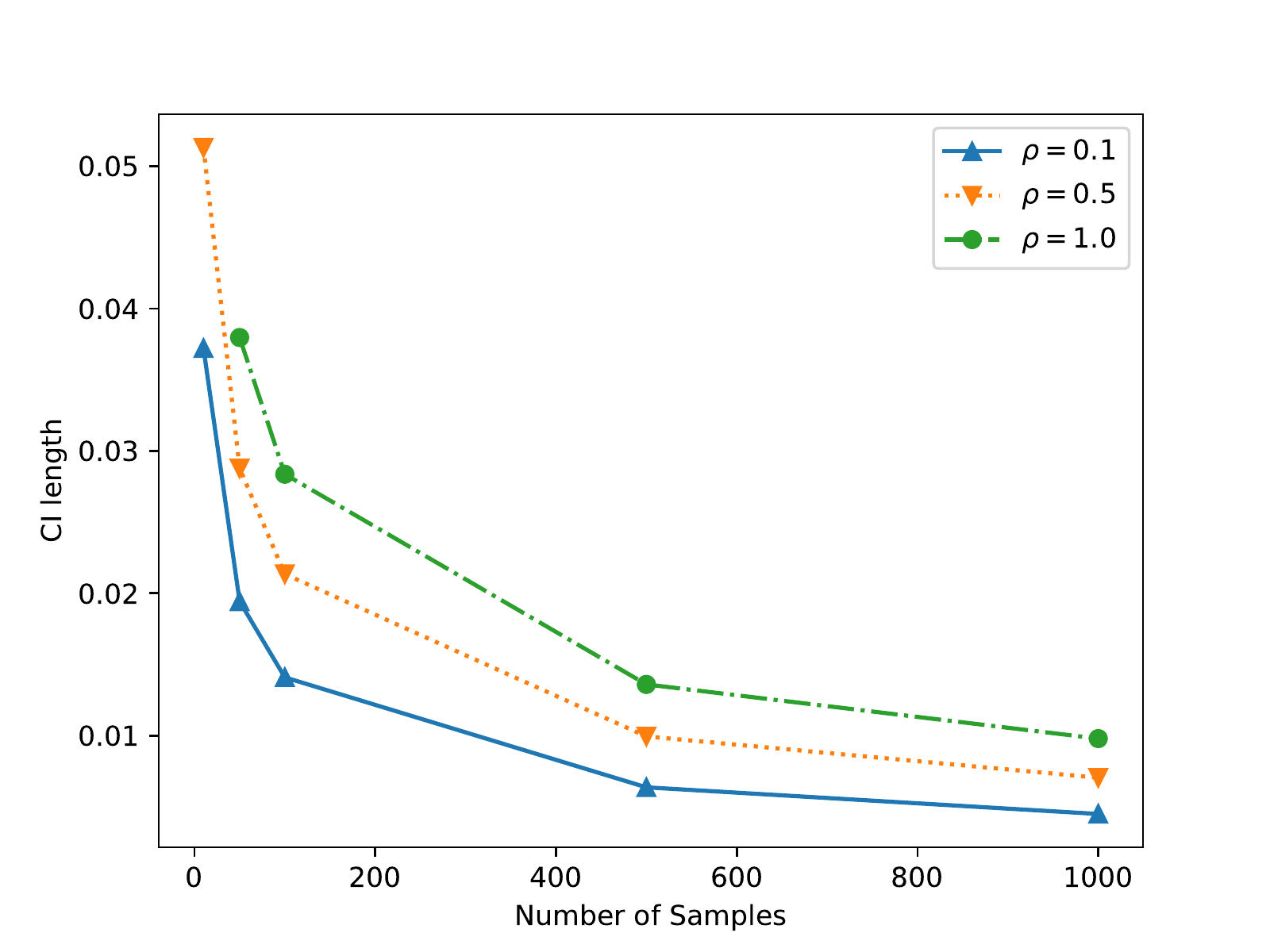}&
        \includegraphics[width=.32\columnwidth]{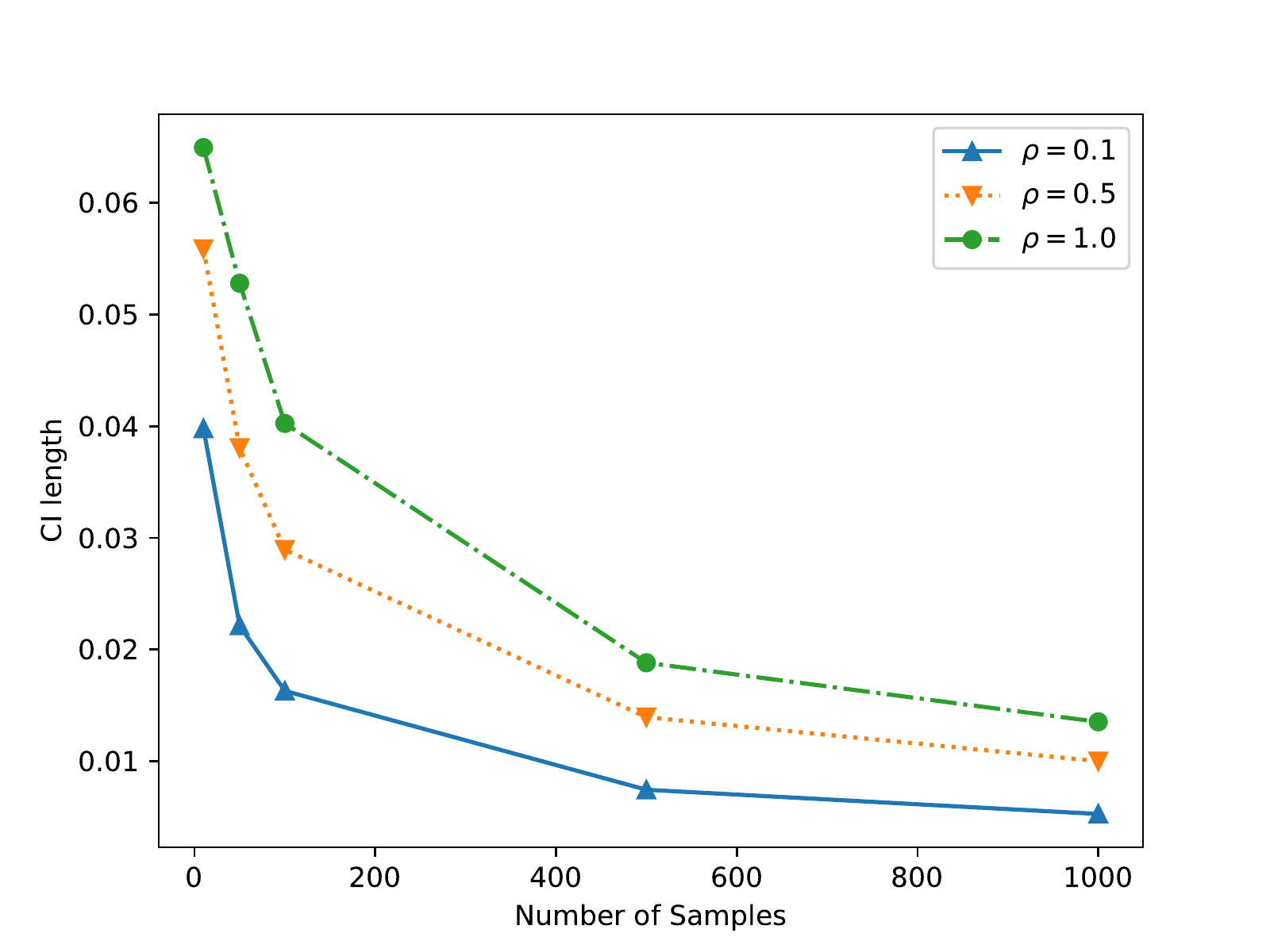}\\
        (d) CIL ($f=L_1$) &(e) CIL ($f=\chi^2$) & (f) CIL ($f=\text{KL}$)\\
    \end{tabular}
    \caption{Coverage rates (CR) and average CI lengths (CIL) under $(s,a)$-rectangular settings.}
    \label{fig: ci}
\end{figure}

\begin{table}[t!]
    \resizebox{\textwidth}{!}{%
    \begin{tabular}{c|c|ccc|ccc|ccc}
    \toprule
    \multicolumn{2}{c}{\multirow{2}{*}{Items}} & \multicolumn{3}{c|}{$n=10$} & \multicolumn{3}{c|}{$n=100$} & \multicolumn{3}{c}{$n=1000$} \\ \cmidrule{3-11} 
    \multicolumn{2}{l}{}  & $\rho=0.1$ & $\rho=0.5$ & $\rho=1.0$ & $\rho=0.1$ & $\rho=0.5$ & $\rho=1.0$ & $\rho=0.1$ & $\rho=0.5$ & $\rho=1.0$  \\ \midrule
    \multirow{3}{*}{\tabincell{c}{CR\\($\%$)}} & $L_1$ & 84.0(1.159) & 87.0(1.063) & 83.2(1.182) & 94.5(0.721) & 93.3(0.791) & 93.0(0.801) & 96.3(0.597) & 95.0(0.689) & 94.5(0.721) \\ 
    & $\chi^2$ & 64.3(1.515) & 54.9(1.574) & 50.1(1.581) & 93.1(0.801) & 93.4(0.785) & 92.1(0.853) & 96.1(0.612) & 95.2(0.676) & 95.1(0.683) \\ 
    & KL & 44.8(1.573) & 13.6(1.084) & 5.8(0.739) & 92.5(0.833) & 89.5(0.969) & 85.2(1.123) & 96.1(0.612) & 95.2(0.676) & 95.1(0.683) \\ \midrule
    \multirow{3}{*}{\tabincell{c}{CIL\\($10^{-2}$)}} & $L_1$ & 3.170(0.410) & 5.187(1.590) & 7.767(3.034) & 1.129(0.058) & 1.885(0.188) & 2.901(0.327) & 0.365(0.006) & 0.604(0.019) & 0.929(0.033) \\ 
    & $\chi^2$ & 3.722(0.469) & 5.131(0.851) & 6.411(1.505) & 1.409(0.075) & 2.135(0.163) & 2.836(0.291) & 0.450(0.008) & 0.705(0.019) & 0.979(0.049) \\ 
    & KL & 3.979(0.566) & 5.588(1.335) & 6.496(2.162) & 1.628(0.102) & 2.893(0.291) & 4.025(0.418) & 0.526(0.011) & 0.999(0.321) & 1.351(0.078)      \\ \bottomrule
    \end{tabular}%
    }
    \caption{Results of coverage rate (CR) and confidence interval length (CIL) under $(s,a)$-rectangular settings: The standard errors of CR $\widehat{p}$ are computed via $\sqrt{\widehat{p}(1-\widehat{p})/1000}\times100\%$ and reported inside the parentheses.}
    \label{tab: coverage}
\end{table}


Similarly, we also conduct experiments under the $s$-rectangular setting, where we choose $f\in\{\chi^2,\text{KL}\}$, $\rho\in\{0.05, 0.1, 0.5\}$ and $n\in\{10, 50,100,500,1000\}$, and run the Bisection algorithm on the random MDP ($|\SM|=|\AM|=5$) for $1000$ times. 
We also conclude the coverage under the $s$-rectangular setting in Table~\ref{tab: coverage_s} and Figure~\ref{fig: ci_s}, where the results of the empirical coverage  meet our expectation.

\begin{figure}[t!]
    \centering
    \begin{tabular}{ccc}
        \includegraphics[width=.32\columnwidth]{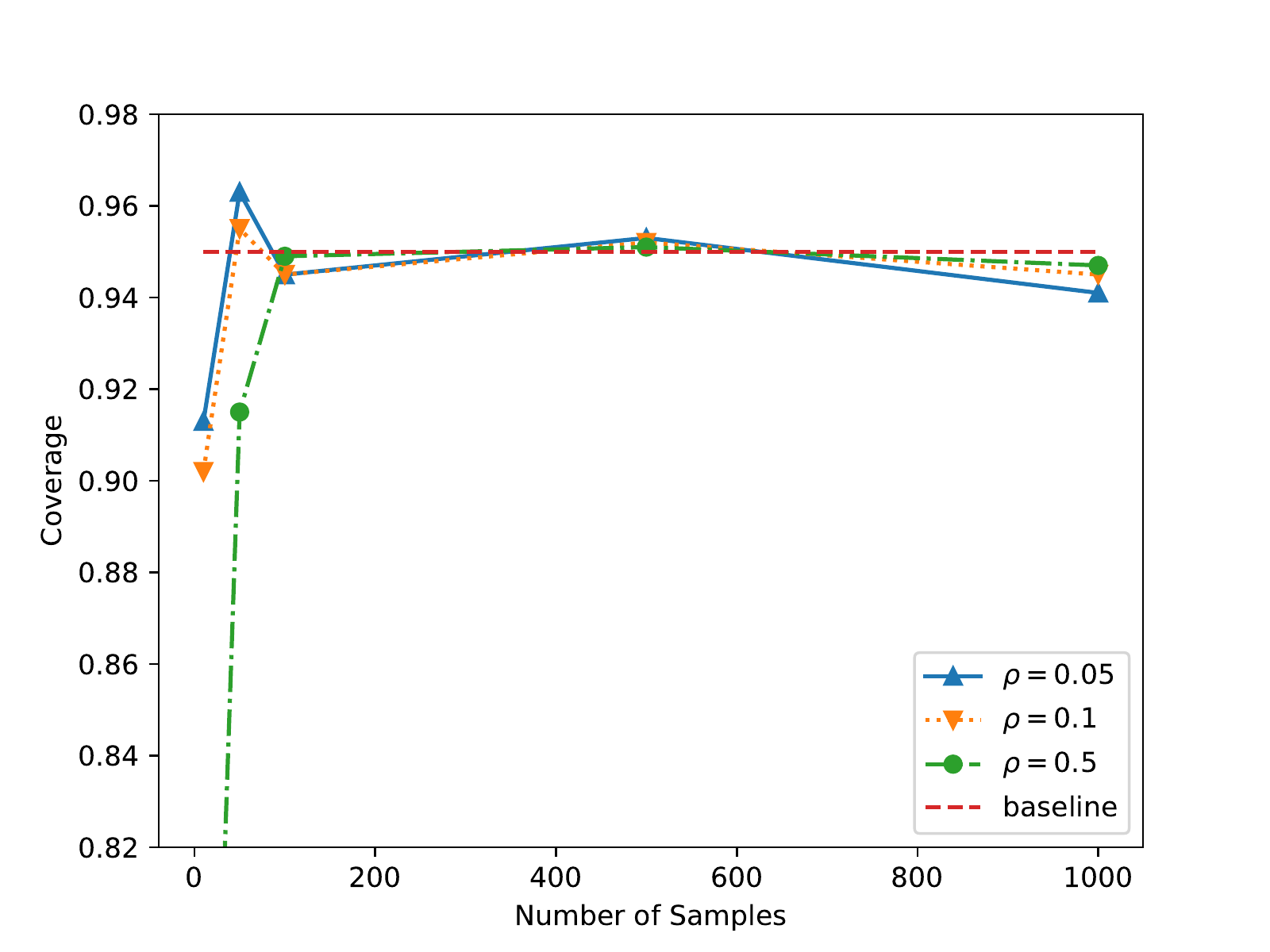}&
        \includegraphics[width=.32\columnwidth]{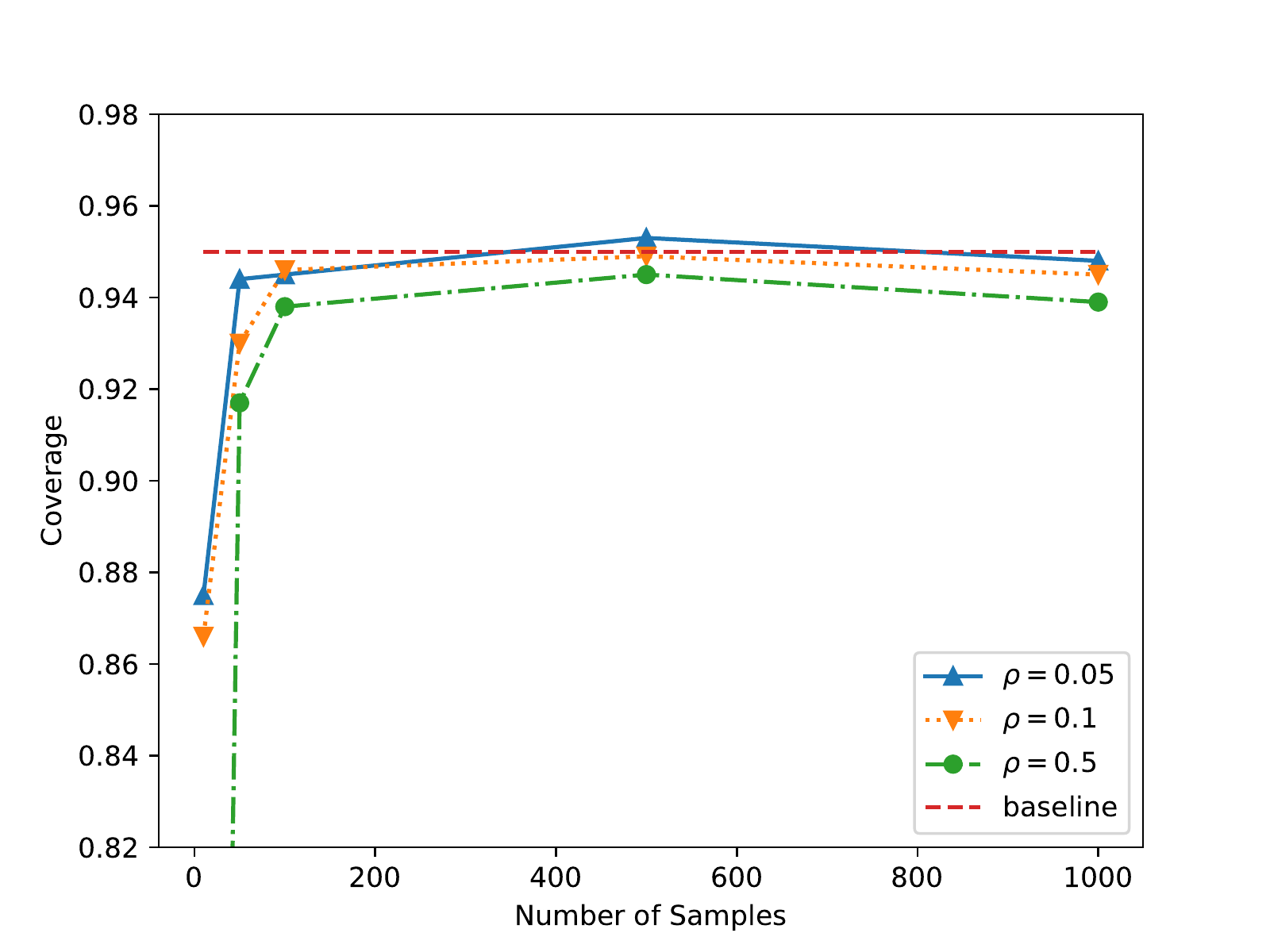}\\
        (a) CR ($f=\chi^2$) & (b) CR ($f=\text{KL}$)\\
        \includegraphics[width=.32\columnwidth]{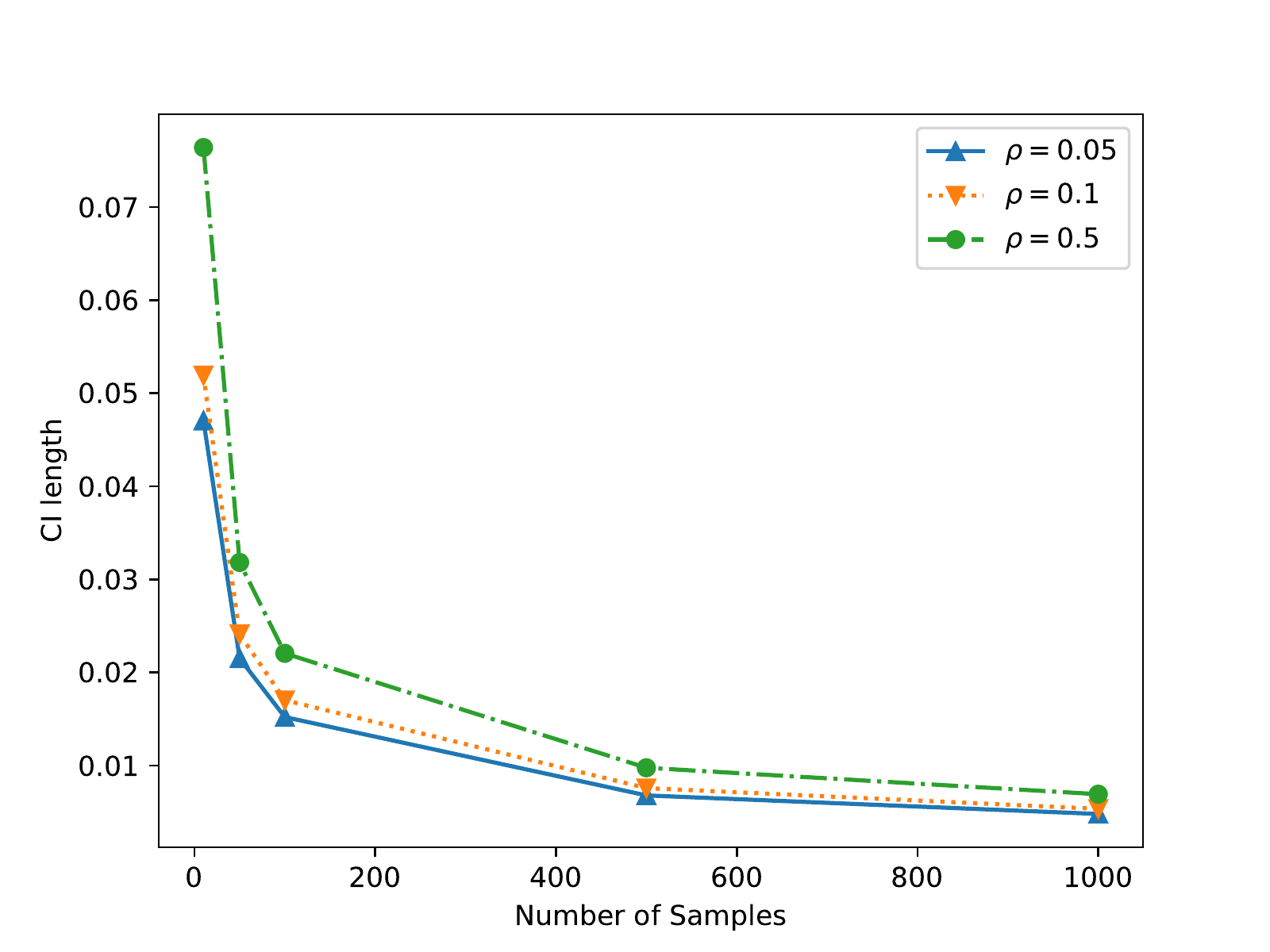}&
        \includegraphics[width=.32\columnwidth]{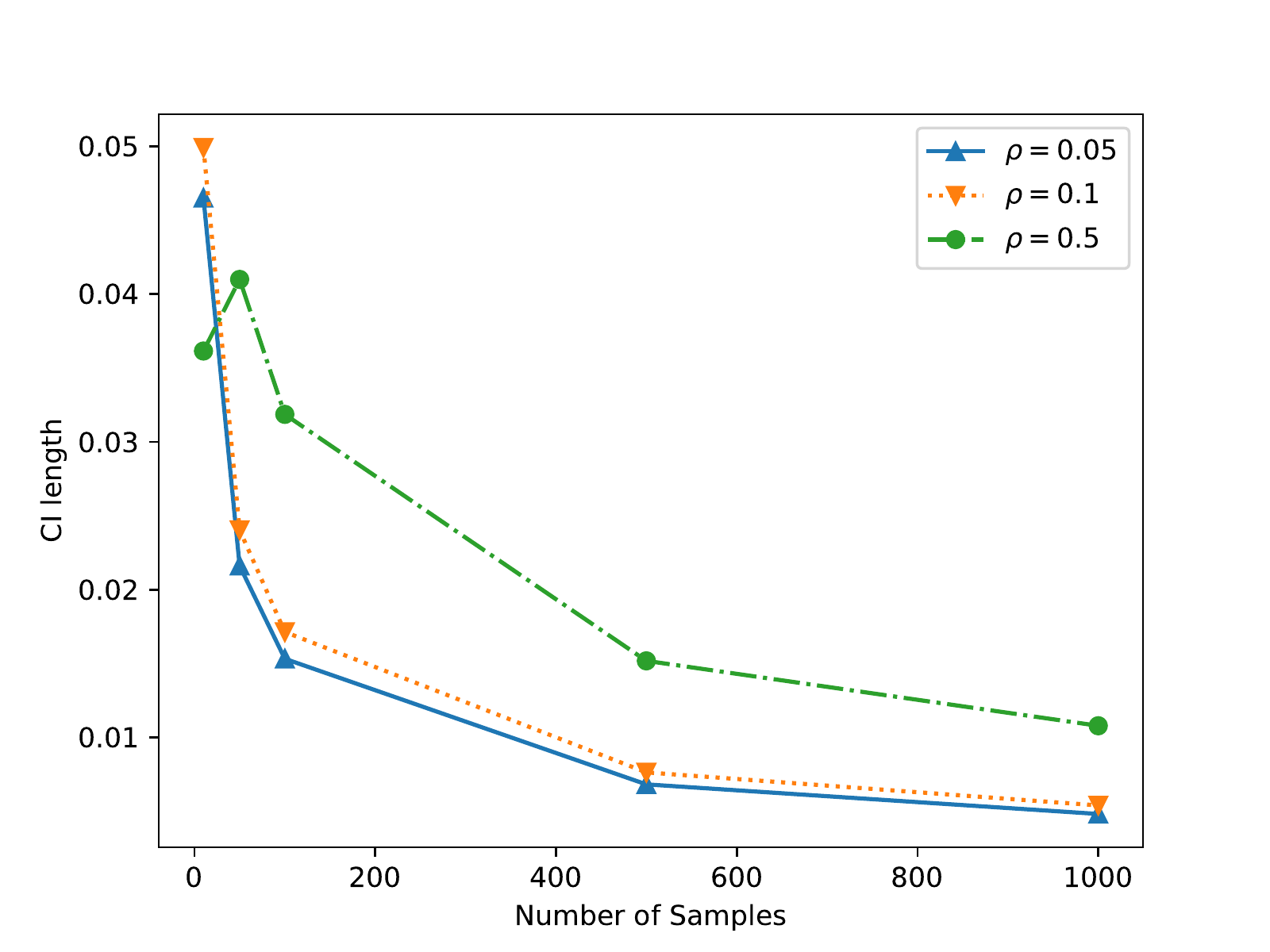}\\
        (c) CIL ($f=\chi^2$) & (d) CIL ($f=\text{KL}$)\\
    \end{tabular}
    \caption{Coverage rates (CR) and average CI lengths (CIL) under $s$-rectangular settings.}
    \label{fig: ci_s}
\end{figure}

\begin{table}[t!]
    \resizebox{\textwidth}{!}{%
    \begin{tabular}{c|c|ccc|ccc|ccc}
    \toprule
    \multicolumn{2}{c}{\multirow{2}{*}{Items}} & \multicolumn{3}{c|}{$n=10$} & \multicolumn{3}{c|}{$n=100$} & \multicolumn{3}{c}{$n=1000$} \\ \cmidrule{3-11} 
    \multicolumn{2}{l}{}  & $\rho=0.05$  & $\rho=0.1$ & $\rho=0.5$    & $\rho=0.05$      & $\rho=0.1$      &  $\rho=0.5$     & $\rho=0.05$      & $\rho=0.1$      &  $\rho=0.5$    \\ \midrule
    \multirow{2}{*}{\tabincell{c}{CR\\($\%)$}}  & $\chi^2$ & 91.3(0.891) & 90.2(0.940) & 68.2(1.473) & 94.5(0.721) & 94.5(0.721) & 94.9(0.696) & 94.1(0.745) & 94.5(0.721) & 94.7(0.708)   \\ 
    &  KL & 87.5(1.046) & 86.6(1.077) & 41.9(1.560) & 94.5(0.721) & 94.6(0.715) & 93.8(0.763) & 94.8(0.702) & 94.5(0.721) & 93.9(0.757)      \\ \midrule
    \multirow{2}{*}{\tabincell{c}{CIL\\($10^{-2}$)}} & $\chi^2$ & 4.704(0.931) & 5.193(1.196) & 7.639(21.657) & 1.522(0.081) & 1.703(0.098) & 2.206(0.336) & 0.481(0.008) & 0.536(0.010) & 0.693(0.031)      \\ 
    &  KL   & 4.650(1.086) & 4.991(1.545) & 3.615(1.720) & 1.533(0.093) & 1.715(0.140) & 3.186(0.612) & 0.484(0.009) & 0.542(0.014) & 1.080(0.046) \\ \hline
    \end{tabular}%
    }
    \caption{Results of coverage rate (CR) and confidence interval length (CIL) under $s$-rectangular settings: The standard errors of CR $\widehat{p}$ are computed via $\sqrt{\widehat{p}(1-\widehat{p})/1000}\times100\%$ and reported inside the parentheses.}
    \label{tab: coverage_s}
\end{table}

%% file: tex/conclusion.tex
In this paper we have studied robust MDPs, which are the foundation of robust RL problems.   
Our primary concern  focuses on the statistical performances  of the optimal robust policy and value function obtained from empirical estimation, including finite-sample results and asymptotics based on the most commonly used uncertainty sets: $L_1$, $\chi^2$, and KL balls. In particular, we have shown that with a polynomial number of samples in the dataset, the performance gap can be controlled well under both the $(s,a)$ and $s$-rectangular assumptions. Furthermore, we have also shown that the empirical robust optimal value function converges with rate $O_P(1/\sqrt{n})$ and converges to a normal distribution, from which we are able to make inferences from the estimators. 

However, some issues still remain open. Firstly, in this paper, the size of the uncertainty set is chosen to be controlled by a positive constant parameter $\rho>0$. The finite-sample results in Section~\ref{sec: gen} tell us that a proper choice of $\rho$ can reduce the sample complexity. There are some prior works \citep{petrik2019beyond,derman2020distributional},  mentioning that the size of the uncertainty set could also be controlled by a shrinking parameter $\rho_n$ (such as $\rho/\sqrt{n}$), whose statistical properties are still unclear. Thus, understanding the adaptive choice of $\rho$ is a vital topic in future research. 

Secondly, in terms of finite-sample results, it is worth noting that there still exists a gap between upper bounds and lower bounds regarding factors $|\SM|$, $|\AM|$ and $1/(1-\gamma)$. Improving the dependence of these parameters is also a significant research direction.

Last but not the least, in the context of asymptotics, we have proved that $\widehat{V}_r^*(\mu)$ is asymptotically normal with rate $\sqrt{n}$ in both the $(s,a)$ and $s$-rectangular assumptions. Under the $(s,a)$-rectangular assumption, the empirical optimal robust policy $\widehat{\pi}^*\in\argmax_\pi \widehat{V}_r^\pi(\mu)$ is exactly the same as $\pi^*\in\argmax_\pi V_r^\pi(\mu)$ when the sample size $n$ is large enough. However, under the $s$-rectangular assumption, we only know that $\widehat{\pi}^*$ converges to $\pi^*$ almost surely without a specific convergence rate. According to \citet{van2000asymptotic}, we argue that if we could have a precise estimate of  $\beta$ in the following equation 
\[
    \EB \sup_{d(\pi_1, \pi_2)<\delta}\sqrt{n}\left|\widehat{V}_r^{\pi_1}(\mu)-V_r^{\pi_1}(\mu)-\widehat{V}_r^{\pi_2}(\mu)+V_r^{\pi_2}(\mu)\right|\le C\delta^\beta,
\]
the convergence rate of $\widehat{\pi}^*$ then could be determined, and then making inference for $\widehat{\pi}^*$ become possible. We would leave it to future work.

%% file: tex/appendix.tex
\begin{appendix}
	\onecolumn
	\begin{center}
		{\huge \textbf{Appendix}}
	\end{center}

In this supplementary material, we provide proofs of theoretical results and experiment details in our manuscript.


\section{Proofs of Section~\ref{sec: preli}}
\input{tex/apdix/preli_proof.tex}

\section{Proofs of Section~\ref{sec: gen}}
\input{tex/apdix/finite_gen.tex}
\subsection{Proofs of Section~\ref{sec: lb}}
\label{apd: lb}
\input{tex/apdix/finite_lb.tex}

\section{Non-asymptotic results with an offline dataset}
\label{apd: res-off}
\input{tex/res-off.tex}

\section{Proofs of Section~\ref{apd: res-off}}
\input{tex/apdix/finite_off.tex}

\section{Proofs of Section~\ref{sec: asymp}}
\label{apd: asymp}
\input{tex/apdix/asymp_proof.tex}

\section{Details of Experiments}
\label{apd: exp}
\input{tex/apdix/exp.tex}

\end{appendix}

%% file: tex/apdix/preli_proof.tex
\begin{proof}[Proof of Theorem~\ref{thm: non-robust-difference}]
	By Bellman equation, we have $V_r^\pi = R^\pi+\gamma\inf_{P\in\PM}P^\pi V_r^\pi$ and $V^\pi = R^\pi+\gamma (P^*)^\pi V^\pi$. Thus,
	\begin{align*}
		\left\|V_r^\pi-V^\pi\right\|_\infty&=\gamma\left\|\inf_{P\in\PM}P^\pi V_r^\pi-(P^*)^\pi V^\pi\right\|_\infty\\
		&\le\gamma\left\|\inf_{P\in\PM}P^\pi V_r^\pi - \inf_{P\in\PM}P^\pi V^\pi \right\|_\infty+\gamma\left\|\inf_{P\in\PM}P^\pi V^\pi-(P^*)^\pi V^\pi\right\|_\infty\\
		&\le \gamma\left\|V_r^\pi-V^\pi\right\|_\infty+\gamma\left\|\inf_{P\in\PM}P^\pi V^\pi-(P^*)^\pi V^\pi\right\|_\infty.
	\end{align*}
	Rearranging the above inequality, we have:
	\begin{align*}
		\left\|V_r^\pi-V^\pi\right\|_\infty&\le\frac{\gamma}{1-\gamma}\left\|\inf_{P\in\PM}P^\pi V^\pi-(P^*)^\pi V^\pi\right\|_\infty\\
		&=\frac{\gamma}{1-\gamma}\max_{s}\left|\inf_{P(\cdot|s,\cdot)\in\PM_s}\sum_{s',a}\pi(a|s)\left(P(s'|s,a)-P^*(s'|s,a)\right)V^\pi(s')\right|\\
		&\le\frac{\gamma}{(1-\gamma)^2}\max_s\sup_{P(\cdot|s,\cdot)\in\PM_s}\sum_{s',a}\pi(a|s)\left|P(s'|s,a)-P^*(s'|s,a)\right|,
	\end{align*}
	where the last inequality is applied by $\|V^\pi\|_\infty\le1/(1-\gamma)$. 
	\paragraph*{\textbf{On Example~\ref{eg: f-set}}} When the $(s,a)$-rectangular assumption is made, we have:
	\begin{align*}
		\left\|V_r^\pi-V^\pi\right\|_\infty&\le\frac{\gamma}{(1-\gamma)^2}\max_s\sum_{a}\pi(a|s)\sup_{P(\cdot|s,a)\in\PM_{s,a}}\sum_{s'}\left|P(s'|s,a)-P^*(s'|s,a)\right|\\
		&\le\frac{\gamma}{(1-\gamma)^2}\max_{s,a}\sup_{P(\cdot|s,a)\in\PM_{s,a}}\sum_{s'}\left|P(s'|s,a)-P^*(s'|s,a)\right|\\
		&\le\frac{\gamma h(\rho)}{(1-\gamma)^2},
	\end{align*}
	where the final step holds because $h$ is monotonically increasing.

	\paragraph*{\textbf{On Example~\ref{eg: f-set-s}}} When the $s$-rectangular assumption is made, we have:
	\begin{align*}
		\left\|V_r^\pi-V^\pi\right\|_\infty&\overset{(a)}{\le}\frac{\gamma}{(1-\gamma)^2}\max_s \sup_{P(\cdot|s,\cdot)\in\PM_s}\sum_{s',a}|P(s'|s,a)-P^*(s'|s,a)|\\
		&\overset{(b)}{\le}\frac{\gamma}{(1-\gamma)^2}\max_s \sup_{P(\cdot|s,\cdot)\in\PM_s}\sum_a h\left(\sum_{s'}f\left(\frac{P(s'|s,a)}{P^*(s'|s,a)}\right)P^*(s'|s,a)\right)\\
		&\overset{(c)}{\le}\frac{\gamma}{(1-\gamma)^2}\max_s\sup_{P(\cdot|s,\cdot)\in\PM_s}|\AM|\cdot h\left(\frac{1}{|\AM|}\sum_{s',a}f\left(\frac{P(s'|s,a)}{P^*(s'|s,a)}\right)P^*(s'|s,a)\right)\\
		&\overset{(d)}{\le}\frac{\gamma |\AM|h(\rho)}{(1-\gamma)^2},
	\end{align*}
	where  (a) we eliminate $\pi(a|s)$ because $\pi(a|s)\le 1$,   (b) we use  Inequality~\eqref{cond: f_h},  (c) we apply Jensen's Inequality by $h(\cdot)$ is concave, and  (d) is due to the fact $h$ is monotonically increasing.
	
\end{proof}

\begin{proof}[Proof of Lemma~\ref{lem: uni-dev-pi}]
	The left hand inequality is trivial. Next we prove the right hand inequality.
	\begin{align*}
		\max_\pi V_r^\pi(\mu)-V_r^{\widehat{\pi}}(\mu)&=\max_\pi V_r^\pi(\mu)-\max_\pi \widehat{V}_r^\pi+\max_\pi \widehat{V}_r^\pi-V_r^{\widehat{\pi}}(\mu)\\
		&\le\max_\pi\left|V_r^\pi(\mu)-\widehat{V}_r^\pi(\mu)\right|+\left(\widehat{V}_r^{\widehat{\pi}}(\mu)-V_r^{\widehat{\pi}}(\mu)\right)\\
		&\le2\sup_{\pi\in\Pi}\left|V_r^\pi(\mu)-\widehat{V}_r^\pi(\mu)\right|.
	\end{align*}
\end{proof}

\begin{proof}[Proof of Lemma~\ref{lem: uni-dev-v}]
	Noting that $V_r^\pi$ and $\widehat{V}_r^\pi$ are fixed points of $\TM_r^\pi$ and $\widehat{\TM}_r^\pi$, we have:
	\begin{align*}
		\left\|V_r^\pi-\widehat{V}_r^\pi\right\|_\infty&=\left\|\TM_r^\pi V_r^\pi-\widehat{\TM}_r^\pi\widehat{V}_r^\pi\right\|_\infty\\
		&\le\left\|\TM_r^\pi V_r^\pi-\widehat{\TM}_r^\pi V_r^\pi\right\|_\infty+\left\|\widehat{\TM}_r^\pi V_r^\pi-\widehat{\TM}_r^\pi \widehat{V}_r^\pi\right\|_\infty\\
		&\le\sup_{V\in\VM}\left\|\TM_r^\pi V - \widehat{\TM}_r^\pi V\right\|_\infty+\gamma\left\|V_r^\pi-\widehat{V}_r^\pi\right\|_\infty.
	\end{align*}
	Arranging both the sides, we obtain:
	\begin{align*}
		\left\|V_r^\pi-\widehat{V}_r^\pi\right\|_\infty\le\frac{1}{1-\gamma}\sup_{V\in\VM}\left\|\TM_r^\pi V - \widehat{\TM}_r^\pi V\right\|_\infty.
	\end{align*}
\end{proof}

\begin{proof}[Proof of Lemma~\ref{lem: eps-v}]
	To simplify, we denote $G^\pi_r V=\TM^\pi_r V-\widehat{\TM}^\pi_r V$. For any $V\in\VM$, there exists $V_0\in\VM_\varepsilon$ such that $\|V-V_0\|_\infty\le\varepsilon$. Thus, we have:
	\[
		\left\|G_r^\pi V\right\|_\infty\le \|G_r^\pi V- G_r^\pi V_0\|_\infty + \|G_r^\pi V_0\|_\infty.
	\]
	Noting that $|\TM_r^\pi V(s) - \TM_r^\pi V_0(s)|\le \gamma\varepsilon$, we conclude that:
	\[
		\left\|G_r^\pi V\right\|_\infty\le2\gamma\varepsilon + \sup_{V\in\VM_\varepsilon}\|G_r^\pi V\|_\infty.
	\]
	Finally, taking supremum over $V$ at LHS, we obtain our results.
\end{proof}

\begin{proof}[Proof of Lemma~\ref{lem: eps-pi}]
	Similar with Lemma~\ref{lem: eps-v}, we denote $G^\pi_r V=\TM^\pi_r V-\widehat{\TM}^\pi_r V$ for any $V\in\VM$. For any $\pi\in\Pi$, there exists $\pi_0\in\NM(\Pi,\|\cdot\|_1,\varepsilon)$, such that $\|\pi(\cdot|s)-\pi_0(\cdot|s)\|_1\le\varepsilon$ for all $s\in\SM$. Thus, we have:
	\begin{align*}
		\sup_{V\in\VM}\left\|G^\pi_r V\right\|_\infty\le\sup_{V\in\VM}\left\|G^\pi_r V-G^{\pi_0}_r V\right\|_\infty+\sup_{V\in\VM}\left\|G^{\pi_0}_r V\right\|_\infty.
	\end{align*}
	For any fixed $V$, we have:
	\begin{align*}
		\left|G_r^\pi V (s) -G_r^{\pi_0} V(s)\right|&\le\gamma\sup_{P\in\PM(s)}\left|\sum_{s'\in\SM}(P^\pi(s'|s)-P^{\pi_0}(s'|s))V(s')\right|\\
		&+\gamma\sup_{P\in\widehat{\PM}(s)}\left|\sum_{s'\in\SM}(P^\pi(s'|s)-P^{\pi_0}(s'|s))V(s')\right|\\
		&\le \frac{2\gamma}{1-\gamma}\left\|\pi(\cdot|s)-\pi_0(\cdot|s)\right\|_1,
	\end{align*}
	where the last step is due to:
	\begin{align*}
		\left|\sum_{s'\in\SM}(P^\pi(s'|s)-P^{\pi_0}(s'|s))V(s')\right|
		&\le\frac{1}{1-\gamma}\sum_{s'\in\SM}\left|(P^\pi(s'|s)-P^{\pi_0}(s'|s))\right|\\
		&\le\frac{1}{1-\gamma}\sum_{s'\in\SM,a\in\AM}P(s'|s,a)\left|\pi(a|s)-\pi_0(a|s)\right|\\
		&=\frac{1}{1-\gamma}\|\pi(\cdot|s)-\pi_0(\cdot|s)\|_1.
	\end{align*}
	Thus, we have:
	\begin{align*}
		\sup_{V\in\VM}\left\|G^\pi_r V\right\|_\infty\le\frac{2\gamma\varepsilon}{1-\gamma}+\sup_{V\in\VM}\left\|G^{\pi_0}_r V\right\|_\infty\le\frac{2\gamma\varepsilon}{1-\gamma}+\sup_{\pi\in\Pi_\varepsilon}\sup_{V\in\VM}\left\|G^{\pi}_r V\right\|_\infty.
	\end{align*}
	In conclusion, taking the supremum of $\pi$ at the left hand side, we have:
	\begin{align*}
		\sup_{\pi\in\Pi, V\in\VM}\left\|G^\pi_r V\right\|_\infty\le\frac{2\gamma\varepsilon}{1-\gamma}+\sup_{\pi\in\Pi_\varepsilon, V\in\VM}\left\|G^{\pi}_r V\right\|_\infty.
	\end{align*}
\end{proof}

\begin{proof}[Proof of Lemma~\ref{lem: num-eps-pi}]
	For $|\VM_\varepsilon|$, we divide $[0, 1/(1-\gamma)]$ by a factor $\varepsilon$ at each dimension, which leads to:
	\begin{align*}
		|\NM(\VM, \|\cdot\|_\infty, \varepsilon)|\le\left(1+\frac{1}{(1-\gamma)\varepsilon}\right)^{|\SM|}.
	\end{align*}
	For $|\Pi_\varepsilon|$, we first consider $|\SM|=1$ and denote sets  $A:=\{p\in\RB_+^{|\AM|}|\sum_{i}p_i=1\}$ and $B:=\{q\in\RB_+^{|\AM|-1}|\sum_{i}q_i\le1\}$. We aim to prove that $|\NM(A,\|\cdot\|_1,\varepsilon)|\le|\NM(B,\|\cdot\|_1,\varepsilon/2)|$. 

	For any $p\in A$, we have $(p_1, \ldots, p_{|\AM|-1})\in B$, there exists $(p_{0,1}, \ldots, p_{0,|\AM|-1})\in\NM(B,\|\cdot\|_1,\varepsilon/2)$ such that:
	\begin{align*}
		\sum_{i=1}^{|\AM|-1}\left|p_i-p_{0,i}\right|\le\frac{\varepsilon}{2}.
	\end{align*}
	And we let $p_{0,|\AM|}=1-\sum_{i=1}^{|\AM|-1}p_{0,i}$, which guarantees that $(p_{0,1}, \ldots, p_{0,|\AM|})\in A$ and:
	\begin{align*}
		\sum_{i=1}^{|\AM|} \left|p_i-p_{0,i}\right|\le\frac{\varepsilon}{2}+\left|p_n-p_{0,|\AM|}\right|
		\le\frac{\varepsilon}{2}+\sum_{i=1}^{|\AM|-1}\left|p_i-p_{0,i}\right|\le\varepsilon.
	\end{align*}
	Thus, we can construct a set $\{p\in\RB_+^{|\AM|}|p_{1:|\AM|-1}\in \NM(B, \|\cdot\|_1,\varepsilon/2), p_{|\AM|}=1-\sum_{i=1}^{|\AM|-1}p_i\}$, which is a $\varepsilon$-net of $A$, from which we obtain the desired result. And also, from Lemma 5.7 of \citet{wainwright2019high}, we obtain that:
	\[
		\left|\NM(B, \|\cdot\|_1,\varepsilon/2)\right|\le\left(1+\frac{4}{\varepsilon}\right)^{|\AM|}.
	\]
	Finally, for $|\SM|>1$, the smallest covering number of $\Pi$ is bounded by:
	\[
		\left|\NM(\Pi,\|\cdot\|_1, \varepsilon)\right|\le\left(1+\frac{4}{\varepsilon}\right)^{|\SM||\AM|}.
	\]
\end{proof}

%% file: tex/apdix/finite_gen.tex
\subsection{Main results with the $(s,a)$-rectangular assumption}
\begin{lem}
	\label{lem: f-eq}
	For any $f$-divergence uncertainty set as Example~\ref{eg: f-set} shows, the convex optimization problem
	\begin{align*}
		\inf_{P}&\sum_{s\in\SM}P(s)V(s), \\
		\text{s.t.}&\hspace{2pt}D_f(P\|P^*)\le\rho, \hspace{4pt}P\in\Delta(\SM),\hspace{4pt}P\ll P^*
	\end{align*}
	can be reformulated as:
	\begin{align*}
		\sup_{\lambda\ge0,\eta\in\RB}-\lambda\sum_{s\in\SM}P^*(s)f^*\left(\frac{\eta-V(s)}{\lambda}\right)-\lambda\rho+\eta,
	\end{align*}
	where $f^*(t)=-\inf_{s\ge0}(f(s)-st)$ is the convex conjugate function \cite{boyd2004convex} of $f$ with restriction to $s\ge0$.
\end{lem}

\begin{proof}
	As we assume $P\ll P^*$, we can set $P^*(s)>0$ for all $s$ without loss of generality. We replace the variable $P$ with $r(s)=P(s)/P^*(s)$, then the original optimization problem can be reformulated as:
	\begin{align*}
		\inf_{r}&\sum_{s\in\SM}r(s)V(s)P^*(s), \\
		\text{s.t.}\hspace{2pt}&\sum_{s\in\SM}f(r(s))P^*(s)\le\rho, \\
		&\sum_{s\in\SM}r(s)P^*(s)=1, \\
		&r(s)\ge0\hspace{2pt}\text{for all $s\in\SM$}.
	\end{align*}
	Thus, we can obtain the Lagrangian function of the problem with domain $r\ge0$, $\lambda\ge0$ and $\eta\in\RB$:
	\begin{align*}
		L(r,\lambda, \eta)=\sum_{s\in\SM}r(s)V(s)P^*(s)+\lambda\left(\sum_{s\in\SM}f(r(s)))P^*(s)-\rho\right)-\eta\left(\sum_{s\in\SM}r(s)P^*(s)-1\right).
	\end{align*}
	Denoting $f^*(t)=-\inf_{s\ge0}(f(s)-st)$, we have:
	\begin{align*}
		\inf_{r\ge0}L(r,\lambda,\eta)=-\lambda\sum_{s\in\SM}P^*(s)f^*\left(\frac{\eta-V(s)}{\lambda}\right)-\lambda\rho+\eta.
	\end{align*}
	By Slater's condition, the primal value equals to the dual value $\sup_{\lambda\ge0,\eta\in\RB} \inf_{r\ge0}L(r,\lambda,\eta)$.
\end{proof}
\subsubsection*{\textbf{Case 1: $L_1$ balls}}
In this case, we set $f(t)=|t-1|$ in Example~\ref{eg: f-set}. The uncertainty set is formulated as: 
\begin{example}[$L_1$ balls]
    For each $(s,a)\in\SM\times\AM$ and by assuming $\rho<2$, the uncertainty sets are defined as:
    \begin{align*}
        &\PM_{s,a}(\rho)=\left\{P(\cdot|s,a)\in\Delta(\SM)\Bigg{|}P(\cdot|s,a)\ll P^*(\cdot|s,a),\hspace{2pt}\sum_{s'\in\SM}\left|P(s'|s,a)-P^*(s'|s,a)\right|\le\rho\right\},\\
        &\widehat{\PM}_{s,a}(\rho)=\left\{P(\cdot|s,a)\in\Delta(\SM)\Bigg{|}P(\cdot|s,a)\ll \widehat{P}(\cdot|s,a),\hspace{2pt}\sum_{s'\in\SM}\left|P(s'|s,a)-\widehat{P}(s'|s,a)\right|\le\rho\right\}.
    \end{align*}
    By the $(s,a)$-rectangular set assumption, we have $\PM=\bigtimes_{(s,a)\in\SM\times\AM}\PM_{s,a}(\rho)$ and $\widehat{\PM}=\bigtimes_{(s,a)\in\SM\times\AM}\widehat{\PM}_{s,a}(\rho)$.
\end{example}

Thus, for any given $V\in\VM$, the explicit forms of $\TM^\pi_r V$ and $\widehat{\TM}^\pi_r V$ are:
\begin{lem}
    \label{lem: l1}
    Under the $(s,a)$-rectangular assumption and $L_1$ balls uncertainty set, for each $s\in\SM$, we have:
    \begin{align*}
        \TM^\pi_r V (s)&=\sum_{a}\pi(a|s)\left(R(s,a) {+} \gamma\sup_{\eta\in\RB}\left( {-} \sum_{s'}P^*(s'|s,a)(\eta {-} V(s'))_+ {-} \frac{(\eta {-} \min_{s'}V(s'))_{+}}{2}\rho {+} \eta \right)\right),\\
        \widehat{\TM}^\pi_r V (s)&=\sum_{a}\pi(a|s)\left(R(s,a) {+} \gamma\sup_{\eta\in\RB}\left( {-} \sum_{s'}\widehat{P}(s'|s,a)(\eta {-} V(s'))_+ {-} \frac{(\eta {-} \min_{s'}V(s'))_{+}}{2}\rho {+} \eta \right)\right).
    \end{align*}
    Moreover, the dual variable $\eta$ can be restricted to interval $[0,\frac{2+\rho}{\rho(1-\gamma)}]$.
\end{lem}
\begin{proof}[Proof of Lemma~\ref{lem: l1}]
	By the $(s,a)$-rectangular set assumption, we have:
	\begin{align*}
		\TM_r^\pi V(s) &= \sum_{a}\pi(a|s)R(s,a)+\gamma\inf_{P\in\PM}\sum_{s'\in\SM}P^\pi(s'|s)V(s')\\
		&= \sum_{a}\pi(a|s)R(s,a)+\gamma\inf_{P\in\PM}\sum_{s'\in\SM, a\in\AM}P(s'|s,a)\pi(a|s)V(s')\\
		&= \sum_{a}\pi(a|s)\left(R(s,a)+\gamma\inf_{P_{s,a}\in\PM_{s,a}(\rho)}\sum_{s'\in\SM}P(s'|s,a)V(s')\right).
	\end{align*}
	Then we solve the following convex optimization problem, where we ignore the dependence on $(s,a)$ in $P(\cdot|s,a)$:
	\begin{align*}
		\inf_{P}&\sum_{s\in\SM}P(s)V(s), \\
		\text{s.t.}\hspace{2pt} &\sum_{s\in\SM}|P(s)-P^*(s)|\le\rho, \\
		&P\in\Delta(\SM),\hspace{4pt}P\ll P^*.
	\end{align*}
	By setting $f(t)=|t-1|$, we have :
	\begin{align*}
		f^*(s)=
		\begin{cases}
			-1,& s\le-1, \\
			s,& s\in[-1,1], \\
			+\infty,& s>1.
		\end{cases}
	\end{align*}
	Thus, by Lemma~\ref{lem: f-eq}, the value of the convex optimization problem is equal to:
	\begin{align*}
		\sup_{\lambda\ge0,\eta\in\RB,\frac{\eta-V(s)}{\lambda}\le1}-\lambda\sum_{s\in\SM}P^*(s)\max\left\{\frac{\eta-V(s)}{\lambda},-1\right\}-\lambda\rho+\eta.
	\end{align*}
	By replacing $\eta$ with $\tilde{\eta}-\lambda$, the problem is equal to:
	\begin{align*}
		\sup_{\lambda\ge0,\tilde{\eta}\in\RB,\frac{\tilde{\eta}-V(s)}{\lambda}\le2}-\sum_{s\in\SM}P^*(s)\max\left\{\tilde{\eta}-\lambda-V(s),-\lambda\right\}-\lambda\rho+\tilde{\eta}-\lambda.
	\end{align*}
	Because $\max\{a,b\}=(a-b)_++b$, the problem turns to 
	\begin{align*}
		\sup_{\lambda\ge0,\tilde{\eta}\in\RB,\frac{\tilde{\eta}-V(s)}{\lambda}\le2}-\sum_{s\in\SM}P^*(s)\left(\tilde{\eta}-V(s)\right)_+ -\lambda\rho+\tilde{\eta}.
	\end{align*}
By optimizing over $\lambda$, the final result is obtained:
	\begin{align*}
		\sup_{\tilde{\eta}\in\RB}-\sum_{s\in\SM}P^*(s)\left(\tilde{\eta}-V(s)\right)_+-\frac{\left(\tilde{\eta}-\min_s V(s)\right)_+}{2}\rho+\tilde{\eta}.
	\end{align*}
	Besides, we denote $g(\eta,P)=\sum_{s\in\SM}P(s)(\eta-V(s))_++\frac{(\eta-\min_s V(s))_+}{2}\rho-\eta$, which is convex in $\eta$. It is easy to observe that $g(\eta,P)=-\eta\ge0$ when $\eta\le0$ and 
    \begin{align*}
        g\left(\frac{2+\rho}{\rho(1-\gamma)},P\right)=-\sum_{s\in\SM}P(s)V(s)+\frac{\frac{2+\rho}{\rho(1-\gamma)}-\min_s V(s)}{2}\rho\ge0
    \end{align*}
    due to $V(s)\in[0,\frac{1}{1-\gamma}]$ for all $s\in\SM$. Thus, we can restrict the dual variable $\eta$ in $[0,\frac{2+\rho}{\rho(1-\gamma)}]$.
\end{proof}

Then, by Lemma~\ref{lem: l1}, we can easily bound the error of $\TM_r^\pi V$ and $\widehat{\TM}_r^\pi V$ with fixed $V\in\VM$ and $\pi\in\Pi$ as Theorem~\ref{thm: l1-fix} states:
\begin{thm}
    \label{thm: l1-fix}
    In the setting of $L_1$ balls, for fixed $V\in\VM$ and $\pi\in\Pi$, the following inequality holds with probability $1-\delta$:
    \begin{align*}
        \left\|\TM_r^\pi V-\widehat{\TM}_r^\pi V\right\|_\infty\le\frac{(2+\rho)\gamma}{\sqrt{2n}\rho(1-\gamma)}\left(1+\sqrt{\log\frac{2|\SM||\AM|(1+(4+\rho)\sqrt{2n})}{\delta}}\right).
    \end{align*}
\end{thm}

\begin{proof}[Proof of Theorem~\ref{thm: l1-fix}]
	To simplify the notation, we also ignore dependence on $(s,a)$. Denote $g(\eta,P)=\sum_{s\in\SM}P(s)(\eta-V(s))_++\frac{(\eta-\min_s V(s))_+}{2}\rho-\eta$, we have:
	\begin{align*}
		g(\eta,\widehat{P})&=\sum_{s\in\SM}\widehat{P}(s)(\eta-V(s))_++\frac{(\eta-\min_s V(s))_+}{2}\rho-\eta\\
		&=\frac{1}{n}\sum_{k=1}^n\sum_{s\in\SM}1(X_k=s)(\eta-V(s))_++\frac{(\eta-\min_s V(s))_+}{2}\rho-\eta\\
		&:=\frac{1}{n}\sum_{k=1}^n Z_k+\frac{(\eta-\min_s V(s))_+}{2}\rho-\eta,
	\end{align*}
	where $\{X_k\}_{k=1}^n$ are i.i.d.\ samples generated by $P(s)$ and we denote $Z_k=\sum_{s\in\SM}1(X_k=s)(\eta-V(s))_+$. Thus, $\EB g(\eta,\widehat{P})=g(\eta,\P)$. With $\eta\in[0,\frac{2+\rho}{\rho(1-\gamma)}]$, we know $Z_k\in[0,\frac{2+\rho}{\rho(1-\gamma)}]$. By Hoeffding's inequality, we have the following inequality:
	\begin{align*}
		\PB\left(\left|g(\eta,\widehat{P})-g(\eta,P)\right|\ge\frac{2+\rho}{\rho(1-\gamma)}\sqrt{\frac{\log\frac{2}{\delta}}{2n}}\right)\le\delta.
	\end{align*}
	Noting that $g(\eta,P)$ is $(2+\frac{\rho}{2})$-Lipschitz w.r.t.\ $\eta$, we can take the $\varepsilon$-net of $[0,\frac{2+\rho}{\rho(1-\gamma)}]$ as $\NM_\varepsilon$ w.r.t.\ metric $|\cdot|$. The size of $\NM_\varepsilon$ is bounded by:
	\begin{align*}
		|\NM_\varepsilon|\le 1+\frac{2+\rho}{\varepsilon\rho(1-\gamma)}.
	\end{align*}
	Thus, we have:
	\begin{align*}
		\sup_{\eta\in[0,\frac{2+\delta}{\delta(1-\gamma)}]}\left|g(\eta,\widehat{P})-g(\eta,P)\right|\le(4+\rho)\varepsilon+\sup_{\eta\in\NM_\varepsilon}\left|g(\eta,\widehat{P})-g(\eta,P)\right|.
	\end{align*}
	Taking $\varepsilon=\frac{2+\rho}{\sqrt{2n}\rho(4+\rho)(1-\gamma)}$, we have the following inequality holds with probability $1-\delta$:
	\begin{align*}
		\sup_{\eta\in[0,\frac{2+\delta}{\delta(1-\gamma)}]}\left|g(\eta,\widehat{P})-g(\eta,P)\right|&\le\frac{2+\rho}{\sqrt{2n}\rho(1-\gamma)}+\frac{2+\rho}{\rho(1-\gamma)}\sqrt{\frac{\log\frac{2(1+(4+\rho)\sqrt{2n})}{\delta}}{2n}}\\
		&=\frac{2+\rho}{\sqrt{2n}\rho(1-\gamma)}\left(1+\sqrt{\log\frac{2(1+(4+\rho)\sqrt{2n})}{\delta}}\right).
	\end{align*}
	Thus, for any fixed $\pi\in\Pi$ and $V\in\VM$, we have:
	\begin{align*}
		\left\|\TM_r^\pi V-\widehat{\TM}_r^\pi V\right\|_\infty\le\gamma\max_{s\in\SM,a\in\AM}\sup_{\eta\in[0,\frac{2+\rho}{\rho(1-\gamma)}]}\left|g(\eta,\widehat{P}_{s,a})-g(\eta,P^*_{s,a})\right|.
	\end{align*}
	Thus, we can take a union bound over $(s,a)\in\SM\times\AM$ and obtain the following inequality with probability $1-\delta$:
	\begin{align*}
		\left\|\TM_r^\pi V-\widehat{\TM}_r^\pi V\right\|_\infty\le\frac{(2+\rho)\gamma}{\sqrt{2n}\rho(1-\gamma)}\left(1+\sqrt{\log\frac{2|\SM||\AM|(1+(4+\rho)\sqrt{2n})}{\delta}}\right).
	\end{align*}
\end{proof}

For the next, we apply a uniform bound over $\VM$ and $\Pi$ to Theorem~\ref{thm: l1-fix}. By the $(s,a)$-rectangular set assumption, the optimal robust policy is deterministic. Thus, we can restrict the policy class $\Pi$ to all the deterministic policies, which is finite with size $|\AM|^{|\SM|}$. Even though $\VM=[0,\frac{1}{1-\gamma}]^{|\SM|}$ is infinite, we can take an $\varepsilon$-net of $\VM$ w.r.t.\ norm $\|\cdot\|_\infty$. Thus, we have the final result as Theorem~\ref{thm: l1-union} states.

\begin{proof}[Proof of Theorem~\ref{thm: l1-union}]
	Let us first consider union bound over $\VM$. We take an $\varepsilon$-net of $\VM$ w.r.t.\ norm $\|\cdot\|_\infty$ and denote it as $\NM_\varepsilon$. For any given $V\in\VM$, there exists $V_\varepsilon\in\NM_\varepsilon$ s.t.\ $\|V-V_\varepsilon\|_\infty\le \varepsilon$. Thus we have:
	\begin{align*}
		\left\|\TM_r^\pi V -\widehat{\TM}_r^\pi V\right\|_\infty&\le\left\|\TM_r^\pi V -\TM_r^\pi V_\varepsilon\right\|_\infty+\left\|\widehat{\TM}_r^\pi V -\widehat{\TM}_r^\pi V_\varepsilon\right\|_\infty+\left\|\TM_r^\pi V_\varepsilon - \widehat{\TM}_r^\pi V_\varepsilon\right\|_\infty\\
		&\le 2\gamma\varepsilon + \sup_{V\in\NM_\varepsilon}\left\|\TM_r^\pi V -\widehat{\TM}_r^\pi V\right\|_\infty.
	\end{align*}
	Thus, we have $\sup_{V\in\VM}\left\|\TM_r^\pi V -\widehat{\TM}_r^\pi V\right\|_\infty\le 2\gamma\varepsilon+\sup_{V\in\NM_\varepsilon}\left\|\TM_r^\pi V -\widehat{\TM}_r^\pi V\right\|_\infty$. Noting that $|\NM_\varepsilon|\le(1+\frac{1}{\varepsilon(1-\gamma)})^{|\SM|}$, we have the following result with probability $1-\delta$:
	\begin{align*}
		&\sup_{V\in\VM}\left\|\TM_r^\pi V -\widehat{\TM}_r^\pi V\right\|_\infty\\
		&\le 2\gamma\varepsilon
		+\frac{(2+\rho)\gamma}{\sqrt{2n}\rho(1-\gamma)}\left(1+\sqrt{\log\frac{2|\SM||\AM|(1+(4+\rho)\sqrt{2n})}{\delta}+|\SM|\log\left(1+\frac{1}{\varepsilon(1-\gamma)}\right)}\right).
	\end{align*}
	Taking $\varepsilon=\frac{2+\rho}{2\sqrt{2n}\rho(1-\gamma)}$, we have the following inequality with probability $1-\delta$:
	\begin{align*}
		&\sup_{V\in\VM}\left\|\TM_r^\pi V -\widehat{\TM}_r^\pi V\right\|_\infty\\
		&\le\frac{(2+\rho)\gamma}{\sqrt{2n}\rho(1-\gamma)}\left(2+\sqrt{\log\frac{2|\SM||\AM|(1+(4+\rho)\sqrt{2n})}{\delta}+|\SM|\log\left(1+\frac{2\sqrt{2n}\rho}{2+\rho}\right)}\right).
	\end{align*}
	Next, we consider union bound over $\Pi$. With the $(s,a)$-rectangular assumption, we can restrict the policy class $\Pi$ to deterministic class with finite size $|\AM|^{|\SM|}$. Thus, with combining  Lemma~\ref{lem: uni-dev-pi} and Lemma~\ref{lem: uni-dev-v}, the following inequality holds with probability $1-\delta$:
	\begin{align*}
		&\max_{\pi}V_r^\pi(\mu)-V_r^{\widehat{\pi}}(\mu)\\
		&\le\frac{2(2+\rho)\gamma}{\sqrt{2n}\rho(1-\gamma)^2}\left(2+\sqrt{\log\frac{2|\SM||\AM|(1+(4+\rho)\sqrt{2n})}{\delta}+|\SM|\log|\AM|\left(1+\frac{2\sqrt{2n}\rho}{2+\rho}\right)}\right).
	\end{align*}
	Noting that $|\SM|\ge1$, we can simplify the above inequality:
	\begin{align*}
		&\max_{\pi}V_r^\pi(\mu)-V_r^{\widehat{\pi}}(\mu)\\
		&\le\frac{2(2+\rho)\gamma\sqrt{|\SM|}}{\sqrt{2n}\rho(1-\gamma)^2}\left(2+\sqrt{\log\frac{2|\SM||\AM|^2(1+(4+\rho)\sqrt{2n})}{\delta}+\log\left(1+\frac{2\sqrt{2n}\rho}{2+\rho}\right)}\right)\\
		&\le\frac{2(2+\rho)\gamma\sqrt{|\SM|}}{\sqrt{2n}\rho(1-\gamma)^2}\left(2+\sqrt{\log\frac{4|\SM||\AM|^2(1+2(2+\rho)\sqrt{2n})^2}{\delta(2+\rho)}}\right).
	\end{align*}
	The final inequality holds by the following observation:
	\begin{align*}
		2(1+(4+\rho)\sqrt{2n})(2+\rho+2\sqrt{2n}\rho)\le4(1+2(2+\rho)\sqrt{2n})^2.
	\end{align*}
\end{proof}

\subsubsection*{\textbf{Case 2: $\chi^2$ balls}}
In this case, we set $f(t)=(t-1)^2$ in Example~\ref{eg: f-set}. Thus, the uncertainty set is formulated as: 
\begin{example}[$\chi^2$ balls]
    For each $(s,a)\in\SM\times\AM$, the uncertainty sets are defined as:
    \begin{align*}
        &\PM_{s,a}(\rho)=\left\{P(\cdot|s,a)\in\Delta(\SM)\Bigg{|}P(\cdot|s,a)\ll P^*(\cdot|s,a),\hspace{2pt}\sum_{s'\in\SM}\frac{\left(P(s'|s,a) {-} P^*(s'|s,a)\right)^2}{P^*(s'|s,a)}\le\rho\right\},\\
        &\widehat{\PM}_{s,a}(\rho)=\left\{P(\cdot|s,a)\in\Delta(\SM)\Bigg{|}P(\cdot|s,a)\ll \widehat{P}(\cdot|s,a),\hspace{2pt}\sum_{s'\in\SM}\frac{\left(P(s'|s,a)-\widehat{P}(s'|s,a)\right)^2}{\widehat{P}(s'|s,a)}\le\rho\right\}.
    \end{align*}
    With the $(s,a)$-rectangular set assumption, we define $\PM=\bigtimes_{(s,a)\in\SM\times\AM}\PM_{s,a}(\rho)$ and $\widehat{\PM}=\bigtimes_{(s,a)\in\SM\times\AM}\widehat{\PM}_{s,a}(\rho)$.
\end{example}

By \cite{duchi2018learning}, the results of $f(t)=(t-1)^2$ can be generalized to $f(t)\propto t^k$ for $k>1$. However, the excess risk in DRO is controlled by $\widetilde{O}_p(\frac{1}{n^{(1-1/k)}}\vee\frac{1}{\sqrt{n}})$, where the fastest rate is obtained when $k=2$. Thus we consider the special case that $k=2$. Similar with the $L_1$ ball case, for any given $V\in\VM$ and $\pi\in\Pi$, the explicit forms of $\TM_r^\pi V$ and $\widehat{\TM}_r^\pi V$ are:
\begin{lem}
    \label{lem: chi2}
    Under the $(s,a)$-rectangular assumption and $\chi^2$ balls uncertainty set, for each $s\in\SM$, we have:
    \begin{align*}
        \TM^\pi_r V (s)&=\sum_{a}\pi(a|s)\left(R(s,a) + \gamma\sup_{\eta\in\RB}\left( -C(\rho)\sqrt{\sum_{s'}P^*(s'|s,a)(\eta-V(s'))_+^2}+\eta\right)\right),\\
        \widehat{\TM}^\pi_r V (s)&=\sum_{a}\pi(a|s)\left(R(s,a) + \gamma\sup_{\eta\in\RB}\left(-C(\rho)\sqrt{\sum_{s'}\widehat{P}(s'|s,a)(\eta-V(s'))_+^2}+\eta\right)\right),
    \end{align*}
    where $C(\rho)=\sqrt{1+\rho}$. Moreover, the dual variable $\eta$ can be restricted to interval $[0,\frac{C(\rho)}{(C(\rho)-1)(1-\gamma)}]$.
\end{lem}

\begin{proof}[Proof of Lemma~\ref{lem: chi2}]
	As $f(t)=(t-1)^2$, we have:
	\begin{align*}
		f^*(s)=
		\begin{cases}
			\frac{t^2}{4}+t, & s\ge-2, \\
			-1, & s<-2.
		\end{cases}
	\end{align*}
	By Lemma~\ref{lem: f-eq}, the value of the convex optimization problem is equal to:
	\begin{align*}
		\sup_{\lambda\ge0,\eta\in\RB}-\lambda\sum_{s\in\SM}P^*(s)\left(\frac{(\eta-V(s)+2\lambda)^2_+}{4\lambda^2}-1\right)-\lambda\rho +\eta,
	\end{align*}
	which is equivalent to:
	\begin{align*}
		\sup_{\lambda\ge0,\eta\in\RB}-\sum_{s\in\SM}P^*(s)\left(\frac{(\eta-V(s)+2\lambda)^2_+}{4\lambda}\right)-\lambda(\rho-1) +\eta.
	\end{align*}
	Replacing $\eta$ with $\tilde{\eta}=\eta+2\lambda$, the problem turns into:
	\begin{align*}
		\sup_{\lambda\ge0,\tilde{\eta}\in\RB}-\frac{1}{4\lambda}\sum_{s\in\SM}P^*(s)\left(\tilde{\eta}-V(s))^2_+\right)-\lambda(\rho+1) +\tilde{\eta}.
	\end{align*}
	Optimizing over $\lambda\ge0$, the problem turns into:
	\begin{align*}
		\sup_{\tilde{\eta}\in\RB}-\sqrt{\rho+1}\sqrt{\sum_{s\in\SM}P^*(s)\left(\tilde{\eta}-V(s)\right)_+^2}+\tilde{\eta}.
	\end{align*}
	Besides, we denote $g(\eta,P)=\sqrt{\rho+1}\sqrt{\sum_{s\in\SM}P(s)(\eta-V(s))_+^2}-\eta$, which is convex in $\eta$. We note that $g(\eta,P)=-\eta\ge0$ when $\eta\le0$ and $g\left(\frac{C(\rho)}{(C(\rho)-1)(1-\gamma)}, P\right)\ge0$. Thus, we can restrict the dual variable $\eta$ to interval $[0,\frac{C(\rho)}{(C(\rho)-1)(1-\gamma)}]$.
\end{proof}

By Lemma~\ref{lem: chi2}, the error bound between $\TM_r^\pi$ and $\widehat{\TM}_r^\pi$ with fixed $V\in\VM$ and $\pi\in\Pi$ is:
\begin{thm}
    \label{thm: chi2-fix}
    In the setting of $\chi^2$ balls, for fixed $V\in\VM$ and $\pi\in\Pi$, the following inequality holds with probability $1-\delta$:
    \begin{align*}
        \left\|\TM_r^\pi V-\widehat{\TM}_r^\pi V\right\|_\infty\le\frac{C^2(\rho)\gamma}{(C(\rho)-1)(1-\gamma)\sqrt{n}}\left(2+\sqrt{2\log\frac{2|\SM||\AM|(1+4\sqrt{n})}{\delta}}\right).
    \end{align*}
\end{thm}

\begin{proof}[Proof of Theorem~\ref{thm: chi2-fix}]
    Firstly, we ignore dependence on $(s,a)$ and denote $g(\eta,P)=-\eta+C(\rho)\sqrt{\sum_{s\in\SM}P(s)(\eta-V(s))_+^2}$. Besides, we also denote $Y_k=\sum_{s\in\SM}1(X_k=s)(\eta-V(s))_+$, which leads to $Y_k^2=\sum_{s\in\SM}1(X_k=s)(\eta-V(s))_+^2$, and $Y:=(Y_1, \ldots, Y_n)^\top$. Thus, we can re-write $g(\eta,\widehat{P})=\frac{C(\rho)}{\sqrt{n}}\|Y\|_2-\eta$, which is $\frac{C(\rho)}{\sqrt{n}}$-Lipschitz w.r.t.\ $Y$ and norm $\|\cdot\|_2$. By Lemma 6 of \cite{duchi2018learning}, we have the following inequality holding with probability $1-\delta$:
	\begin{align*}
		\left|g(\eta,\widehat{P})-\EB g(\eta,\widehat{P})\right|\le\frac{\sqrt{2} C^2(\rho)}{(C(\rho)-1)(1-\gamma)\sqrt{n}}\sqrt{\log\frac{2}{\delta}}.
	\end{align*}
	To bound the difference $|\EB g(\eta,\widehat{P})-g(\eta,P)|$, we apply Lemma 8 of \cite{duchi2018learning} and first obtain:
	\begin{align*}
		\frac{1}{\sqrt{n}}\EB\|Y\|_2\ge \sqrt{\sum_{k=1}^n \EB Y_k^2}-\sqrt{\frac{C(\rho)}{(C(\rho)-1)(1-\gamma)}}\frac{1}{\sqrt{n}}.
	\end{align*}
	By $\EB\|Y\|_2\le\sqrt{n\sum_{k=1}^n\EB Y_k^2}$, we then have:
	\begin{align*}
		\left|\EB g(\eta,\widehat{P})- g(\eta,P)\right|\le C(\rho)\sqrt{\frac{C(\rho)}{(C(\rho)-1)(1-\gamma)}}\frac{1}{\sqrt{n}}.
	\end{align*}
	Putting all these together, we have the following inequality with probability $1-\delta$:
	\begin{align*}
		\left|g(\eta, \widehat{P})- g(\eta, P)\right|\le\frac{C^2(\rho)}{(C(\rho)-1)(1-\gamma)\sqrt{n}}\left(1+\sqrt{2\log\frac{2}{\delta}}\right).
	\end{align*}
	Noting that $g(\eta,P)$ is $1+C(\rho)$-Lipschitz w.r.t.\ $\eta$, we take the $\varepsilon$-net of $[0,\frac{C(\rho)}{(C(\rho)-1)(1-\gamma)}]$ as $\NM_\varepsilon$ w.r.t.\ metric $|\cdot|$. The size of $\NM_\varepsilon$ is bounded by:
	\begin{align*}
		|\NM_\varepsilon|\le1+\frac{C_{2}(\rho)}{(C_{2}(\rho)-1)(1-\gamma)\varepsilon}.
	\end{align*}
	Thus, we have:
	\begin{align*}
		\sup_{\eta\in[0,\frac{C(\rho)}{(C(\rho)-1)(1-\gamma)}]}\left|g(\eta,\widehat{P})-g(\eta,P)\right|\le2(1+C(\rho))\varepsilon+\sup_{\eta\in\NM_\varepsilon}\left|g(\eta,\widehat{P})-g(\eta,P)\right|.
	\end{align*}
	By taking $\varepsilon=\frac{C^2(\rho)}{2\sqrt{n}(1-\gamma)(C^2(\rho)-1)}$, the following inequality holds with probability $1-\delta$:
	\begin{align*}
		&\sup_{\eta\in[0,\frac{C(\rho)}{(C(\rho)-1)(1-\gamma)}]}\left|g(\eta,\widehat{P})-g(\eta,P)\right|\\
		&\le\frac{C^2(\rho)}{(C(\rho)-1)(1-\gamma)\sqrt{n}}\left(2+\sqrt{2\log\frac{2|\NM_\varepsilon|}{\delta}}\right)\\
		&\le\frac{C^2(\rho)}{(C(\rho)-1)(1-\gamma)\sqrt{n}}\left(2+\sqrt{2\log\frac{2(1+2(1+1/C(\rho))\sqrt{n})}{\delta}}\right).
	\end{align*}
	As $C(\rho)=\sqrt{1+\rho}\ge1$, we then have
	\begin{align*}
		\sup_{\eta\in[0,\frac{C(\rho)}{(C(\rho)-1)(1-\gamma)}]}\left|g(\eta,\widehat{P})-g(\eta,P)\right|\le\frac{C^2(\rho)}{(C(\rho) {-} 1)(1{-} \gamma)\sqrt{n}}\left(2 {+} \sqrt{2\log\frac{2(1 {+} 4\sqrt{n})}{\delta}}\right)
	\end{align*}
	holding with probability $1-\delta$. Thus, the final result is obtained by union bound over $(s,a)\in\SM\times\AM$.
\end{proof}

Similar with the case of $L_1$ balls, we can extend Theorem~\ref{thm: chi2-fix} to a uniform bound over $\VM$ and deterministic policy class $\Pi$ as Theorem~\ref{thm: chi2-uni} states.

\begin{proof}[Proof of Theorem~\ref{thm: chi2-uni}]
	Similar with Theorem~\ref{thm: l1-union}, we take an $\varepsilon$-net of $\VM$ w.r.t.\ norm $\|\cdot\|_\infty$ and denote it as $\VM_\varepsilon$. By Theorem~\ref{thm: chi2-fix}, we have the following inequality with probability $1-\delta$:
	\begin{align*}
		&\sup_{V\in\VM}\left\|\TM_r^\pi V -\widehat{\TM}_r^\pi V\right\|_\infty\\&
		\le 2\gamma\varepsilon+\frac{C^2(\rho)\gamma}{(C(\rho) {-} 1)(1 {-} \gamma)\sqrt{n}}\left(2+\sqrt{2\log\frac{2|\SM||\AM|(1 {+} 4\sqrt{n})}{\delta}+|\SM|\log\left(1 {+} \frac{1}{(1 {-} \gamma)\varepsilon}\right)}\right).
	\end{align*}
	Taking $\varepsilon=\frac{C^2(\rho)}{(C(\rho)-1)(1-\gamma)\sqrt{n}}$, we have the following inequality with probability $1-\delta$:
	\begin{align*}
		&\sup_{V\in\VM}\left\|\TM_r^\pi V -\widehat{\TM}_r^\pi V\right\|_\infty\\
		&\le\frac{C^2(\rho)\gamma}{(C(\rho) {-} 1)(1 {-} \gamma)\sqrt{n}}\left(4+\sqrt{2\log\frac{2|\SM||\AM|(1 {+} 4\sqrt{n})}{\delta}+|\SM|\log\left(1 {+} \frac{(C(\rho) {-} 1)\sqrt{n}}{C^2(\rho)}\right)}\right).
	\end{align*}
	By the $(s,a)$-rectangular assumption, the policy class $\Pi$ being deterministic is enough. Thus we have the final result with probability $1-\delta$:
	\begin{align*}
		&\max_\pi V_r^\pi(\mu)-V_r^{\widehat{\pi}}(\mu)\\
		&\le\frac{2C^2(\rho)\gamma}{(C(\rho) {-} 1)(1 {-} \gamma)^2\sqrt{n}}\left(4+\sqrt{2|\SM|\log\frac{2|\SM||\AM|^2(1+4\sqrt{n})(1+(C(\rho)-1)\sqrt{n})}{\delta C^2(\rho)}}\right)\\
		&\le\frac{2C^2(\rho)\gamma}{(C(\rho) {-} 1)(1 {-} \gamma)^2\sqrt{n}}\left(4+\sqrt{2|\SM|\log\frac{2|\SM||\AM|^2(1+(C(\rho)+3)\sqrt{n})^2}{\delta C^2(\rho)}}\right),
	\end{align*}
	where the final inequality holds because
	\begin{align*}
		(1+4\sqrt{n}) [1+(C(\rho)-1)\sqrt{n} ] \le [1+(C(\rho)+3)\sqrt{n}]^2.
	\end{align*}
\end{proof}

\subsubsection*{\textbf{Case 3: KL balls}}
In this case, we set $f(t)=t\log t$ in Example~\ref{eg: f-set}. The uncertainty set is formulated as:
\begin{example}[KL balls]
    For each $(s,a)\in\SM\times\AM$, the uncertainty sets are defined as:
    \begin{align*}
        &\PM_{s,a}(\rho)=\left\{P(\cdot|s,a)\in\Delta(\SM)\Bigg{|}P(\cdot|s,a)\ll P^*(\cdot|s,a),\hspace{2pt}\sum_{s'\in\SM}P(s'|s,a)\log\frac{P(s'|s,a)}{P^*(s'|s,a)}\le\rho\right\},\\
        &\widehat{\PM}_{s,a}(\rho)=\left\{P(\cdot|s,a)\in\Delta(\SM)\Bigg{|}P(\cdot|s,a)\ll \widehat{P}(\cdot|s,a),\hspace{2pt}\sum_{s'\in\SM}P(s'|s,a)\log\frac{P(s'|s,a)}{\widehat{P}(s'|s,a)}\le\rho\right\}.
    \end{align*}
    By the $(s,a)$-rectangular set assumption, we define $\PM=\bigtimes_{(s,a)\in\SM\times\AM}\PM_{s,a}(\rho)$ and $\widehat{\PM}=\bigtimes_{(s,a)\in\SM\times\AM}\widehat{\PM}_{s,a}(\rho)$.
\end{example}
Similar with the $L_1$ ball case, for any given $V\in\VM$ and $\pi\in\Pi$, the explicit forms of $\TM_r^\pi V$ and $\widehat{\TM}_r^\pi V$ are:
\begin{lem}
    \label{lem: kl}
    Under the $(s,a)$-rectangular assumption and KL balls uncertainty set, for each $s\in\SM$, we have:
    \begin{align*}
        \TM^\pi_r V (s)&=\sum_{a}\pi(a|s)\left(R(s,a) + \gamma\sup_{\lambda\ge0}\left(-\lambda\rho-\lambda\log\sum_{s'}P^*(s'|s,a)\exp(-\frac{V(s')}{\lambda})\right)\right),\\
        \widehat{\TM}^\pi_r V (s)&=\sum_{a}\pi(a|s)\left(R(s,a) + \gamma\sup_{\lambda\ge0}\left(-\lambda\rho-\lambda\log\sum_{s'}\widehat{P}(s'|s,a)\exp(-\frac{V(s')}{\lambda})\right)\right).
    \end{align*}
    Moreover, the dual variable $\lambda$ can be restricted to interval $[0, \frac{1}{\rho(1-\gamma)}]$.
\end{lem}
\begin{proof}[Proof of Lemma~\ref{lem: kl}]
	As $f(t)=t\log t$, we have $f^*(s)=\exp(s-1)$. By Lemma~\ref{lem: f-eq}, the value of the convex optimization problem is equal to:
	\begin{align*}
		\sup_{\lambda\ge0, \eta\in\RB}-\lambda\sum_{s\in\SM}P^*(s)\exp(\frac{\eta-V(s)}{\lambda}-1)-\lambda\rho+\eta.
	\end{align*}
	Optimizing over $\eta$, we obtain the equivalent form:
	\begin{align*}
		\sup_{\lambda\ge0}-\lambda\log\sum_{s\in\SM}P^*(s)\exp(-\frac{V(s)}{\lambda})-\lambda\rho.
	\end{align*}
	Besides, we denote $g(\lambda,P)=\lambda\rho+\lambda\log\sum_{s}P(s)\exp(-V(s)/\lambda)$, which is convex in $\eta$. Even though the domain of $g(\lambda,P)$ does not conclude $\lambda=0$, we observe $g(\lambda,P)\ge0$ for all $\lambda\ge\frac{1}{\rho(1-\gamma)}$ and $g(\lambda,P)$ is monotonically increasing in $\lambda$ when $\lambda\ge\frac{1}{\rho(1-\gamma)}$.  Thus, the optimal dual variable $\lambda^*$ takes value in interval $[0, \frac{1}{\rho(1-\gamma)}]$.
\end{proof}

Thus, by Lemma~\ref{lem: kl}, for fixed $V\in\VM$ and $\pi\in\Pi$, we have:
\begin{thm}
    \label{thm: kl-fix}
    In the setting of KL balls, for fixed $V\in\VM$ and $\pi\in\Pi$, the following inequality holds with probability $1-\delta$:
    \begin{align*}
        \left\|\TM_r^\pi V-\widehat{\TM}_r^\pi V\right\|_\infty\le\frac{2\gamma}{\rho(1-\gamma)\underline{p}\sqrt{n}}\sqrt{\log\frac{2|\SM|^2|\AM|}{\delta}},
    \end{align*}
    where $\underline{p}=\min_{P^*(s'|s,a)>0}P^*(s'|s,a)$.
\end{thm}

\begin{proof}[Proof of Theorem~\ref{thm: kl-fix}]
	We denote $g(\lambda, P)=\inf_{\lambda\in[0,1/\rho(1-\gamma)]}\lambda\rho+\lambda\log\sum_{s}P(s)\exp(-V(s)/\lambda)$. Thus we have:
	\begin{align*}
		\left|g(\lambda,\widehat{P})-g(\lambda,P)\right|&\le\sup_{\lambda\in[0,\frac{1}{\rho(1-\gamma)}]}\left|\lambda\log\sum_{s}\widehat{P}(s)\exp(-\frac{V(s)}{\lambda})-\lambda\log\sum_{s}P(s)\exp(-\frac{V(s)}{\lambda})\right|\\
		&\le\frac{1}{\rho(1-\gamma)}\sup_{\lambda\in[0,\frac{1}{\rho(1-\gamma)}]}\left|\log\left(1+\frac{\sum_{s}(\widehat{P}(s)-P(s))\exp(-\frac{V(s)}{\lambda})}{\sum_{s}P(s)\exp(-\frac{V(s)}{\lambda})}\right)\right|\\
		&\le \frac{2}{\rho(1-\gamma)}\frac{\sum_{s}\left|\widehat{P}(s)-P(s)\right|\exp(-\frac{V(s)}{\lambda})}{\sum_{s}P(s)\exp(-\frac{V(s)}{\lambda})},
	\end{align*}
	where the last inequality holds by $|\log(1+x)|\le 2|x|$ for $|x|\le1/2$. Noting that $\widehat{P}\ll P$ by generative model assumption, we then have:
	\begin{align*}
		\left|g(\lambda,\widehat{P})-g(\lambda,P)\right|\le\frac{2}{\rho(1-\gamma)}\max_{s}\left|\frac{\widehat{P}(s)}{P(s)}-1\right|.
	\end{align*}
	Denote $\underline{p}=\min_{P(s'|s,a)>0}P(s'|s,a)$ and we know $\widehat{P}(s)=\frac{1}{n}\sum_{k}1(X_k=s)$.  Hoeffding's inequality tells us:
	\begin{align*}
		\PB\left(\max_s \left|\frac{\widehat{P}(s)}{P(s)}-1\right|\ge\sqrt{\frac{1}{n\underline{p}^2}\log\frac{2|\SM|}{\delta}}\right)\le\delta.
	\end{align*}
	Thus, with probability $1-\delta$, we have:
	\begin{align*}
		\left|g(\lambda,\widehat{P})-g(\lambda,P)\right|\le\frac{2}{\rho\underline{p}(1-\gamma)\sqrt{n}}\sqrt{\log\frac{2|\SM|}{\delta}}.
	\end{align*}
	The final result is obtained by union bound over $(s,a)\in\SM\times\AM$.
\end{proof}

Next, we can extend Theorem~\ref{thm: kl-fix} to a uniform bound over $\VM$ and deterministic policy class $\Pi$ as Theorem~\ref{thm: kl-uni} states.

\begin{proof}[Proof of Theorem~\ref{thm: kl-uni}]
	Similar with Theorem~\ref{thm: l1-union}, we take an $\varepsilon$-net of $\VM$ w.r.t.\ norm $\|\cdot\|_\infty$ and denote it as $\VM_\varepsilon$. By Theorem~\ref{thm: kl-fix}, we have the following inequality with probability $1-\delta$:
	\begin{align*}
		\sup_{V\in\VM}\left\|\TM_r^\pi V -\widehat{\TM}_r^\pi V\right\|_\infty\le 2\gamma\varepsilon+\frac{2\gamma}{\rho(1-\gamma)\underline{p}\sqrt{n}}\sqrt{\log\frac{2|\SM|^2|\AM|}{\delta}+|\SM|\log\left(1+\frac{1}{(1-\gamma)\varepsilon}\right)}.
	\end{align*}
	Taking $\varepsilon=\frac{1}{\rho(1-\gamma)\underline{p}\sqrt{n}}$, with probability $1-\delta$, we have:
	\begin{align*}
		\sup_{V\in\VM}\left\|\TM_r^\pi V -\widehat{\TM}_r^\pi V\right\|_\infty\le\frac{2\gamma}{\rho(1-\gamma)\underline{p}\sqrt{n}}\left(1+\sqrt{\log\frac{2|\SM|^2|\AM|}{\delta}+|\SM|\log\left(1+\rho\underline{p}\sqrt{n}\right)}\right).
	\end{align*}
	Taking a union bound over deterministic policy class $\Pi$, we have:
	\begin{align*}
		\max_\pi V_r^\pi(\mu)-V_r^{\widehat{\pi}}(\mu)&\le\frac{2\gamma}{\rho(1-\gamma)\underline{p}\sqrt{n}}\left(1+\sqrt{\log\frac{2|\SM|^2|\AM|}{\delta}+|\SM|\log|\AM|\left(1+\rho\underline{p}\sqrt{n}\right)}\right)\\
		&=\frac{2\gamma}{\rho(1-\gamma)\underline{p}\sqrt{n}}\left(1+\sqrt{|\SM|\log\frac{2|\SM|^2|\AM|^2(1+\rho\underline{p}\sqrt{n})}{\delta}}\right).
	\end{align*}
\end{proof}

\subsection{Main results with the $s$-rectangular assumption}

\begin{lem}
	\label{lem: f-eq-s}
	For any $f$-divergence uncertainty set as Example~\ref{eg: f-set-s} states, the convex optimization problem
	\begin{align*}
		\inf_{P}&\sum_{s\in\SM, a\in\AM}P_a(s)\pi(a)V(s), \\
		\text{s.t.}&\hspace{2pt}\sum_{a\in\AM}D_f(P_a\|P_a^*)\le|\AM|\rho, \hspace{4pt}P_a\in\Delta(\SM),\hspace{4pt}P_a\ll P_a^* \hspace{4pt}\text{for all $a\in\AM$}
	\end{align*}
	can be reformulated as:
	\begin{align*}
		\sup_{\lambda\ge0,\eta\in\RB^{|\AM|}}-\lambda\sum_{s\in\SM, a\in\AM}P_a^*(s)f^*\left(\frac{\eta_a-\pi(a)V(s)}{\lambda}\right)-\lambda|\AM|\rho+\sum_{a\in\AM}\eta_a,
	\end{align*}
	where $f^*(t)=-\inf_{s\ge0}(f(s)-st)$.
\end{lem}
\begin{proof}
	Similar with the proof of Lemma~\ref{lem: f-eq},  we first replace the variable $P_a$ with $r_a(s)=P_a(s)/P^*_a(s)$. The original optimization problem can be reformulated as:
	\begin{align*}
	    \inf_{r}&\sum_{s\in\SM,a\in\AM}r_a(s)P^*_a(s)V(s),\\
	    \text{s.t.}&\hspace{2pt}\sum_{s\in\SM, a\in\AM}f(r_a(s))P^*_a(s)\le|\AM|\rho,\\
	    &\hspace{2pt}\sum_{s\in\SM}r_a(s)P_a^*(s)=1,\hspace{2pt}\text{for all $a\in\AM$,}\\
	    &\hspace{2pt}r_a(s)\ge0,\hspace{2pt}\text{for all $s\in\SM$, $a\in\AM$}.
	\end{align*}
	Then we obtain the Lagrangian function of the problem with  $r\ge0$, $\lambda\ge0$ and $\eta\in\RB^{|\AM|}$:
	\begin{align*}
	    L(r,\lambda,\eta)=&\sum_{s\in\SM, a\in\AM}r_a(s)P_a^*(s)V(s)+\lambda\left(\sum_{s\in\SM, a\in\AM}f(r_a(s))P^*_a(s)-|\AM|\rho\right)\\
	    &-\sum_{a\in\AM}\eta_{a}\left(\sum_{s\in\SM}r_a(s)P_a^*(s)-1\right).
	\end{align*}
	Denoting $f^*(t)=-\inf_{s\ge0}(f(s)-st)$, we have the dual objective:
	\begin{align*}
	    \inf_{r\ge0}L(r,\lambda,\eta)=-\lambda\sum_{s\in\SM, a\in\AM}f^*\left(\frac{\eta_a-V(s)}{\lambda}\right)-\lambda|\AM|\rho+\sum_{a\in\AM}\eta_a.
	\end{align*}
	By Slater's condition, the primal value equals to the dual value $\sup_{\lambda\ge0,\eta\in\RB^{|\AM|}}\inf_{\r\ge0}L(r,\lambda,\eta)$.
\end{proof}

\subsubsection*{\textbf{Case 1: $L_1$ balls}}
In this case, we set $f(t)=|t-1|$ in Example~\ref{eg: f-set-s}. The uncertainty set is formulated as:
\begin{example}[$L_1$ balls]
    For each $s\in\SM$, the uncertainty sets are defined as:
    \begin{align*}
        &\PM_{s}(\rho)=\left\{P(\cdot|s,a)\in\Delta(\SM)\Bigg{|}P(\cdot|s,a)\ll P^*(\cdot|s,a),\hspace{2pt}\sum_{s'\in\SM, a\in\AM}\left|P(s'|s,a) {-} P^*(s'|s,a)\right|\le|\AM|\rho\right\},\\
        &\widehat{\PM}_{s}(\rho)=\left\{P(\cdot|s,a)\in\Delta(\SM)\Bigg{|}P(\cdot|s,a)\ll \widehat{P}(\cdot|s,a),\hspace{2pt}\sum_{s'\in\SM, a\in\AM}\left|P(s'|s,a) {-} \widehat{P}(s'|s,a)\right|\le|\AM|\rho\right\}.
    \end{align*}
    By the $s$-rectangular set assumption, we define $\PM=\bigtimes_{s\in\SM}\PM_s(\rho)$ and $\widehat{\PM}=\bigtimes_{s\in\SM} \widehat{\PM}_s(\rho)$.
\end{example}
For any given $V\in\VM$ and $\pi\in\Pi$, the explicit forms of $\TM_r^\pi V$ and $\widehat{\TM}_r^\pi V$ are:
\begin{lem}
    \label{lem: l1-s}
    Under the s-rectangular assumption and $L_1$ balls uncertainty set, for each $s\in\SM$, we have:
    \begin{align*}
        \TM_r^\pi V(s)&=R^\pi(s)+\gamma\sup_{\eta\in\RB^{|\AM|}}\left({-} \sum_{s',a}P^*(s'|s,a)\left(\eta_a {-} \pi(a|s)V(s')\right)_+ {-} |\AM|h(\eta,\pi(\cdot|s), V)\rho {+} \sum_{a}\eta_a\right),\\
        \widehat{\TM}_r^\pi V(s)&=R^\pi(s)+\gamma\sup_{\eta\in\RB^{|\AM|}}\left(-\sum_{s',a}\widehat{P}(s'|s,a)\left(\eta_a {-} \pi(a|s)V(s')\right)_+ {-} |\AM|h(\eta,\pi(\cdot|s), V)\rho {+} \sum_{a}\eta_a\right),
    \end{align*}
    where $R^{\pi}(s)=\sum_a \pi(a|s)R(s,a)$ and $h(\eta,\pi, V)=(\max_{a,s'}\frac{\eta_a-\pi(a)V(s')}{2})_+$.
\end{lem}

\begin{proof}[Proof of Lemma~\ref{lem: l1-s}]
	By definition of the s-rectangular set assumption, for any given $V\in\VM$ and $\pi\in\Pi$, we have:
	\begin{align*}
		\TM_r^\pi V(s)=\sum_{a}\pi(a|s)R(s,a)+\gamma\inf_{P_s\in\PM_s(\rho)}\sum_{s'\in\SM, a\in\AM}P(s'|s,a)\pi(a|s)V(s').
	\end{align*}
	Then we solve the following convex optimization problem, where we ignore the dependence on $s$ of $P(\cdot|s,a)$:
	\begin{align*}
		\inf_{P}&\sum_{s\in\SM}P_a(s)\pi(a)V(s), \\
		\text{s.t.}\hspace{2pt} &\sum_{s\in\SM, a\in\AM}|P_a(s)-P_a^*(s)|\le|\AM|\rho, \\
		&P_a\in\Delta(\SM),\hspace{4pt}P_a\ll P_a^*\hspace{4pt} \text{for all $a\in\AM$}.
	\end{align*}
	Taking $f=|t-1|$, we have:
	\begin{align*}
		f^*(s)=
		\begin{cases}
			-1,& s\le-1, \\
			s,& s\in[-1,1], \\
			+\infty, & s>1.
		\end{cases}
	\end{align*}
	Thus, by Lemma~\ref{lem: f-eq-s}, we turn the optimization problem into:
	\begin{align*}
		\sup_{\lambda\ge0,\eta\in\RB^{|\AM|}, \frac{\eta_a-\pi(a)V(s)}{\lambda}\le1}-\lambda\sum_{s\in\SM, a\in\AM}P^*_a(s)\max\left\{\frac{\eta_a-\pi(a)V(s)}{\lambda}, -1\right\}-\lambda|\AM|\rho+\sum_{a\in\AM}\eta_a.
	\end{align*}
	By similar proof of Lemma~\ref{lem: l1}, the optimization problem can be formulated as:
	\begin{align*}
		\sup_{\eta\in\RB^{|\AM|}}-\sum_{s\in\SM,a\in\AM}P^*_a(s)\left(\eta_a-\pi(a)V(s)\right)_+-\left(\max_{a,s}\frac{\eta_a-\pi(a)V(s)}{2}\right)_+|\AM|\rho+\sum_{a\in\AM}\eta_a.
	\end{align*}
\end{proof}

Different from the case when the $(s,a)$-rectangular assumption holds, the explicit form of $\TM_r^\pi V$ in Lemma~\ref{lem: l1-s} is determined by a $|\AM|$-dimensional vector $\eta$. However, we can still find a way to deal with this situation. 

\begin{lem}
	\label{lem: l1-value}
	For fixed $\pi$ and $V$, we denote $g(\eta,P)=\sum_{s,a}P_a(s)(\eta_a-\pi(a)V(s))_++(\max_{a,s}\frac{\eta_a-\pi(a)V(s)}{2})_+|\AM|\rho-\sum_{a}\eta_a$. The infimum $\eta$ of $g(\eta,P)$ locates in set:
	\begin{align*}
		I_{L_1}=\left\{\eta\in\RB^{|\AM|}\Bigg{|}\eta_a\ge0, \sum_{a}\eta_a\le\frac{2+\rho}{\rho(1-\gamma)}\right\}.
	\end{align*}
\end{lem}

\begin{proof}[Proof of Lemma~\ref{lem: l1-value}]
	If there exists $\widehat{a}\in\AM$ such that $\eta_{\widehat{a}}\le0$, we have:
	\begin{align*}
		g(\eta,P)=-\eta_{\widehat{a}}+\sum_{s,a\not=\widehat{a}}P_a(s)(\eta_a-\pi(a)V(s))_++\left(\max_{s,a\not=\widehat{a}}\frac{\eta_a-\pi(a)V(s)}{2}\right)_+|\AM|\rho-\sum_{a\not=\widehat{a}}\eta_a.
	\end{align*}
	We note that the infimum of $g(\eta,P)$ w.r.t variable $\eta_{\widehat{a}}$ is obtained when $\eta_{\widehat{a}}=0$. Thus, we can safely say that the infimum of $g(\eta,P)$ locates in $\RB^{|\AM|}_+$ and $g(0,P)=0$.

	Besides, we also have:
	\begin{align*}
		g(\eta,P)&\ge\sum_{s,a}P_a(s)(\eta_a-\pi(a)V(s))+\left(\max_a\frac{\eta_a-\pi(a)V_{\min}}{2}\right)_+|\AM|\rho-\sum_a\eta_a\\
		&\ge-\frac{1}{1-\gamma}+\left(\max_a\frac{\eta_a-\pi(a)V_{\min}}{2}\right)_+|\AM|\rho.
	\end{align*}
	Note that we have:
	\begin{align*}
		\left(\max_a\frac{\eta_a-\pi(a)V_{\min}}{2}\right)_+&=\max_a\left(\frac{\eta_a-\pi(a)V_{\min}}{2}\right)_+\\
		&\ge\frac{1}{|\AM|}\sum_{a}\left(\frac{\eta_a-\pi(a)V_{\min}}{2}\right)_+\\
		&\ge\frac{1}{2|\AM|}\left(\sum_a\eta_a-\frac{1}{1-\gamma}\right).
	\end{align*}
	Thus, when $\sum_a\eta_a\ge\frac{2+\rho}{\rho(1-\gamma)}$, $g(\eta,P)\ge0$. By convexity of $g(\eta,P)$, the infimum of $g(\eta,P)$ locates in set:
	\begin{align*}
		\left\{\eta\in\RB_+^{|\AM|}\Bigg{|}\sum_{a}\eta_a\le\frac{2+\rho}{\rho(1-\gamma)}\right\}.
	\end{align*}
\end{proof}

For any fixed $\pi\in\Pi$ and $V\in\VM$, we have the following Theorem.

\begin{thm}
    \label{thm: l1-fix-s}
    In the setting of $L_1$ balls, for fixed $V\in\VM$ and $\pi\in\Pi$, the following inequality holds with probability $1-\delta$:
    \begin{align*}
        \left\|\TM_r^\pi V-\widehat{\TM}_r^\pi V\right\|_\infty\le\frac{\gamma(2+\rho)}{\rho(1-\gamma)\sqrt{2n}}\left(1+\sqrt{\log\frac{2|\SM|}{\delta}+|\AM|\log\left(1+2\sqrt{2n}(4+\rho)\right)}\right).
    \end{align*}
\end{thm}

\begin{proof}[Proof of Theorem~\ref{thm: l1-fix-s}]
	Firstly, we denote $g(\eta,P)=\sum_{s,a}P_a(s)(\eta_a-\pi(a)V(s))_+-\sum_{a}\eta_a+(\max_{a,s}\frac{\eta_a-\pi(a)V(s)}{2})_+|\AM|\rho$. We also denote $Y_k^a=\sum_s 1(X_k^a=s)(\eta_a-\pi(a)V(s))_+$ and $Z_k=\sum_a Y_k^a$, where the $X_k^a$ are generated by $P_a(\cdot)$ independently. Thus, we have $g(\eta,\widehat{P})=\frac{1}{n}\sum_{k}Z_k+(\max_{a,s}\frac{\eta_a-\pi(a)V(s)}{2})_+|\AM|\rho-\sum_a\eta_a$. By restricting $\eta$ in the set of Lemma~\ref{lem: l1-value}, we have $0\le Z_k\le\sum_a\eta_a\le\frac{2+\rho}{\rho(1-\gamma)}$ and the $\{Z_k\}$ are i.i.d.\ random variables. By Hoeffding's inequality, with probability $1-\delta$, we have:
	\begin{align*}
		\left|g(\eta,\widehat{P})- g(\eta,P)\right|=\left|g(\eta,\widehat{P})-\EB g(\eta,\widehat{P})\right|\le\frac{2+\rho}{\rho(1-\gamma)}\sqrt{\frac{\log\frac{2}{\delta}}{2n}}.
	\end{align*}
	Next, we turn to bound the deviation over $I_{L_1}$ uniformly. Noticing that, for any two $\widetilde{\eta}$ and $\widehat{\eta}$, we have:
	\begin{align*}
		\left|g(\widetilde{\eta},P)-g(\widehat{\eta},P)\right|\le(2+\frac{\rho}{2})\left\|\widetilde{\eta}-\widehat{\eta}\right\|_1.
	\end{align*}
	Besides, by taking the smallest $\varepsilon$-net of $I_{L_1}$ as $\NM_\varepsilon$ w.r.t.\ metric $\|\cdot\|_1$, the size of $\NM_\varepsilon$ is bounded by:
	\begin{align*}
		|\NM_\varepsilon|\le\left(1+\frac{2(2+\rho)}{\rho(1-\gamma)\varepsilon}\right)^{|\AM|}.
	\end{align*}
	Thus, with probability $1-\delta$, we have:
	\begin{align*}
		\sup_{\eta\in I_{L_1}}\left|g(\eta,\widehat{P})- g(\eta,P)\right|&\le(4+\rho)\varepsilon+\sup_{\eta\in\NM_\varepsilon}\left|g(\eta,\widehat{P})- g(\eta,P)\right|\\
		&\le(4+\rho)\varepsilon+\frac{2+\rho}{\rho(1-\gamma)\sqrt{2n}}\sqrt{\log\frac{2}{\delta}+\log|\NM_\varepsilon|}.
	\end{align*}
	Taking $\varepsilon=\frac{2+\rho}{\rho(1-\gamma)\sqrt{2n}(4+\rho)}$, with probability $1-\delta$, we have:
	\begin{align*}
		\sup_{\eta\in I_{L_1}}\left|g(\eta,\widehat{P})- g(\eta,P)\right|&\le\frac{2+\rho}{\rho(1-\gamma)\sqrt{2n}}\left(1+\sqrt{\log\frac{2}{\delta}+|\AM|\log\left(1+2\sqrt{2n}(4+\rho)\right)}\right).
	\end{align*}
	The final result is obtained by union bound over $s\in\SM$.
\end{proof}

By a similar union bound over $\VM$ and Lemma~\ref{lem: eps-pi}, we obtain the bound of performance gap as Theorem~\ref{thm: l1-union-s} states.

\begin{proof}[Proof of Theorem~\ref{thm: l1-union-s}]
	Firstly, we bound the deviation uniformly with $V\in\VM$. Similar with the proof of Theorem~\ref{thm: l1-union}, we take the smallest $\varepsilon$-net of $\VM$ w.r.t.\ norm $\|\cdot\|_\infty$ and denote it as $\VM_\varepsilon$. By Theorem~\ref{thm: l1-fix-s}, with probability $1-\delta$, we have:
	\begin{align*}
		&\sup_{V\in\VM}\left\|\TM_r^\pi V - \widehat{\TM}_r^\pi V\right\|_\infty\\
		&\le2\gamma\varepsilon+\frac{\gamma(2+\rho)}{\rho(1-\gamma)\sqrt{2n}}\left(1+\sqrt{\log\frac{2|\SM|}{\delta}+|\AM|\log\left(1 {+}2\sqrt{2n}(4 {+} \rho)\right)+\log|\VM_\varepsilon|}\right)\\
		&\le 2\gamma\varepsilon+\frac{\gamma(2 {+} \rho)}{\rho(1 {-}\gamma)\sqrt{2n}}\Bigg{(}1+\sqrt{\log\frac{2|\SM|}{\delta}+|\AM|\log\left(1 {+} 2\sqrt{2n}(4 {+} \rho)\right)+|\SM|\log\left(1 {+} \frac{1}{(1 {-} \gamma)\varepsilon}\right)}\Bigg{)}.
	\end{align*}
	Taking $\varepsilon=\frac{(2+\rho)}{2\rho(1-\gamma)\sqrt{2n}}$, with probability $1-\delta$, we have:
	\begin{align*}
		\sup_{V\in\VM}\left\|\TM_r^\pi V - \widehat{\TM}_r^\pi V\right\|_\infty\le&\frac{\gamma(2+\rho)}{\rho(1-\gamma)\sqrt{2n}}\Bigg{(}2+\\
		&\sqrt{\log\frac{2|\SM|}{\delta}+|\AM|\log\left(1+2\sqrt{2n}(4+\rho)\right)+|\SM|\log\left(1+\frac{2\rho\sqrt{2n}}{2+\rho}\right)}\Bigg{)}.
	\end{align*}
	Next, we bound the deviation uniformly with $\pi\in\Pi$. By Lemma~\ref{lem: eps-pi}, with probability $1-\delta$, we have:
	\begin{align*}
		\sup_{\pi\in\Pi, V\in\VM}&\left\|\TM_r^\pi V - \widehat{\TM}_r^\pi V\right\|_\infty\le\frac{2\gamma\varepsilon}{1-\gamma}+\frac{\gamma(2+\rho)}{\rho(1-\gamma)\sqrt{2n}}\Bigg{(}2+\\
		&\sqrt{\log\frac{2|\SM|}{\delta}+|\AM|\log\left(1+2\sqrt{2n}(4+\rho)\right)+|\SM|\log\left(1+\frac{2\rho\sqrt{2n}}{2+\rho}\right)+\log|\Pi_\varepsilon|}\Bigg{)}.
	\end{align*}
	Taking $\varepsilon=\frac{2+\rho}{\rho\sqrt{2n}}$, by Lemma~\ref{lem: num-eps-pi}, with probability $1-\delta$, we have:
	\begin{align*}
		&\sup_{\pi\in\Pi, V\in\VM}\left\|\TM_r^\pi V - \widehat{\TM}_r^\pi V\right\|_\infty\le\frac{\gamma(2+\rho)}{\rho(1-\gamma)\sqrt{2n}}\Bigg{(}4+\\
		&\sqrt{\log\frac{2|\SM|}{\delta}+|\AM|\log\left(1 {+} 2\sqrt{2n}(4 {+} \rho)\right)+|\SM|\log\left(1 {+} \frac{2\rho\sqrt{2n}}{2 {+} \rho}\right)+|\SM||\AM|\log\left(1 {+} \frac{4\rho\sqrt{2n}}{2 {+} \rho}\right)}\Bigg{)}.
	\end{align*}
	By some calculation, we can simplify the inequality with:
	\begin{align*}
		\sup_{\pi\in\Pi, V\in\VM}\left\|\TM_r^\pi V - \widehat{\TM}_r^\pi V\right\|_\infty\le\frac{\gamma(2+\rho)\sqrt{|\SM||\AM|}}{\rho(1-\gamma)\sqrt{2n}}\Bigg{(}4+\sqrt{\log\frac{2|\SM|(1+2\sqrt{2n}(\rho+4))^3}{\delta}}\Bigg{)}.
	\end{align*}
\end{proof}

\subsubsection*{\textbf{Case 2: $\chi^2$ balls}}
In this case, we set $f(t)=(t-1)^2$ in Example~\ref{eg: f-set-s}. The uncertainty set is formulated as:

\begin{example}[$\chi^2$ balls]
    For each $s\in\SM$, the uncertainty sets are defined as:
    \begin{align*}
        &\PM_{s}(\rho)=\left\{P(\cdot|s,a)\in\Delta(\SM)\Bigg{|}P(\cdot|s,a)\ll P^*(\cdot|s,a),\hspace{2pt}\sum_{s'\in\SM, a\in\AM}\frac{\left(P(s'|s,a)-P^*(s'|s,a)\right)^2}{P^*(s'|s,a)}\le|\AM|\rho\right\},\\
        &\widehat{\PM}_{s}(\rho)=\left\{P(\cdot|s,a)\in\Delta(\SM)\Bigg{|}P(\cdot|s,a)\ll \widehat{P}(\cdot|s,a),\hspace{2pt}\sum_{s'\in\SM, a\in\AM}\frac{\left(P(s'|s,a)-\widehat{P}(s'|s,a)\right)^2}{\widehat{P}(s'|s,a)}\le|\AM|\rho\right\}.
    \end{align*}
    By the $s$-rectangular set assumption, we define $\PM=\bigtimes_{s\in\SM} \PM_s(\rho)$ and $\widehat{\PM}=\bigtimes_{s\in\SM} \widehat{\PM}_s(\rho)$.
\end{example}
Thus, for any fixed $\pi\in\Pi$ and $V\in\VM$, the explicit forms of $\TM_r^\pi$ and $\widehat{\TM}_r^\pi$ are:
\begin{lem}
    \label{lem: chi2-s}
    Under the s-rectangular assumption and $\chi^2$ balls uncertainty set, for each $s\in\SM$, we have:
    \begin{align*}
        \TM_r^\pi V(s)&=R^\pi(s)+\gamma\sup_{\eta\in\RB^{|\AM|}}\left(-\sqrt{(\rho+1)|\AM|}\sqrt{\sum_{s',a}P^*(s'|s,a)(\eta_a-\pi(a|s)V(s'))_+^2}+\sum_{a}\eta_a\right),\\
        \widehat{\TM}_r^\pi V(s)&=R^\pi(s)+\gamma\sup_{\eta\in\RB^{|\AM|}}\left(-\sqrt{(\rho+1)|\AM|}\sqrt{\sum_{s',a}\widehat{P}(s'|s,a)(\eta_a-\pi(a|s)V(s'))_+^2}+\sum_{a}\eta_a\right).
    \end{align*}
\end{lem}
\begin{proof}[Proof of Lemma~\ref{lem: chi2-s}]
	Similar with the proof of Lemma~\ref{lem: l1-s}, when $f(t)=(t-1)^2$, we have:
	\begin{align*}
		f^*(s)=
		\begin{cases}
			\frac{t^2}{4}+t, & s\ge-2, \\
			-1, & s<-2.
		\end{cases}
	\end{align*}
	By Lemma~\ref{lem: f-eq-s}, the optimization problem turns into:
	\begin{align*}
		\sup_{\lambda\ge0, \eta\in\RB^{|\AM|}}-\frac{1}{4\lambda}\sum_{s,a}P_a(s)\left(\eta_a-\pi(a)V(s)\right)_+^2-\lambda|\AM|(\rho+1)+\sum_{a}\eta_a.
	\end{align*}
 By	optimizing over $\lambda$, the problem becomes:
	\begin{align*}
		\sup_{\eta\in\RB^{|\AM|}}-\sqrt{(\rho+1)|\AM|}\sqrt{\sum_{s,a}P_a(s)(\eta_a-\pi(a)V(s))_+^2}+\sum_{a}\eta_a.
	\end{align*}
\end{proof}
Similar with the case of $L_1$ balls, $\TM_r^\pi$ and $\widehat{\TM}_r^\pi$ depend on an $|\AM|$-dimensional vector. By a similar approach with the case of $L_1$ balls under $s$-rectangular assumption, we can also obtain the following result.
\begin{lem}
	\label{lem: chi2-value}
	Denote $g(\eta)=\sqrt{(\rho+1)|\AM|}\sqrt{\sum_{s,a}P_a(s)(\eta_a-\pi(a)V(s))_+^2}-\sum_{a}\eta_a$. The optimal infimum of $g(\eta)$ lies in set 
	\begin{align*}
		I_{\chi^2}=\left\{\eta\in\RB^{|\AM|}\Bigg{|}\eta_a\ge 0, \sum_{a}\eta_a\le\frac{C(\rho)}{(C(\rho)-1)(1-\gamma)}\right\},
	\end{align*}
	where $C(\rho)=\sqrt{1+\rho}$.
\end{lem}

\begin{proof}[Proof of Lemma~\ref{lem: chi2-value}]
	If there exists $\widehat{a}\in\AM$ such that $\eta_{\widehat{a}}\le0$, we have:
	\begin{align*}
		g(\eta)=-\eta_{\widehat{a}}+\sqrt{(\rho+1)|\AM|}\sqrt{\sum_{s,a\not=\widehat{a}}P_a(s)(\eta_a-\pi(a)V(s))_+^2}-\sum_{a\not=\widehat{a}}\eta_a.
	\end{align*}
	Thus, the infimum of $g(\eta)$ could be obtained when $\eta_{\widehat{a}}=0$. Besides, by the Cauchy-Schwartz inequality, we have:
	\begin{align*}
		g(\eta)&\ge\sqrt{(\rho+1)|\AM|}\frac{\sum_{s,a}P_a(s)(\eta_a-\pi(a)V(s))_+}{\sqrt{\sum_{s,a}P_a(s)}}-\sum_a \eta_a\\
		&=\sqrt{1+\rho}\sum_{s,a}P_a(s)(\eta_a-\pi(a)V(s))_+-\sum_a\eta_a\\
		&\ge\sqrt{1+\rho}\sum_{s,a}P_a(s)(\eta_a-\pi(a)V(s))-\sum_a\eta_a\\
		&=\left(\sqrt{1+\rho}-1\right)\sum_a\eta_a - \sqrt{1+\rho}\sum_{s,a}P_a(s)\pi(a)V(s).
	\end{align*}
	Thus, when $\sum_a\eta_a\ge\frac{C(\rho)}{(C(\rho)-1)(1-\gamma)}$, we have:
	\begin{align*}
		g(\eta)\ge\sqrt{1+\rho}\left(\frac{1}{1-\gamma}-\sum_{s,a}P_a(s)\pi(a)V(s)\right)\ge0.
	\end{align*}
	By convexity of $g$, the optimal infimum of $g(\eta)$ locates in set:
	\begin{align*}
		\left\{\eta\in\RB^{|\AM|}\Bigg{|}\eta_a\ge 0, \sum_{a}\eta_a\le\frac{C(\rho)}{(C(\rho)-1)(1-\gamma)}\right\},
	\end{align*}
	where $C(\rho)=\sqrt{1+\rho}$.
\end{proof}
\begin{thm}
    \label{thm: chi2-s-fix}
    In the setting of $\chi^2$ balls, for fixed $V\in\VM$ and $\pi\in\Pi$, the following inequality holds with probability $1-\delta$:
    \begin{align*}
        \left\|\TM_r^\pi V - \widehat{\TM}_r^\pi V\right\|_\infty\le\frac{\gamma C^2(\rho)\sqrt{|\AM|}}{(C(\rho)-1)(1-\gamma)\sqrt{n}}\left(2+\sqrt{2\log\frac{2}{\delta}+2|\AM|\log(1+8\sqrt{n})}\right),
    \end{align*}
    where $C(\rho)=\sqrt{\rho+1}$.
\end{thm}
\begin{proof}[Proof of Theorem~\ref{thm: chi2-s-fix}]
	Similar with the proof of Theorem~\ref{thm: chi2-uni}, we firstly ignore dependence on $s$ and denote $g(\eta,P)=C(\rho)\sqrt{|\AM|}\sqrt{\sum_{s,a}P_a(s)(\eta_a-\pi(a)V(s))_+^2}-\sum_a\eta_a$, where $C(\rho)=\sqrt{\rho+1}$. Besides, we also denote $Y_{k}^a=\sum_s 1(X_k^a=s)(\eta_a-\pi(a)V(s))_+$ and $Z_k = \sqrt{\sum_a (Y_k^a)^2}$, where the $X_k^a$ are generated by $P_a(\cdot)$ independently. Thus, we have $g(\eta,\widehat{P})=C(\rho)\sqrt{\frac{1}{n}\sum_{k=1}^n Z_k^2}-\sum_a\eta_a$. By restricting $\eta$ in the set of Lemma~\ref{lem: chi2-value}, we have $Z_k\le\sqrt{\sum_a \eta_a^2}\le\sum_a \eta_a$ and the $\{Z_k\}$ are i.i.d.\ random variables. By Lemma 6 of \cite{duchi2018learning}, with probability $1-\delta$, we have:
	\begin{align*}
		\left|g(\eta,\widehat{P})-\EB g(\eta,\widehat{P})\right|\le\frac{\sqrt{2|\AM|}C^2(\rho)}{(C(\rho)-1)(1-\gamma)\sqrt{n}}\sqrt{\log\frac{2}{\delta}}.
	\end{align*}
	Besides, by Lemma 8 of \cite{duchi2018learning}, we have:
	\begin{align*}
		\frac{1}{\sqrt{n}}\EB\|Z\|_2\ge\sqrt{\sum_{k=1}^n \EB Z_k^2}-\sqrt{\frac{C(\rho)}{(C(\rho)-1)(1-\gamma)}}\frac{1}{\sqrt{n}}.
	\end{align*}
	Combing these together, with probability $1-\delta$, we have:
	\begin{align*}
		\left|g(\eta,\widehat{P})-g(\eta,P)\right|\le\frac{C^2(\rho)\sqrt{|\AM|}}{(C(\rho)-1)(1-\gamma)\sqrt{n}}\left(1+\sqrt{2\log\frac{2}{\delta}}\right).
	\end{align*}
	Next, we turn to bound the deviation over $I_{\chi^2}$ uniformly. Note that for any two $\widetilde{\eta}$ and $\widehat{\eta}$ we have:
	\begin{align*}
		\left|g(\widetilde{\eta},P)-g(\widehat{\eta}, P)\right|\le (C(\rho)\sqrt{|\AM|}+1)\left\|\widetilde{\eta}-\widehat{\eta}\right\|_1.
	\end{align*}
	By taking the smallest $\varepsilon$-net of $I_{\chi^2}$ as $\NM_\varepsilon$ w.r.t.\ metric $\|\cdot\|_1$, the size of $\NM_\varepsilon$ is bounded by:
	\begin{align*}
		|\NM_\varepsilon|\le\left(1+\frac{2C(\rho)}{(C(\rho)-1)(1-\gamma)\varepsilon}\right)^{|\AM|}.
	\end{align*}
	Thus, we have:
	\begin{align*}
		\sup_{\eta\in I_{\chi^2}}\left|g(\eta,\widehat{P})-g(\eta,P)\right|\le 2(1+C(\rho)\sqrt{|\AM|})\varepsilon+\sup_{\eta\in\NM_\varepsilon}\left|g(\eta,\widehat{P})-g(\eta,P)\right|.
	\end{align*}
	By taking $\varepsilon=\frac{C^2(\rho)\sqrt{|\AM|}}{2(1+\sqrt{|\AM|}C(\rho))(C(\rho)-1)(1-\gamma)\sqrt{n}}$, with probability $1-\delta$, we have:
	\begin{align*}
		&\sup_{\eta\in I_{\chi^2}}\left|g(\eta,\widehat{P})-g(\eta,P)\right|\le\frac{C^2(\rho)\sqrt{|\AM|}}{(C(\rho) {-} 1)(1 {-} \gamma)\sqrt{n}}\left(2+\sqrt{2\log\frac{2|\NM_\varepsilon|}{\delta}}\right)\\
		&\le\frac{C^2(\rho)\sqrt{|\AM|}}{(C(\rho)-1)(1-\gamma)\sqrt{n}}\left(2+\sqrt{2\log\frac{2}{\delta}+2|\AM|\log\left(1+\frac{4\sqrt{n}(1+C(\rho)\sqrt{|\AM|})}{C(\rho)\sqrt{|\AM|}}\right)}\right).
	\end{align*}
	Noting that $C(\rho)=\sqrt{\rho+1}\ge1$ and $|\AM|\ge1$, we can simplify the inequality with:
	\begin{align*}
		\sup_{\eta\in I_{\chi^2}}\left|g(\eta,\widehat{P})-g(\eta,P)\right|\le\frac{C^2(\rho)\sqrt{|\AM|}}{(C(\rho)-1)(1-\gamma)\sqrt{n}}\left(2+\sqrt{2\log\frac{2}{\delta}+2|\AM|\log(1 {+} 8\sqrt{n})}\right).
	\end{align*}
	Thus, the final result is obtained by union bound over $s\in\SM$.
\end{proof}

Finally, by union bounding over $\VM$ and Lemma~\ref{lem: eps-pi}, we obtain the bound of performance gap as Theorem~\ref{thm: chi2-s-union} states.

\begin{proof}[Proof of Theorem~\ref{thm: chi2-s-union}]
	Firstly, we bound the deviation uniformly with $V\in\VM$. Similar with the proof of Theorem~\ref{thm: l1-union}, we take the smallest $\varepsilon$-net of $\VM$ w.r.t.\ norm $\|\cdot\|_\infty$ and denote it as $\VM_\varepsilon$. By Theorem~\ref{thm: chi2-s-fix}, with probability $1-\delta$, we have:
	\begin{align*}
		\sup_{V\in\VM}\left\|\TM_r^\pi V-\widehat{\TM}_r^\pi V\right\|_\infty\le&2\gamma\varepsilon+\frac{\gamma C^2(\rho)\sqrt{|\AM|}}{(C(\rho)-1)(1-\gamma)\sqrt{n}}\Bigg{(}2+\\
		&\sqrt{2\log\frac{2|\SM|}{\delta}+2|\AM|\log\left(1+8\sqrt{n}\right)+2|\SM|\log\left(1+\frac{1}{(1-\gamma)\varepsilon}\right)}\Bigg{)}.
	\end{align*}
	By taking $\varepsilon=\frac{C^2(\rho)\sqrt{|\AM|}}{(C(\rho)-1)(1-\gamma)\sqrt{n}}$, with probability $1-\delta$, we have:
	\begin{align*}
		\sup_{V\in\VM}&\left\|\TM_r^\pi V-\widehat{\TM}_r^\pi V\right\|_\infty\le \frac{\gamma C^2(\rho)\sqrt{|\AM|}}{(C(\rho)-1)(1-\gamma)\sqrt{n}}\Bigg{(}4+\\
		&\sqrt{2\log\frac{2|\SM|}{\delta}+2|\AM|\log\left(1+8\sqrt{n}\right)+2|\SM|\log\left(1+\frac{\sqrt{n}(C(\rho)-1)}{C^2(\rho)\sqrt{|\AM|}}\right)}\Bigg{)}.
	\end{align*}

	Next, we bound the deviation uniformly with $\pi\in\Pi$. By Lemma~\ref{lem: eps-pi}, with probability $1-\delta$, we have:
	\begin{align*}
		&\sup_{\pi\in\Pi, V\in\VM}\left\|\TM_r^\pi V- \widehat{\TM}_r^\pi V\right\|_\infty\le\frac{2\gamma\varepsilon}{1-\gamma}+\frac{\gamma C^2(\rho)\sqrt{|\AM|}}{(C(\rho)-1)(1-\gamma)\sqrt{n}}\Bigg{(}4+\\
		&\sqrt{2\log\frac{2|\SM|}{\delta}+2|\AM|\log\left(1+8\sqrt{n}\right)+2|\SM|\log\left(1+\frac{\sqrt{n}(C(\rho)-1)}{C^2(\rho)\sqrt{|\AM|}}\right)+2\log|\Pi_\varepsilon|}\Bigg{)}.
	\end{align*}
	By taking $\varepsilon=\frac{C^2(\rho)\sqrt{|\AM|}}{(C(\rho)-1)\sqrt{n}}$, with probability $1-\delta$, we have:
	\begin{align*}
		&\sup_{\pi\in\Pi, V\in\VM}\left\|\TM_r^\pi V- \widehat{\TM}_r^\pi V\right\|_\infty\le\frac{\gamma C^2(\rho)\sqrt{|\AM|}}{(C(\rho)-1)(1-\gamma)\sqrt{n}}\Bigg{(}6+\\
		&\sqrt{2\log\frac{2|\SM|}{\delta}+2|\AM|\log\left(1+8\sqrt{n}\right)+2|\SM|\log\left(1+\frac{\sqrt{n}(C(\rho)-1)}{C^2(\rho)\sqrt{|\AM|}}\right)+2\log|\Pi_\varepsilon|}\Bigg{)}.
	\end{align*}
	By Lemma~\ref{lem: num-eps-pi} and some calculation, with probability $1-\delta$, we have:
	\begin{align*}
		&\sup_{\pi\in\Pi, V\in\VM}\left\|\TM_r^\pi V- \widehat{\TM}_r^\pi V\right\|_\infty\le\frac{\gamma C^2(\rho)\sqrt{|\AM|}}{(C(\rho)-1)(1-\gamma)\sqrt{n}}\Bigg{(}6+\\
		&\sqrt{2\log\frac{2|\SM|}{\delta}+2|\AM|\log\left(1+8\sqrt{n}\right)+4|\SM||\AM|\log\left(1+\frac{4\sqrt{n}(C(\rho)-1)}{C^2(\rho)\sqrt{|\AM|}}\right)}\Bigg{)}.
	\end{align*}
	By $C(\rho)\ge1$ and $|\AM|\ge1$, we can simplify it as:
	\begin{align*}
		&\sup_{\pi\in\Pi, V\in\VM}\left\|\TM_r^\pi V- \widehat{\TM}_r^\pi V\right\|_\infty\\
		&\le\frac{\gamma C^2(\rho)\sqrt{|\SM||\AM|^2}}{(C(\rho)-1)(1-\gamma)\sqrt{n}}\left(6+\sqrt{2\log\frac{2|\SM|(1+8\sqrt{n})(1+4\sqrt{n}(C(\rho)-1))^2}{\delta}}\right).
	\end{align*}
\end{proof}

\subsubsection*{\textbf{Case 3: KL balls}}
In this case, we set $f(t)=t\log t$ in Example~\ref{eg: f-set-s}. The uncertainty set is formulated as:
\begin{example}[KL balls]
    For each $(s,a)\in\SM\times\AM$, the uncertainty sets are defined as:
    \begin{align*}
        &\PM_{s}(\rho)=\left\{P(\cdot|s,a)\in\Delta(\SM)\Bigg{|}P(\cdot|s,a)\ll P^*(\cdot|s,a),\hspace{2pt}\sum_{s'\in\SM, a\in\AM}P(s'|s,a)\log\frac{P(s'|s,a)}{P^*(s'|s,a)}\le|\AM|\rho\right\},\\
        &\widehat{\PM}_{s}(\rho)=\left\{P(\cdot|s,a)\in\Delta(\SM)\Bigg{|}P(\cdot|s,a)\ll \widehat{P}(\cdot|s,a),\hspace{2pt}\sum_{s'\in\SM, a\in\AM}P(s'|s,a)\log\frac{P(s'|s,a)}{\widehat{P}(s'|s,a)}\le|\AM|\rho\right\}.
    \end{align*}
    By the $s$-rectangular set assumption, we define $\PM=\bigtimes_{s\in\SM} \PM_s(\rho)$ and $\widehat{\PM}=\bigtimes_{s\in\SM} \widehat{\PM}_s(\rho)$.
\end{example}
Thus, for any fixed $\pi\in\Pi$ and $V\in\VM$, the explicit forms of $\TM_r^\pi$ and $\widehat{\TM}_r^\pi$ are:
\begin{lem}
    \label{lem: kl-s}
    Under the s-rectangular assumption and KL balls uncertainty set, for each $s\in\SM$, we have:
    \begin{align*}
        \TM^\pi_r V (s)&=R^\pi(s) + \gamma\sup_{\lambda\ge0}\left(-\lambda|\AM|\rho-\lambda\sum_{a}\log\sum_{s'}P^*(s'|s,a)\exp\left(-\frac{\pi(a|s)V(s')}{\lambda}\right)\right),\\
        \widehat{\TM}^\pi_r V (s)&=R^\pi(s) + \gamma\sup_{\lambda\ge0}\left(-\lambda|\AM|\rho-\lambda\sum_{a}\log\sum_{s'}\widehat{P}(s'|s,a)\exp\left(-\frac{\pi(a|s)V(s')}{\lambda}\right)\right).
    \end{align*}
\end{lem}

\begin{proof}[Proof of Lemma~\ref{lem: kl-s}]
	Similar with the proof of Lemma~\ref{lem: l1-s}, when $f(t)=t\log t$, we have $f^*(s)=\exp(s-1)$. By Lemma~\ref{lem: f-eq-s}, we optimize over $\eta$, the optimization problem turns into:
	\begin{align*}
		\sup_{\lambda\ge0}-\lambda\sum_{a\in\AM}\log\sum_{s\in\SM}P^*_a(s)\exp\left(-\frac{\pi(a)V(s)}{\lambda}\right)-\lambda|\AM|\rho.
	\end{align*}
\end{proof}

In this case, $\TM_r^\pi$ and $\widehat{\TM}_r^\pi$ only depend on a scalar $\lambda$, which is similar with the $(s,a)$-rectangular assumption. We denote $g(\lambda,P) = \lambda|\AM|\rho + \lambda\sum_a \log\sum_{s}P_a(s)\exp(-\pi(a)V(s)/\lambda)$. Similar with the case of the $(s,a)$-rectangular assumption, when $\lambda\ge\frac{1}{|\AM|\rho(1-\gamma)}$, $g(\lambda,P)\ge0$ and it is monotonically increasing in $\lambda$. Thus, the optimal value $\lambda^*$ locates in $[0,\frac{1}{|\AM|\rho(1-\gamma)}]$. Thus, for fixed $V\in\VM$ and $\pi\in\Pi$, we have:
\begin{thm}
    \label{thm: kl-s-fix}
    In the setting of KL balls, for fixed $V\in\VM$ and $\pi\in\Pi$, the following inequality holds with probability $1-\delta$:
    \begin{align*}
        \left\|\TM_r^\pi V-\widehat{\TM}_r^\pi V\right\|_\infty\le\frac{2\gamma}{\rho\underline{p}(1-\gamma)\sqrt{n}}\sqrt{\log\frac{2|\SM||\AM|}{\delta}},
    \end{align*}
    where $\underline{p}=\min_{P^*(s'|s,a)>0}P^*(s'|s,a)$.
\end{thm}

\begin{proof}[Proof of Theorem~\ref{thm: kl-s-fix}]
	Similar with the proof of Theorem~\ref{thm: kl-fix}, we denote\\ $g(P)=\inf_{\lambda\in[0,1/|\AM|\rho(1-\gamma)]}\lambda|\AM|\rho+\lambda\sum_a\log\sum_{s}P_a(s)\exp(-\pi(a)V(s)/\lambda)$. We have:
	\begin{align*}
		&\left|g(\widehat{P})-g(P)\right|\\
		&\le\sup_{\lambda\in[0,\frac{1}{|\AM|\rho(1-\gamma)}]}\left|\lambda\sum_a\left(\log\sum_s \widehat{P}_a(s)\exp(-\frac{\pi(a)V(s)}{\lambda})-\log\sum_s P_a(s)\exp(-\frac{\pi(a)V(s)}{\lambda})\right)\right|\\
		&\le\frac{1}{|\AM|\rho(1-\gamma)}\sup_{\lambda\in[0,\frac{1}{|\AM|\rho(1-\gamma)}]}\left|\sum_a\log\left(1+\frac{\sum_s \left(\widehat{P}_a(s)-P_a(s)\right)\exp\left(-\frac{\pi(a)V(s)}{\lambda}\right)}{\sum_s P_a(s)\exp\left(-\frac{\pi(a)V(s)}{\lambda}\right)}\right)\right|.
	\end{align*}
	Noting that $|\log(1+x)|\le2|x|$ when $|x|\le\frac{1}{2}$, we  have:
	\begin{align*}
		\left|g(\widehat{P})-g(P)\right|&\le\frac{2}{|\AM|\rho(1-\gamma)}\sum_a\max_{s}\left|\frac{\widehat{P}_a(s)}{P_a(s)}-1\right|\le\frac{2}{\rho(1-\gamma)}\max_{s,a}\left|\frac{\widehat{P}_a(s)}{P_a(s)}-1\right|.
	\end{align*}
	Denote $\underline{p}=\min_{P(s'|s,a)>0}P(s'|s,a)$. The Hoeffding's inequality tells us:
	\begin{align*}
		\PB\left(\max_{s,a}\left|\frac{\widehat{P}_a(s)}{P_a(s)}-1\right|\ge\sqrt{\frac{1}{n\underline{p}^2}\log\frac{2|\SM||\AM|}{\delta}}\right)\le\delta.
	\end{align*}
	Thus, with probability $1-\delta$, we have:
	\begin{align*}
		\left|g(\widehat{P})-g(P)\right|\le\frac{2}{\rho(1-\gamma)}\sqrt{\frac{1}{n\underline{p}^2}\log\frac{2|\SM||\AM|}{\delta}}.
	\end{align*}
	The final result is obtained by union bound over $s\in\SM$.
\end{proof}

By union bound over $\VM$ and Lemma~\ref{lem: eps-pi}, we obtain Theorem~\ref{thm: kl-s-union}.

\begin{proof}[Proof of Theorem~\ref{thm: kl-s-union}]
	Firstly, we take the smallest $\varepsilon$-net of $\VM$ w.r.t norm $\|\cdot\|_\infty$ and denote it as $\VM_\varepsilon$. By Theorem~\ref{thm: kl-s-fix}, with probability $1-\delta$, we have:
	\begin{align*}
		\sup_{V\in\VM}\left\|\TM_r^\pi V-\widehat{\TM}_r^\pi V\right\|_\infty\le 2\gamma\varepsilon+\frac{2\gamma}{\rho\underline{p}(1-\gamma)\sqrt{n}}\sqrt{\log\frac{2|\SM|^2|\AM|}{\delta}+|\SM|\log\left(1+\frac{1}{(1-\gamma)\varepsilon}\right)}.
	\end{align*}
	By taking $\varepsilon=\frac{1}{\rho\underline{p}(1-\gamma)\sqrt{n}}$, with probability $1-\delta$, we have:
	\begin{align*}
		\sup_{V\in\VM}\left\|\TM_r^\pi V-\widehat{\TM}_r^\pi V\right\|_\infty\le \frac{2\gamma}{\rho\underline{p}(1-\gamma)\sqrt{n}}\left(1+\sqrt{\log\frac{2|\SM|^2|\AM|}{\delta}+|\SM|\log\left(1+\rho\underline{p}\sqrt{n}\right)}\right).
	\end{align*}
	Next, we bound the deviation uniformly with $\pi\in\Pi$. By Lemma~\ref{lem: eps-pi}, with probability $1-\delta$, we have:
	\begin{align*}
		&\sup_{\pi\in\Pi, V\in\VM}\left\|\TM_r^\pi V -\widehat{\TM}_r^\pi V\right\|_\infty\\
		&\le\frac{2\gamma\varepsilon}{1-\gamma}+\frac{2\gamma}{\rho\underline{p}(1-\gamma)\sqrt{n}}\left(1+\sqrt{\log\frac{2|\SM|^2|\AM|}{\delta}+|\SM|\log\left(1+\rho\underline{p}\sqrt{n}\right)+\log|\Pi_\varepsilon|}\right).
	\end{align*}
	By taking $\varepsilon=\frac{1}{\rho\underline{p}\sqrt{n}}$, with probability $1-\delta$, we have:
	\begin{align*}
		\sup_{\pi\in\Pi, V\in\VM}\left\|\TM_r^\pi V -\widehat{\TM}_r^\pi V\right\|_\infty\le\frac{2\gamma}{\rho\underline{p}(1-\gamma)\sqrt{n}}\left(2+\sqrt{\log\frac{2|\SM|^2|\AM|}{\delta}+|\SM|\log\left(1+\rho\underline{p}\sqrt{n}\right)+\log|\Pi_\varepsilon|}\right).
	\end{align*}
	By Lemma~\ref{lem: num-eps-pi}, with probability $1-\delta$, we have:
	\begin{align*}
		\sup_{\pi\in\Pi, V\in\VM}\left\|\TM_r^\pi V -\widehat{\TM}_r^\pi V\right\|_\infty&\le\frac{2\gamma}{\rho\underline{p}(1-\gamma)\sqrt{n}}\left(2+\sqrt{\log\frac{2|\SM|^2|\AM|}{\delta}+2|\SM||\AM|\log\left(1+4\rho\underline{p}\sqrt{n}\right)}\right)\\
		&\le\frac{2\gamma\sqrt{|\SM||\AM|}}{\rho\underline{p}(1-\gamma)\sqrt{n}}\left(2+\sqrt{2\log\frac{2|\SM|^2|\AM|(1+4\rho\underline{p}\sqrt{n})}{\delta}}\right).
	\end{align*}
\end{proof}

%% file: tex/apdix/finite_lb.tex
First of all, we consider a simple MDP ($|\SM|=2$, $|\AM|=1$) as Fig.~\ref{fig: simplemdp} shows, where $r(z_0)=1$ and $r(z_1)=0$.
\begin{figure}[htbp!]
    \centering
    \includegraphics[scale=1]{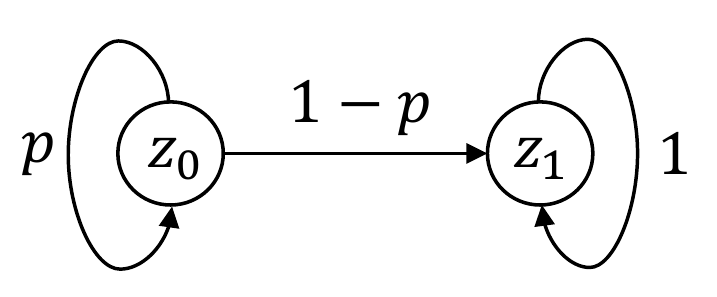}
    \caption{Simple MDP}
    \label{fig: simplemdp}
\end{figure}

We can verify that $V(z_0)=\frac{1}{1-\gamma p}$ and $V(z_1)=0$. Denoting $g(p)=\inf_{D_f(q\|p)\le\rho}q$ where $D_f(q\|p)=pf(\frac{q}{p})+(1-p)f(\frac{1-q}{1-p})$, we have $V_r(z_0)=\frac{1}{1-\gamma g(p)}$ and $V_r(z_1)=0$.

\begin{lem}
    \label{lem: f_prop}
    Assuming $f(t)$ is strictly convex at $t=1$ and continuous in $t\in\RB^+$, we have the following properties:
    \begin{enumerate}[label=(\alph*)]
        \item $D_f(q\|p)$ is non-increasing w.r.t. $q$ when $q\le p$;
        \item $D_f(q\|p)$ is non-decreasing w.r.t. $p$ when $p\ge q$;
        \item $D_f(q\|p)$ is jointly convex.
    \end{enumerate}
\end{lem}

\begin{proof}
    (a) For any $q_1\le q_2\le p$, there exists $\lambda\in[0,1]$ s.t. $q_2=\lambda q_1 +(1-\lambda)p$, thus we have:
    \begin{align*}
        D_f(q_2\|p)\le\lambda D_f(q_1\|p)+(1-\lambda) D_f(p\|p)\le D_f(q_1\|p).
    \end{align*}

    (b) For any $q\le p_2\le p_1$, there exists $\lambda\in[0,1]$ s.t. $p_2=\lambda p_1 +(1-\lambda)q$, thus we have:
    \begin{align*}
        D_f(q\|p_2)\le\lambda D_f(q\|p_1)+(1-\lambda)D_f(q\|q)\le D_f(q\|p_1).
    \end{align*}
    
    (c) Note that $p f(\frac{q}{p})$ is convex by the fact that it's the perspective function of $f$. Thus, $D_f(q\|p)$ is jointly convex.
\end{proof}

\begin{lem}
    \label{lem: g_prop}
    $g(p)$ satisfies the following properties:
    \begin{enumerate}[label=(\alph*)]
        \item $g(p)\le p$;
        \item $D_f(g(p)\|p)=\rho$;
        \item For any $\delta>0$, $p\in(0,1)$ and $p+\delta<1$, we have $g(p+\delta)\ge g(p)$;
        \item $g(p)$ is convex for $p\in(0,1)$. Thus $g(p)$ is differentialble a.s.;
        \item $g'(p)=-\frac{\partial_p D_f(g(p)\|p)}{\partial_q D_f(g(p)\|p)}$.
    \end{enumerate}
\end{lem}

\begin{proof}
    (a) By definition of $g(p)$, (a) is trivial. 
    
    (b) If $D_f(g(p)\|p)<\rho$, by continuous of $f$, there exists $\delta>0$ s.t. $D_f(g(p)\|p)\le D_f(g(p)-\delta\|p)<\rho$, which leads to a contradiction.

    (c) By Lemma~\ref{lem: f_prop} and (b), we have:
    \begin{align*}
        D_f(g(p+\delta)\|p)\le D_f(g(p+\delta)\|p+\delta)=D_f(g(p)\|p).
    \end{align*}
    Thus, we conclude $g(p+\delta)\ge g(p)$.

    (d) For any $\lambda\in[0,1]$ and $p_1,p_2\in(0,1)$, by joint convexity of $D_f$, we have:
    \begin{align*}
        D_f((1-\lambda)g(p_1)+\lambda g(p_2)\|(1-\lambda)p_1+\lambda p_2)\le (1-\lambda)D_f(g(p_1)\|p_1)+\lambda D_f(g(p_2)\|p_2)=\rho.
    \end{align*}
    Thus, we have $g((1-\lambda)p_1+\lambda p_2)\le(1-\lambda)g(p_1)+\lambda g(p_2)$.

    (e) By (b), we have:
    \begin{align*}
        0=\frac{d D_f(g(p)\|p)}{d p}=g'(p) \partial_q D_f(g(p)\|p)+\partial_p D_f(g(p)\|p).
    \end{align*}
\end{proof}


\begin{thm}
    \label{thm: simple_lb}
    By choosing $\delta=\frac{2\varepsilon(1-\gamma g(p))^2}{\gamma g'(p)}$, for every RL algorithm $\mathscr{A}$, there exist a simple MDP (Fig.~\ref{fig: simplemdp}) and constants $c_1, c_2>0$, such that:
    \begin{align*}
        \PB\left(|V_r(z_0, p_i)-V_r^{\mathscr{A}}(z_0)|>\varepsilon|p_i\right)>\frac{1}{c_2}\exp\left(-\frac{c_1\varepsilon^2(1-\gamma g(p))^4 n}{(g'(p))^2p(1-p)}\right),
    \end{align*}
    where $p_i\in\{p, p+\delta\}$.
\end{thm}

\begin{proof}
    By Lemma~\ref{lem: g_prop}, we have:
    \begin{align*}
        \frac{1}{1-\gamma g(p+\delta)}-\frac{1}{1-\gamma g(p)}=\frac{\gamma(g(p+\delta)-g(p))}{(1-\gamma g(p+\delta))(1-\gamma g(p))}\ge\frac{\gamma\delta g'(p)}{(1-\gamma g(p))^2}=2\varepsilon.
    \end{align*}
    Thus, by Lemma 17 in \citet{azar2013minimax}, our results are concluded.
\end{proof}

Thus, with given probability $\delta$ in Theorem~\ref{thm: simple_lb}, the sample complexity is:
\begin{align*}
    n=\frac{(g'(p))^2 p(1-p)}{c_1\varepsilon^2(1-\gamma g(p))^4}\log\frac{1}{c_2\delta}.
\end{align*}

Next, we extend the result to general MDPs as Fig.~\ref{fig: gen_mdp} shows, which is the same with \citet{azar2013minimax}. 
\begin{figure}[htbp!]
    \centering
    \includegraphics[scale=0.8]{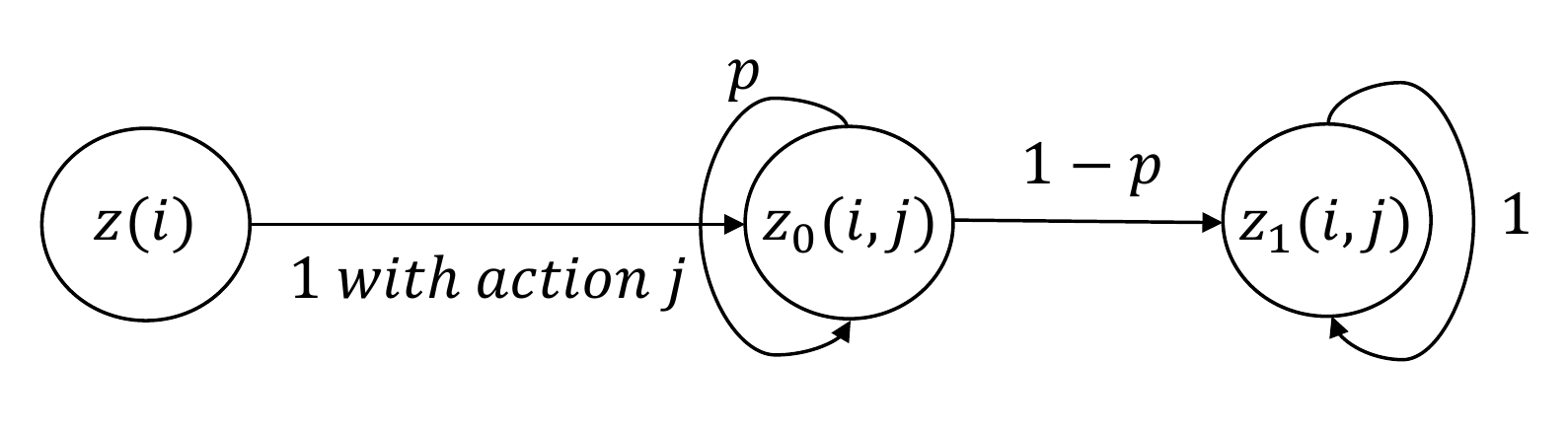}
    \caption{General MDP}
    \label{fig: gen_mdp}
\end{figure}

For each $z(i)$, $i\in[|\SM|]$, $z(i)$ can be transited to $z_0(i,j)$ with action $j\in[|\AM|]$. Thus the total state-action space is $3|\SM||\AM|$. By Lemma 18 in \citet{azar2013minimax}, we obtain the final sample complexity lower bound:
\begin{align*}
    \widetilde{\Omega}\left(\frac{|\SM||\AM|(g'(p))^2 p(1-p)}{\varepsilon^2(1-\gamma g(p))^4}\right).
\end{align*}

\begin{cor}[Lower bound of $L_1$]
    By choosing $f(t)=|t-1|$, the lower bound is:
    \begin{align*}
        \Omega\left(\frac{|\SM||\AM|(1-\gamma)}{\varepsilon^2}\min\left\{\frac{1}{(1-\gamma)^4}, \frac{1}{\rho^4}\right\}\right).
    \end{align*}
\end{cor}

\begin{proof}
    If $f(t)=|t-1|$, we know that $D_f(q\|p)=2|q-p|$. Thus,
    \begin{align*}
        &g(p)=\inf_{D_f(q\|p)\le\rho}q=p-\frac{\rho}{2},\\
        &g'(p)=1.
    \end{align*}
    By taking $p=\frac{2\gamma-1}{\gamma}$, where $\gamma\in(\frac{3}{4},1)$, we have:
    \begin{align*}
        \frac{(g'(p))^2 p(1-p)}{c_1\varepsilon^2(1-\gamma g(p))^4}\log\frac{1}{c_2\delta}\ge\Omega\left(\frac{1-\gamma}{\varepsilon^2}\min\left\{\frac{1}{(1-\gamma)^4}, \frac{1}{\rho^4}\right\}\right).
    \end{align*}
\end{proof}

\begin{cor}[Lower bound of $\chi^2$]
    By choosing $f(t)=(t-1)^2$, the lower bound is:
    \begin{align*}
        \Omega\left(\frac{|\SM||\AM|}{\varepsilon^2(1-\gamma)^2}\min\left\{\frac{1}{1-\gamma}, \frac{1}{\rho}\right\}\right).
    \end{align*}
\end{cor}

\begin{proof}
    If $f(t)=(t-1)^2$, we know that $D_f(q\|p)=\frac{(q-p)^2}{p(1-p)}$. Thus,
    \begin{align*}
        &g(p)=\inf_{D_f(q\|p)\le\rho}q=p-\sqrt{\rho p(1-p)},\\
        &g'(p)=1+\frac{\sqrt{\rho}(2p-1)}{2\sqrt{p(1-p)}}.
    \end{align*}
    By taking $p=\frac{2\gamma-1}{\gamma}$, where $\gamma\in(\frac{3}{4},1)$, we have:
    \begin{align*}
        \frac{(g'(p))^2 p(1-p)}{c_1\varepsilon^2(1-\gamma g(p))^4}\log\frac{1}{c_2\delta}\ge\Omega\left(\frac{1}{\varepsilon^2(1-\gamma)^2}\min\left\{\frac{1}{1-\gamma}, \frac{1}{\rho}\right\}\right).
    \end{align*}
\end{proof}


%% file: tex/res-off.tex
Different from Section~\ref{sec: gen}, we assume there is only an access to an given offline dataset $\DM=\{(s_k,a_k,s_k', r_k)\}_{k=1}^n$, where the reward function is deterministic as Remark~\ref{rem: deter-r} states. The offline dataset is generated by the following rule:
\begin{align}
    (s_k, a_k)\sim \nu(s,a), \hspace{4pt}s_k'\sim P^*(\cdot|s_k, a_k).
\end{align}
Here $\nu$ is a probability measure on space $\SM\times\AM$. We also assume $\nu_{\min}=\min_{s,a}\nu(s,a)>0$ and discuss with this assumption in Remark~\ref{rem: concen}. Next, we  construct an estimated transition probability from the dataset $\DM$:
\begin{align}
    \widehat{P}(s'|s,a)=\frac{\sum_{k=1}^n 1({s_k, a_k, s_k'}=(s, a, s'))}{\sum_{k=1}^n 1({s_k, a_k}=(s, a))}:=\frac{n(s,a,s')}{n(s,a)}.
\end{align}

Lemma~\ref{lem: off-dis} tells us the sample size of $n(s,a)$ is large with high probability. Thus, we can extend the results of the generative model to the offline dataset with a little modification in the $(s,a)$-rectangular assumption as follows. We always assume $n\ge\frac{8}{\nu_{\min}}\log\frac{2|\SM||\AM|}{\delta}$ in the following results.

\begin{lem}
    \label{lem: off-dis}
    Denoting $\nu_{\min}=\min_{s,a,\nu(s,a)>0}\nu(s,a)$, when $n\ge\frac{8}{\nu_{\min}}\log\frac{|\SM||\AM|}{\delta}$, we have:
    \begin{align*}
        \PB\left(\forall (s,a)\in\SM\times\AM, \; n(s,a)\ge\frac{n}{2}\nu_{\min}\right)\ge1-\delta.
    \end{align*}
\end{lem}
\begin{rem}
    \label{rem: concen}
    The assumption $\nu_{\min}>0$ is similar with concentrability assumption in \citet{munos2003error} and \citet{chen2019information}, which stands for that the all admissible distributions can be dominated by $\nu$. In fact, $1/\nu_{\min}$ is an upper bound of the concentrability coefficient in \citet{chen2019information}.
\end{rem}

\subsection{Results under the $(s,a)$-rectangular assumption}

Under the $(s,a)$-rectangular assumption, we can fix the randomness coming from $\nu$ and condition on the event $E=\{\forall (s,a),n(s,a)\ge\frac{n}{2}\nu_{\min}\}$. The following results could be obtained by extending the analysis in Section~\ref{sec: gen}.

\begin{thm}
    \label{thm: finite-sa-off}
    Under the offline dataset, $(s,a)$-rectangular assumption, $\nu_{\min}>0$, the following results hold:
    \begin{enumerate}[label=(\alph*), ref=\thethm (\alph*)]
        \item \label{thm: l1-union-off} Setting $f(t)=|t-1|$ in Example~\ref{eg: f-set} ($L_1$ balls), with probability $1-\delta$:
        \begin{align*}
            \max_{\pi}V_r^\pi(\mu)-V_r^{\widehat{\pi}}(\mu)\le\frac{2(2+\rho)\gamma\sqrt{|\SM|}}{\rho(1-\gamma)^2\sqrt{n\nu_{\min}}}\left(2+\sqrt{\log\frac{8|\SM||\AM|^2(1+2(2+\rho)\sqrt{2n})^2}{\delta(2+\rho)}}\right).   
        \end{align*}
        \item \label{thm: chi2-union-off} Setting $f(t)=(t-1)^2$ in Example~\ref{eg: f-set} ($\chi^2$ balls), with probability $1-\delta$:
        \begin{align*}
            \max_{\pi}V_r^\pi(\mu)-V_r^{\widehat{\pi}}(\mu)\le\frac{2C^2(\rho)\gamma\sqrt{|\SM|}}{(C(\rho){-}1)(1{-}\gamma)^2\sqrt{n\nu_{\min}}}\left(4{+}\sqrt{2\log\frac{4|\SM||\AM|^2(1{+}(C(\rho) {+} 3)\sqrt{n})^2}{\delta C^2(\rho)}}\right),
        \end{align*}
        where $C(\rho)=\sqrt{1+\rho}$.
        \item \label{thm: kl-union-off} Setting $f(t)=t\log t$ in Example~\ref{eg: f-set} (KL balls), with probability $1-\delta$:
        \begin{align*}
            \max_{\pi}V_r^\pi(\mu)-V_r^{\widehat{\pi}}(\mu)\le\frac{4\gamma\sqrt{|\SM|}}{\rho(1-\gamma)^2\underline{p}\sqrt{n\nu_{\min}}}\left(1+\sqrt{2\log\frac{2|\SM|^2|\AM|^2(1+\rho\underline{p}\sqrt{n})}{\delta}}\right),
        \end{align*}
        where $\underline{p}=\min_{P^*(s'|s,a)>0}P^*(s'|s,a)$.
    \end{enumerate}
\end{thm}

Different from the results in the setting of the generative model, to achieve $\varepsilon$ performance gap, the sample complexity should be $\widetilde{O}\left(\frac{|\SM|}{\nu_{\min}\varepsilon^2\rho^2(1-\gamma)^4}\right)$, up to some logarithmic factors. When the dataset is balanced by $\nu(s,a)\approx\frac{1}{|\SM||\AM|}$ for all $(s,a)$ pairs, the sample complexity matches with the setting of the generative model.

\subsection{Results under the $s$-rectangular assumption}

However, the situation becomes complicated when the $s$-rectangular is assumed. In Section~\ref{sec: gen}, we assume the dataset is generated by a generative model, where the $n(s,a)$ are all the same. When we face with the setting of the offline dataset, the $n(s,a)$ are not the same. Inspired by \citet{yin2020near}, we construct a new dataset $\DM'\subset\DM$ (modified offline dataset), where we only aggregate the first $n'$ samples  for each $(s,a)$ pair in $\DM$, and $n':=n'(s,a)=\min_{s,a}n(s,a)$ for any $(s,a)\in\SM\times\AM$. From Lemma~\ref{lem: off-dis}, we know that $n'$ is large with high probability. Thus, we can extend our results in Section~\ref{sec: gen} in this modified offline dataset setting as follows. Prior to that, we also detail the reason why we construct a modified dataset under the $s$-rectangular assumption in Remark~\ref{rem: why-modify}.

\begin{rem}
    \label{rem: why-modify}
    Here we give a high-level explanation on why we construct a new offline dataset $\DM'$. Here we assume $s$ is a fixed state. In the generative model setting, where the $n(s,a)$ are all the same, assuming the $Y_k^{s,a}$ are i.i.d.\ bounded random variables for each $(s, a)$ pair, we can always deal with the concentrabtion bound of $f(\sum_{a} \frac{1}{n(s,a)}\sum_{k=1}^{n(s,a)}Y_k^{s,a})$ by constructing a new sum of i.i.d.\ random variables $\frac{1}{n}\sum_{k=1}^n Z_k^s$, where $f$ is a Lipschitz function, $n=n(s,a)$ and $Z_k^s=\sum_{a} Y_k^{s,a}$. However, in the offline dataset and assuming randomness caused by $\nu$ is fixed, $n(s,a)$ can vary a lot and it is difficult to deal with the concentration bound of $f(\sum_{a} \frac{1}{n(s,a)}\sum_{k=1}^{n(s,a)}Y_k^{s,a})$. If we bound each $\frac{1}{n(s,a)}\sum_{k=1}^{n(s,a)}Y_k^{s,a}$ for $a\in\AM$ separately, the statistical error would be amplified by a factor of $|\AM|$ when $f$ is a linear function. When $f$ is a non-linear function, the situation becomes much more complicated. Instead, we do not need to worry about this problem under the $(s,a)$-rectangular assumption, as we only need to deal with $f(\frac{1}{n_{s,a}}\sum_{k=1}^{n(s,a)}Y_k^{s,a})$ for each $(s,a)$ pair.
\end{rem}

\begin{thm}
    \label{thm: finite-s-off}
    Under the offline dataset, $s$-rectangular assumption, $\nu_{\min}>0$, the following results hold:
    \begin{enumerate}[label=(\alph*), ref=\thethm (\alph*)]
        \item \label{thm: l1-union-off-s} Setting $f(t)=|t-1|$ in Example~\ref{eg: f-set} ($L_1$ balls), with probability $1-\delta$:
        \begin{align*}
            \max_\pi V_r^\pi(\mu)-V_r^{\widehat{\pi}}(\mu)\le\frac{2\gamma(2+\rho)\sqrt{|\SM||\AM|}}{\rho(1-\gamma)^2\sqrt{n\nu_{\min}}}\Bigg{(}4+\sqrt{\log\frac{2|\SM|(1+2\sqrt{2n}(\rho+4))^3}{\delta}}\Bigg{)}.
        \end{align*}
        \item \label{thm: chi2-union-off-s} Setting $f(t)=(t-1)^2$ in Example~\ref{eg: f-set} ($\chi^2$ balls), with probability $1-\delta$:
        \begin{align*}
            \max_\pi V_r^\pi(\mu)-V^{\widehat{\pi}}_r(\mu)\le\frac{2\gamma C^2(\rho)\sqrt{|\SM||\AM|^2}}{(C(\rho)-1)(1-\gamma)^2\sqrt{n\nu_{\min}}}\left(6+\sqrt{2\log\frac{2|\SM|(1+8\sqrt{n}C(\rho))^3}{\delta}}\right),
        \end{align*}
        where $C(\rho)=\sqrt{1+\rho}$.
        \item \label{thm: kl-union-off-s} Setting $f(t)=t\log t$ in Example~\ref{eg: f-set} (KL balls), with probability $1-\delta$:
        \begin{align*}
            \max_\pi V_r^\pi(\mu)-V^{\widehat{\pi}}_r(\mu)\le\frac{4\gamma\sqrt{|\SM||\AM|}}{\rho(1-\gamma)^2\underline{p}\sqrt{n\nu_{\min}}}\left(2+\sqrt{2\log\frac{2|\SM|^2|\AM|\left(1+4\rho\underline{p}\sqrt{n}\right)}{\delta}}\right),
        \end{align*}
        where $\underline{p}=\min_{P^*(s'|s,a)>0}P^*(s'|s,a)$.
    \end{enumerate}
\end{thm}

Similar with the results in the setting of the generative model, the sample complexity in this part is about $\widetilde{O}\left(\frac{|\SM||\AM|}{\nu_{\min}\varepsilon^2\rho^2(1-\gamma)^4}\right)$ for the uncertainty sets  $L_1$ and KL balls, while it is $\widetilde{O}\left(\frac{|\SM||\AM|^2}{\nu_{\min}\varepsilon^2\rho^2(1-\gamma)^4}\right)$ in the $\chi^2$ ball. Compared with the results in the setting of the generative model, a $|\SM||\AM|$ factor is replaced with $1/\nu_{\min}$.

%% file: tex/apdix/finite_off.tex
\begin{proof}[Proof of Lemma~\ref{lem: off-dis}]
	By the Chernoff bound, for fixed $(s,a)\in\SM\times\AM$, we have:
	\begin{align*}
		\PB\left(n(s,a)<\frac{n}{2}\nu(s,a)\right)\le\left(\frac{2}{e}\right)^{\frac{n\nu(s,a)}{2}}\le e^{-\frac{n\nu(s,a)}{8}}\le e^{-\frac{n\nu_{\min}}{8}}.
	\end{align*}
	Thus, by union bound over $\SM\times\AM$, we have:
	\begin{align*}
		\PB\left(\exists(s,a)\in\SM\times\AM, n(s,a)<\frac{n}{2}\nu(s,a)\right)\le|\SM||\AM|e^{-\frac{n\nu_{\min}}{8}}.
	\end{align*}
	Thus, when $n\ge\frac{8}{\nu_{\min}}\log\frac{|\SM||\AM|}{\delta}$, with probability $1-\delta$, for all $(s,a)\in\SM\times\AM$, we have:
	\begin{align*}
		n(s,a)\ge\frac{n}{2}\nu(s,a).
	\end{align*}
\end{proof}

We only give a proof of Theorem~\ref{thm: l1-union-off}. The other results can be obtained by trivial analogies.

\begin{proof}[Proof of Theorem~\ref{thm: l1-union-off}]
	Denote $g(\eta,P)=\sum_{s\in\SM}P(s)(\eta-V(s))_+-\frac{(\eta-\min_s V(s)_+}{2}\rho-\eta$, by Proof of Theorem~\ref{thm: l1-fix}, we have:
	\begin{align*}
		\PB\left(\sup_{\eta\in[0,\frac{2+\delta}{\delta(1-\gamma)}]}\left|g(\eta,\widehat{P}_{s,a})-g(\eta,P_{s,a})\right|>\varepsilon(s,a)\Bigg{|}\sigma(X)\right)\le\frac{\delta}{2|\SM||\AM|},
	\end{align*}
	where $\varepsilon(s,a)=\frac{2+\rho}{\sqrt{2n(s,a)}\rho(1-\gamma)}\left(1+\sqrt{\log\frac{4|\SM||\AM|(1+(4+\rho)\sqrt{2n(s,a)})}{\delta}}\right)$ and $X\sim\nu(s,a)$ are random variables generating $(s,a)$ pairs in $\DM$. $\sigma(X)$ stands for that all randomness by $\nu(s,a)$ is fixed. Denote $\varepsilon=\frac{2+\rho}{\sqrt{n\nu_{\min}}\rho(1-\gamma)}\left(1+\sqrt{\log\frac{4|\SM||\AM|(1+(4+\rho)\sqrt{2n})}{\delta}}\right)$ and $E=\{\forall (s,a), n(s,a)\ge n\nu_{\min}/2\}$. Thus, by union bound over $(s,a)$, we have:
	\begin{align*}
		&\PB\left(\exists(s,a),\sup_{\eta\in[0,\frac{2+\delta}{\delta(1-\gamma)}]}\left|g(\eta,\widehat{P}_{s,a})-g(\eta,P_{s,a})\right|>\varepsilon, E\right)\\
		\le&\sum_{s,a}\PB\left(\sup_{\eta\in[0,\frac{2+\delta}{\delta(1-\gamma)}]}\left|g(\eta,\widehat{P}_{s,a})-g(\eta,P_{s,a})\right|>\varepsilon(s,a)\right)\\
		\le&\frac{\delta}{2}.
	\end{align*}
	Thus, for any fixed $V\in\VM$ and $\pi\in\Pi$, when $n\ge\frac{8}{\nu_{\min}}\log\frac{2|\SM||\AM|}{\delta}$, with probability at most $\delta$, we have events $E$ and:
	\begin{align*}
		\left\|\TM_r^\pi V-\widehat{\TM}_r^\pi V\right\|_\infty\ge\frac{(2+\rho)\gamma}{\sqrt{n\nu_{\min}}\rho(1-\gamma)}\left(1+\sqrt{\log\frac{4|\SM||\AM|(1+(4+\rho)\sqrt{2n})}{\delta}}\right).
	\end{align*}
	At last, we can union bound over $\VM$ and $\Pi$ by similar approach in Theorem~\ref{thm: l1-union} and then combine $\PB(A)\le \PB(A\cap E)+\PB(E^c)$ for any event $A$. Letting $n\ge\frac{8}{\nu_{\min}}\log\frac{2|\SM||\AM|}{\delta}$, with probability $1-\delta$, we have:
	\begin{align*}
		\max_{\pi}V_r^\pi(\mu)-V_r^{\widehat{\pi}}(\mu)\le\frac{2(2+\rho)\gamma\sqrt{|\SM|}}{\sqrt{n\nu_{\min}}\rho(1-\gamma)^2}\left(2+\sqrt{\log\frac{8|\SM||\AM|^2(1+2(2+\rho)\sqrt{2n})^2}{\delta(2+\rho)}}\right).
	\end{align*}
\end{proof}

\begin{proof}[Proof of Theorem~\ref{thm: chi2-union-off},~\ref{thm: kl-union-off},~\ref{thm: l1-union-off-s},~\ref{thm: chi2-union-off-s}, and~\ref{thm: kl-union-off-s}]
	Similar with Proof of Theorem~\ref{thm: l1-union-off}.
\end{proof}

%% file: tex/apdix/asymp_proof.tex
\subsection{Proof of Theorem~\ref{thm: l1_sa_normal}}
Firstly, we prove that $\widehat{V}_r^\pi\toas V_r^\pi$. 
\begin{lem}
    \label{lem: l1_as}
  For a fixed policy $\pi$, we have $\widehat{V}_r^\pi\toas V_r^\pi$. Furthermore, if we assume $V_r^\pi(s_1)< \cdots <V_r^\pi(s_{|\SM|})$, with probability 1, there exists a number $N$ such that, for every $n>N$, we have $\widehat{V}_r^\pi(s_1)< \cdots <\widehat{V}_r^\pi(s_{|\SM|})$.
\end{lem}
\begin{proof}[Proof of Lemma~\ref{lem: l1_as}]
    For any $\varepsilon>0$, we denote event $A_n(\varepsilon)=\{\|\widehat{V}_r^\pi-V_r^\pi\|_\infty\ge\varepsilon\}$. Thus, we have:
    \begin{align*}
        \PB\left(A_n(\varepsilon)\right)\le\PB\left(\sup_{V\in\VM}\left\|\widehat{\TM}_r^\pi V - \TM_r^\pi V\right\|_\infty\ge(1-\gamma)\varepsilon\right).
    \end{align*}
    By the proof of Theorem~\ref{thm: l1-union}, we know that the probability of RHS is controlled by $O\left(n^\alpha\exp(-n)\right)$, where $\alpha$ is a positive constant, while we ignore other factors (which are fixed). Thus, we have
    \begin{align*}
        \sum_{n=1}^\infty\PB(A_n(\varepsilon))<\infty.
    \end{align*}
    By the Borel-Cantelli Lemma, we know that $\PB(A_n(\varepsilon), i.o.)=0$ for any $\varepsilon>0$. Besides, we denote $\delta(\pi)=\min_{i\not=j}|V_r^\pi(s_i)-V_r^\pi(s_j)|$. Then we have $\PB(A_n(\delta(\pi)/2), i.o.)=0$. Thus, our results are concluded.
\end{proof}
\begin{rem}
    In fact, for each of $\chi^2$ and KL uncertainty sets, we always have $\widehat{V}_r^\pi\toas V_r^\pi$ by Theorems~\ref{thm: chi2-uni} and~\ref{thm: kl-uni}, which tell us the tail bound of event $A_n(\varepsilon)$ is controlled well.
\end{rem}

By Lemma~\ref{lem: l1_as}, we can always assume $V_r^\pi$ and $\widehat{V}_r^{\pi}$ are ordered the same when we investigating its asymptotic behavior. Besides, we denote:
\begin{align*}
    g(\eta,P,V)=-\sum_{s\in\SM}P(s)\left(\eta-V(s)\right)_{+}-\frac{\rho}{2}\left(\eta-\min_s V(s)\right)_{+}+\eta
\end{align*}
for $P\in\Delta(\SM)$ and $V\in\VM$. Denote $\eta^*(P, V)=\argmax_\eta g(\eta, P, V)$. There is an explicit expression of $\eta^*(P, V)$ (Lemma~\ref{lem: l1_sol}). Besides, we also denote $f(P, V)=\max_{\eta\ge0}g(\eta, P, V)$. By Danskin's Theorem (Lemma~\ref{lem: danskin}), we can prove that $f(\widehat{P}, V)$ weakly converges to $f(P, V)$ with rate $\sqrt{n}$ (Lemma~\ref{lem: normal_l1_sa}), where $\widehat{P}$ is an empirical unbiased estimation of $P$ with number of samples $n$ (Eqn.~\eqref{eq: gen}).

\begin{lem}
    \label{lem: l1_sol}
    Fix $P\in\Delta(\SM)$ and $V\in\VM$, and denote $\eta^*(P, V)=\argmax_\eta g(\eta, P, V)$. We have:
    \begin{align*}
        \eta^*(P, V)=\min\left\{t\middle|\sum_{s:V(s)\leq t} P(s)> 1-\frac{\rho}{2}\right\}.
    \end{align*}
    Besides, if we order $V$ as $V(s_1)< \cdots <V(s_{|\SM|})$ and denote $K(P):=\min\{l\in\ZB_+|\sum_{k\le l}P(s_k)>1-\rho/2\}$, we have $V(s_{K(P)})=\eta^*(P, V)$.
\end{lem}
\begin{proof}[Proof of Lemma~\ref{lem: l1_sol}]
    Firstly, we order $V$ as $V(s_1)< \cdots <V(s_{|\SM|})$. If there exists $m$, s.t. $\sum_{i\leq m-1}P(s_i)<1-\rho/2<\sum_{i\leq m}P(s_i)$, we conclude that $\eta^*(P, V)=V(s_m)$. In fact, for $\eta\in(V(s_{m-1}), V(s_m))$, we have:
    \begin{align*}
        g(\eta,P,V)-g(\eta^*,P,V)
        &=-\sum_{i\leq m-1}P(s_i)(\eta-V(s_i))+\sum_{i\leq m-1}P(s_i)(\eta^*-V(s_i))+\left(1-\frac{\rho}{2}\right)(\eta-\eta_*)\\
        &= - \left[\sum_{i\leq m-1}P(s_i)-\left(1-\frac{\rho}{2}\right)\right](\eta-\eta_*)\le 0.
    \end{align*}
    Besides, for $\eta\in(V(s_m), V(s_{m+1}))$, we have:
    \begin{align*}
        g(\eta,P,V)-g(\eta^*,P,V)
        &=-\sum_{i\leq m}P(s_i)(\eta-V(s_i))+\sum_{i\leq m}P(s_i)(\eta^*-V(s_i))+\left(1-\frac{\rho}{2}\right)(\eta-\eta_*)\\
        &= - \left[\sum_{i\leq m}P(s_i)-\left(1-\frac{\rho}{2}\right)\right](\eta-\eta_*)\ge 0.
    \end{align*}
    Thus, $\eta^*(P, V)$ is a local maximizer of $g(\eta, P, V)$. By convexity of $g$, $\eta^*(P, V)$ is a global maximizer of $g(\eta, P, V)$. If there exists $m$, s.t. $\sum_{i\le m-1}P(s_i)=1-\rho/2$, we have $\eta^*(P, V)=V(s_{m})$.
\end{proof}

\begin{lem}[Danskin's Theorem]
    \label{lem: danskin}
    Suppose $\phi(x, z) \colon \RB^n\times Z\to\RB$ is a continuous function of two arguments,  where $Z\subset\RB^m$ is a compact set. Further assume $\phi(x,z)$ is convex in $x$ for every $z\in Z$. Define $Z_0(x)=\left\{\bar{z}: \phi(x, \bar{z})=\max _{z \in Z} \phi(x, z)\right\}$. If $Z_0(x)$ consists of a single element $\bar{z}$, then $f(x):=\max_{z\in Z}\phi(x,z)$ is differentiable at $x$ and $\frac{\partial f}{\partial x}=\frac{\partial \phi(x, \bar{z})}{\partial x}$.
\end{lem}

\begin{lem}
    \label{lem: normal_l1_sa}
    Fix $V\in\VM$, assume $V(s_1)< \cdots <V(s_{|\SM|})$ and $\sum_{i\leq m}P(s_i)\not=1-\rho/2$ for $m\in\{1, \ldots, |\SM|\}$, and denote $\eta^*(P,V)=\argmax_\eta g(\eta,P,V)>0$. Then we have 
    \begin{align*}
        \sqrt{n}\left(f(\widehat{P},V)-f(P,V)\right)\tod\NM\left(0,\sigma^2(P,V)\right),
    \end{align*}
    as long as $\sigma^2(P,V)=b^T\Sigma b\not=0$. Here $\Sigma\in\RB^{|\SM|\times|\SM|}$ with the $(i,j)$th element  defined as 
    \begin{align*}
        \Sigma(i,j)=-P(s_i)\cdot P(s_j)+P(s_i)\mathbf{1}\{i=j\},
    \end{align*}
    and $b=f^\prime(P)\in \RB^{|\SM|}$ with the $i$-th element defined as 
    \begin{align*}
        b(i)=-\left(\eta_*(P,V)-V(s_i)\right)_+.  
    \end{align*}
\end{lem}

\begin{proof}[Proof of Lemma~\ref{lem: normal_l1_sa}]
    Firstly, by CLT, we have:
    \begin{align*}
        \sqrt{n}\left(\widehat{P}-P\right)\tod\NM\left(0, \Sigma\right).
    \end{align*}
    Besides, we note that $\eta\in[0,\frac{1}{\rho(1-\gamma)}]$ and $g(\eta, P, V)$ is continuous and convex in $P$. By Lemma~\ref{lem: normal_l1_sa}, we know:
    \begin{align*}
        \frac{\partial f(P, V)}{\partial P(s_i)} = -\left(\eta^*(P,V)-V(s_i)\right)_+.
    \end{align*}
    Thus, by the Delta method, our result is concluded.
\end{proof}

Thus, by the fact that $V_r^\pi = \TM_r^\pi V_r^\pi$ and Lemma~\ref{lem: normal_l1_sa}, we conclude that $\widehat{\TM}_r^\pi V_r^\pi$ weakly converges to $V_r^\pi$ with rate $\sqrt{n}$ (Corollary~\ref{cor: normal_l1_sa}).
\begin{cor}
    \label{cor: normal_l1_sa}
    Under the $(s,a)$-rectangular assumption and the setting $f(t)=|t-1|$ in Example~\ref{eg: f-set}, for any policy $\pi$, we have:
    \begin{align*}
        \sqrt{n}\left(\widehat{\TM}_r^\pi - I\right)V_r^\pi\tod \NM\left(0, \Lambda\right),
    \end{align*}
    where $\Lambda=\diag\{\sigma_1^2, \ldots, \sigma^2_{|\SM|}\}$,
    as long as $\forall s\in\SM,a\in\AM,m\in\{1, \ldots, |\SM|\},\sum_{i\leq m}P^*(s_i|s,a)\not=1-\rho/2$ and $\forall s\in\SM$ 
    \begin{align*}
        \sigma_s^2=\gamma^2\sum_{a\in\AM}\pi(a|s)^2\sigma^2(P^*(\cdot|s,a),V_r^\pi)\not=0, 
    \end{align*}
    where $\{s_i\}$ are defined as to make $V_r^\pi(s_1)< \cdots <V_r^\pi(s_{|\SM|})$.
\end{cor}

We denote $\widehat{\Phi}^\pi_n=\widehat{\TM}_r^\pi - I$ and $\Phi^\pi = \TM_r^\pi - I$. By the fact that $\widehat{\Phi}^\pi_n (\widehat{V}_r^\pi)=0$ and $\Phi^\pi (V_r^\pi)=0$ (fixed points), we have:
\begin{align*}
    \widehat{\Phi}^\pi_n V_r^\pi=\widehat{\Phi}^\pi_n V_r^\pi - \widehat{\Phi}^\pi_n \widehat{V}_r^\pi.
\end{align*}

Thus, by Lemma~\ref{lem: l1_express} and assuming that $\widehat{V}_r^\pi$, $V_r^\pi$ are ordered the same, we have:
\begin{align*}
    \widehat{\Phi}^\pi_n V_r^\pi=\M_n^\pi\cdot(\widehat{V}_r^\pi-V_r^\pi),
\end{align*}
where $M_n^\pi$ is similar with $M^\pi$ in Lemma~\ref{lem: l1_express} with $P$ being replaced with $\widehat{P}$.

\begin{lem}
    \label{lem: l1_express}
   For a fixed policy $\pi$,  and for any $V_1, V_2\in\VM$ as long as they both satisfy $V_i(s_1)< \cdots <V_i(s_{|\SM|})$ for $i=1,2$, there exists $\M^\pi\in\RB^{|\SM|}\times\RB^{|\SM|}$, such that:
    \begin{align*}
        \Phi^\pi (V_1) - \Phi^\pi (V_2) = -\M^\pi\cdot(V_1-V_2),
    \end{align*}
    where the $(i,j)$th element of $M^\pi$ is
    \begin{align*}
        \M^\pi(i,j)&=\mathbf{1}\{i=j\}-\gamma\sum_{a\in\AM}\pi(a|s_i)\Bigg[P(s_j|s_i,a)\mathbf{1}\{j<K(P(\cdot|s_i,a))\}\\
        &-\left(\sum_{k< K(P(\cdot|s_i,a))}P(s_k|s_i,a)-(1-\frac{\rho}{2})\right)\mathbf{1}\{j=K(P(\cdot|s_i,a))\}+\frac{\rho}{2}\mathbf{1}\{j=1\}\Bigg].
    \end{align*}
    and $K(P):=\min\{l\in\ZB_+|\sum_{k\le l}P(s_k)>1-\rho/2\}$.
\end{lem}
\begin{proof}[Proof of Lemma~\ref{lem: l1_express}]
    By the expression of $\TM_r^\pi$, we know that:
    \begin{align*}
        \TM_r^\pi V (s)= R^\pi(s)+\gamma\sum_{a\in\AM}\pi(a|s)f(P_{s,a}, V).
    \end{align*}
    By Lemma~\ref{lem: l1_sol} and Lemma~\ref{lem: l1}, we note that:
    \begin{align*}
        &f(P_{s,a}, V)=g(\eta^*(P_{s,a}, V), P_{s,a}, V)\\
        &=-\sum_{s'\in\SM}P(s'|s,a)\left(\eta^*(P_{s,a}, V)-V(s')\right)_+-\frac{\rho}{2}\left(\eta^*(P_{s,a}, V)-\min_s V(s)\right)_+ + \eta^*(P_{s,a}, V)\\
        &=-\sum_{k=1}^{|\SM|}P(s_k|s,a)\left(V(s_{K(P)})-V(s_k)\right)_+-\frac{\rho}{2}\left(V(s_{K(P)})-\min_s V(s)\right) +V(s_{K(P)})\\
        &=\sum_{k=1}^{|\SM|}P(s_k|s,a)V(s_k)\mathbf{1}\{k< K(P)\}-\left(\sum_{k<K(P)}P(s_k|s,a)-(1-\frac{\rho}{2})\right)V(s_{K(P)})+\frac{\rho}{2}V(s_1).
    \end{align*}
    Thus, by a simple calculation and $V_1, V_2$ sharing the same order, we have:
    \begin{align*}
        \TM_r^\pi V_1 - \TM_r^\pi V_2 = - \M^\pi\cdot(V_1-V_2) + (V_1-V_2).
    \end{align*}
    Then, our result is concluded.
\end{proof}

Thus, as long as $\M_n^\pi$ converges a.s. (Lemma~\ref{lem: M_converge}), we conclude that $\widehat{V}_r^\pi$ weakly converges to $V_r^\pi$ with rate $\sqrt{n}$. If $\sigma_{\min}(\M^\pi)>0$, combining with Corollary~\ref{cor: normal_l1_sa}, by Slusky's theorem, we have:
\begin{align*}
    \sqrt{n}\left(\widehat{V}_r^\pi - V_r^\pi\right)\tod \NM\left(0, (\M^\pi)^{-1}\Lambda(\M^\pi)^{-\top}\right).
\end{align*}

\begin{lem}
    \label{lem: M_converge}
    For any fixed $\pi$, assuming $\sum_{i\le m}P^*(s_i|s,a)\not=1-\rho/2$ for any $m=1,..,|\SM|$ and $(s,a)\in\SM\times\AM$, and $\widehat{V}_r^\pi$, $V_r^\pi$ are ordered the same, we have $\M_n^\pi\toas\M^\pi$.
\end{lem}
\begin{proof}
    Firstly, by strong law of large numbers (SLLN), we know that for any $(s,a,s')\in\SM\times\AM\times\AM$, 
    \begin{align*}
        \widehat{P}(s'|s,a)\toas P^*(s'|s,a).
    \end{align*}
    Next, we turn to prove that $K(\widehat{P}_{s,a})\toas K(P^*_{s,a})$. In fact, we have:
    \begin{align*}
        \PB\left(K(\widehat{P}_{s,a})\not= K(P^*_{s,a})\right)&\le\PB\left(K(\widehat{P}_{s,a})>K(P^*_{s,a})\right)+\PB\left(K(\widehat{P}_{s,a})<K(P^*_{s,a})\right)\\
        &\le \PB\left(\sum_{i\le K(P_{s,a})} \widehat{P}(s_i|s,a) \le 1-\frac{\rho}{2}\right) + \PB\left(\sum_{i\le K(P_{s,a})-1} \widehat{P}(s_i|s,a) > 1-\frac{\rho}{2}\right).
    \end{align*}
    Note that $\sum_{i\le m}P^*(s_i|s,a)\not= 1-\rho/2$ for any $m=1, \ldots, |\SM|$, there exists $\delta_{s,a}>0$ such that $\sum_{i\le K(P^*_{s,a})}P^*(s_i|s,a)>1-\rho/2 +\delta_{s,a}$ and $\sum_{i\le K(P^*_{s,a})-1}P^*(s_i|s,a)<1-\rho/2 -\delta_{s,a}$. Thus, we have:
    \begin{align*}
        \PB\left(K(\widehat{P}_{s,a})\not= K(P^*_{s,a})\right)&\le\PB\left(\sum_{i\le K(P^*_{s,a})}\left(\widehat{P}(s_i|s,a)-P(s_i|s,a)\right)\le-\delta_{s,a}\right)\\
        &+\PB\left(\sum_{i\le K(P^*_{s,a})-1}\left(\widehat{P}(s_i|s,a)-P(s_i|s,a)\right)>\delta_{s,a}\right)\\
        &\le 2\PB\left(\left\|\widehat{P}_{s,a}-P^*_{s,a}\right\|_1>\delta_{s,a}\right)\\
        &\le 4|\SM|\exp\left(-\frac{2n\delta^2_{s,a}}{|\SM|^2}\right),
    \end{align*}
    where the last inequality holds by Hoeffing's inequality. By the Borel-Cantelli Lemma, we observe that $\sum_{n=1}^\infty \PB(K(\widehat{P}_{s,a})\not= K(P^*_{s,a}))<\infty$, which leads to $\PB(K(\widehat{P}_{s,a})\not= K(P^*_{s,a}), i.o.)=0$. The same argument can also be applied to $\mathbf{1}\{j<K(\widehat{P}_{s,a})\}$ and $\mathbf{1}\{j=K(\widehat{P}_{s,a})\}$. Thus, $\M_n^\pi\toas \M^\pi$.
\end{proof}

\subsection{Proof of Theorem~\ref{thm: chi2_sa_normal}}
Just like Lemma~\ref{lem: l1_as}, when $\chi^2$ uncertainty set is applied, we still have $\widehat{V}_r^\pi\toas V_r^\pi$. Besides, we denote:
\begin{align*}
    g(\eta,P,V)=-C(\rho) \sqrt{\sum_{s \in \mathcal{S}} P(s)(\eta-V(s))_{+}^{2}}+\eta
\end{align*}
for any $P\in\Delta(\SM)$ and $V\in \VM$, where $C(\rho)=\sqrt{1+\rho}>0$. And we also denote $\eta^*(P, V)=\argmax_{\eta} g(\eta, P, V)$ and $f(P, V)=\max_{\eta}g(\eta, P, V)$. By Danskin's Theorem (Lemma~\ref{lem: danskin}), we can also prove that $f(\widehat{P}, V)$ weakly converges to $f(P, V)$ with rate $\sqrt{n}$ (Lemma~\ref{lem: normal_chi2_sa}).

\begin{lem}
    \label{lem: normal_chi2_sa}
    Fix $V\in\VM$, assume $V(s_1)<\cdots <V(s_{|\SM|})$, and denote $\eta^*(P,V)=\argmax_\eta g(\eta,P,V)>0$. Then we have 
    \begin{align*}
        \sqrt{n}\left(f(\widehat{P},V)-f(P,V)\right)\tod\NM\left(0,\sigma^2(P,V)\right),
    \end{align*}
    as long as $\sigma^2(P,V)=b^T\Sigma b\not=0$, where $\Sigma\in\RB^{|\SM|\times|\SM|}$ with the $(i,j)$th element defined as 
    \begin{align*}
        \Sigma(i,j)=-P(s_i)\cdot P(s_j)+P(s_i)\mathbf{1}\{i=j\},
    \end{align*}
    and $b=f^\prime(P)\in \RB^{|\SM|}$ with the $i$-th element defined as 
    \begin{align*}
        b(i)=-C(\rho)\frac{(\eta^*(P,V)-V(s_i))^2_+}{2\sqrt{\sum_{s\in\SM} P(s)(\eta^*(P,V)-V(s))^2_+}}.
    \end{align*}
\end{lem}
\begin{proof}[Proof of Lemma~\ref{lem: normal_chi2_sa}]
    By CLT, we have:
    \begin{align*}
        \sqrt{n}\left(\widehat{P}-P\right)\tod \NM\left(0, \Sigma\right).
    \end{align*}
    Besides, we note that $\eta\in[0, \frac{C(\rho)}{(C(\rho)-1)(1-\gamma)}]$ and $g(\eta, P, V)$ is continuous and convex in P. By Danskin's Theorem (Lemma~\ref{lem: danskin}), we know:
    \begin{align*}
        \frac{\partial f(P, V)}{\partial P(s_i)} = -C(\rho)\frac{(\eta^*(P,V)-V(s_i))^2_+}{2\sqrt{\sum_{s\in\SM} P(s)(\eta^*(P,V)-V(s))^2_+}}.
    \end{align*}
    Thus, by Delta method, our result is concluded.
\end{proof}

By the fact that $V_r^\pi=\TM_r^\pi V_r^\pi$ and Lemma~\ref{lem: normal_chi2_sa}, we conclude that $\widehat{\TM}_r^\pi V_r^\pi$ weakly converges to $V_r^\pi$ with rate $\sqrt{n}$ (Corollary~\ref{cor: normal_chi2_sa}).
\begin{cor}
    \label{cor: normal_chi2_sa}
    Under the $(s,a)$-rectangular assumption and the setting $f(t)=(t-1)^2$ in Example~\ref{eg: f-set}, for any policy $\pi$, we have:
    \begin{align*}
        \sqrt{n}\left(\widehat{\TM}_r^\pi - I\right)V_r^\pi\tod\NM\left(0, \Lambda\right),
    \end{align*}
    where $\Lambda=\diag\{\sigma_1^2, \cdots, \sigma^2_{|\SM|}\}$, and $\forall s\in\SM$,
    \begin{align*}
        \sigma_s^2=\gamma^2\sum_{a\in\AM}\pi(a|s)^2\sigma^2(P^*(\cdot|s,a),V_r^\pi)\not=0.
    \end{align*}
\end{cor}

Different with the case of $L_1$ uncertainty set, by Lemma~\ref{lem: chi2_diff}, we have:
\begin{align*}
    \widehat{\Phi}_n^\pi V_r^\pi=\widehat{\Phi}_n^\pi V_r^\pi-\widehat{\Phi}_n^\pi \widehat{V}_r^\pi=\M_n^\pi\cdot(V_r^\pi-\widehat{V}_r^\pi)+o_P(V_r^\pi-\widehat{V}_r^\pi).
\end{align*}
where $\M_n^\pi = \widehat{\Phi}'_n(V_r^\pi)$. As we already know that $\widehat{V}_r^\pi\toas V_r^\pi$, next we need to prove that $\widehat{V}_r^\pi$ is $\sqrt{n}$-consistent (Lemma~\ref{lem: consistent}). Thus, as long as $\M_n^\pi$ is consistent (Lemma~\ref{lem: chi2_M_converge}), we prove that:
\begin{align*}
    \sqrt{n}\left(\widehat{V}_r^\pi-V_r^\pi\right)\tod\NM\left(0, (\M^\pi)^{-1}\Lambda(\M^\pi)^{-\top}\right).
\end{align*}

\begin{lem}
    \label{lem: chi2_diff}
    Fix a policy $\pi$, $\Phi^\pi$ is continuous differentiable w.r.t. $V\in\VM$.
\end{lem}
\begin{proof}[Proof of Lemma~\ref{lem: chi2_diff}]
    W.L.G., we assume $V(s_1)< \cdots <V(s_{|\SM|})$. Firstly, we know that $\eta^*(P, V)\in[V(s_1), \frac{C(\rho)}{(C(\rho)-1)(1-\gamma)}]$. If $\eta^*(P, V)> V(s_1)$, $g(\eta, P,V)$ has a continuous partial derivative w.r.t. $\eta$ at $\eta^*(P, V)$:
    \begin{align*}
        \frac{\partial f(P,V)}{\partial V}
        &=\frac{\partial g(\eta^*(P,V),P,V)}{\partial V}+\frac{\partial g(\eta^*(P,V),P,V)}{\partial \eta}\frac{\partial \eta^*(P,V)}{\partial V}\\
        &=\frac{\partial g(\eta^*(P,V),P,V)}{\partial V},
    \end{align*}
    where $\partial \eta^*(P,V)/\partial V$ exists by implicit function theorem. Thus, for each $s'\in\SM$, we have:
    \begin{align*}
        \frac{\partial f(P,V)}{\partial V(s')}=C(\rho)\frac{P(s')(\eta^*(P,V)-V(s'))_+}{\sqrt{\sum_{s\in\SM}P(s)(\eta^*(P,V)-V(s))_+^2}}.
    \end{align*}
    If $\eta^*(P,V)=V(s_1)$, we have:
    \begin{align*}
        \frac{\partial f(P,V)}{\partial V(s_1)}=1.
    \end{align*}
    Next, we prove that $\eta^*(P,V)$ is continuous w.r.t.\ $V$. For every convergent sequence $V_n\to V$, where $V(s_1)<\cdots <V(s_{|\SM|})$, we can choose a large enough $n$ s.t. $V_n(s_1)<\cdots <V_n(s_{|\SM|})$. Noting that for every two $V_1, V_2\in\VM$, we have:
    \begin{align*}
        \left|g(\eta, P, V_1)-g(\eta, P, V_2)\right|\le\sqrt{\sum_{s\in\SM}P(s)(V_1(s)-V_2(s))^2}=\|V_1-V_2\|_{2, P}.
    \end{align*}
    Thus, by Lemma~\ref{lem: domain_converge} we have:
    \begin{align*}
        \sup_{\eta\in[V(s_1), \frac{C(\rho)}{(C(\rho)-1)(1-\gamma)}]}\left|g(\eta, P, V_n)-g(\eta, P, V)\right|&\le\sqrt{\sum_{s\in\SM}P(s)(V_n(s)-V(s))^2}\\
        &=\|V_n-V\|_{2, P}\to 0,
    \end{align*}
    which implies that $\eta^*(P, V_n)\to \eta^*(P, V)$. Thus, $\eta^*(P, V)$ is continuous w.r.t. $V$ and our result is completed.
\end{proof}

\begin{lem}
    \label{lem: domain_converge}
    Suppose we have a sequence of concave function $\{f_n(x)\}\colon \XM\to\RB$, and each $f_n$ admits a unique maximizer $x_n\in\XM$, where $\XM$ is a closed convex set. Then if $\exists f_*\colon \XM\to\RB$ such that $f_{*}$ is concave and has a unique maximizer $x_*$ and $\sup_{x\in\XM} |f_n(x)-f_*(x)|\to 0$, we have $x_n\to x_*$.
\end{lem}
\begin{proof}[Proof of Lemma~\ref{lem: domain_converge}]
    Denote $a=\limsup x_n$ and $b=\liminf x_n$. First we will show $b\geq x_*$ by contradiction. Assume $b<x_*$, then by concavity $f_*(\frac{b+x_*}{2})\geq \frac{f_*(x_*)+f_*(b)}{2}$ and we have $f_*(b)< f_*(\frac{b+x_*}{2})< f_*(x_*)$. Then for $n$ large enough, we have $f_n(b)< f_n(\frac{b+x_*}{2})< f_n(x_*)$, which suggests $x_n>\frac{b+x_*}{2}$. So $b\geq\frac{b+x_*}{2}$, which is a contradiction. Similarly we can show $a\leq x_*$, and the proof is completed.
\end{proof}

\begin{lem}
    \label{lem: consistent}
    Suppose $Z_n = X_n + Y_n$, where $Z_n=O_p(1)$ and $Y_n = o_P(X_n)$. Then,   we have $X_n=O_P(1)$ and $Y_n=o_P(1)$.
\end{lem}
\begin{proof}[Proof of Lemma~\ref{lem: consistent}]
    By the fact that $Y_n=o_P(X_n)$, we have $Z_n=X_n(1+o_P(1))$. As we know $(1+o_P(1))^{-1}=O_P(1)$, we have $X_n=O_P(Z_n)$. Combining with $Z_n=O_p(1)$, we obtain that $X_n=O_P(1)$ and $Y_n=o_P(1)$.
\end{proof}

\begin{lem}
    \label{lem: chi2_M_converge}
    For a fixed policy $\pi$, we have $\M_n^\pi\toas \M^\pi$.
\end{lem}

\begin{proof}[Proof of Lemma~\ref{lem: chi2_M_converge}]
    Noting that $\M_n^\pi=\widehat{\Phi}_n'(V_r^\pi)$ for $s,s'\in\SM$, we have:
    \begin{align*}
        \M_n^\pi(s,s^\prime)=\mathbf{1}\{s=s^\prime\}-\gamma\sum_a \pi(a|s)C(\rho)\frac{\widehat{P}(s'|s,a)(\eta^*(\widehat P(\cdot|s,a),V_r^\pi)-V_r^\pi(s^\prime))_+}{\sqrt{\sum_{\tilde s\in\SM}\widehat P(\tilde s|s,a)(\eta^*(\widehat P(\cdot|s,a),V_r^\pi)-V_r^\pi(\tilde s))_+^2}}.
    \end{align*}
    By SLLN, we know that for any $(s,a,s')\in\SM\times\AM\times\SM$, 
    \begin{align*}
        \widehat{P}(s'|s,a)\toas P^*(s'|s,a).
    \end{align*}
    Next, we note that for any $P_1, P_2\in\Delta(\SM)$, we have:
    \begin{align*}
        \sup_{\eta\in[V(s_1),\frac{C(\rho)}{(C(\rho)-1)(1-\gamma)}]}\left|g(\eta, P_1, V)-g(\eta, P_2, V)\right|&\le \sqrt{\sum_{s\in\SM}|P_1(s)-P_2(s)|(\eta-V(s))_+^2}\\
        &\le \frac{C(\rho)}{(C(\rho)-1)(1-\gamma)}\sqrt{\|P_1-P_2\|_1}.
    \end{align*}
    Thus, by Lemma~\ref{lem: domain_converge}, for every convergent sequence $P_n\to P$, we have $\eta^*(P_n, V)\to \eta^*(P, V)$, which implies that $\eta^*(P, V)$ is continuous w.r.t.\ $P$. Thus, we conclude that $\eta^*(\widehat{P}(\cdot|s,a), V_r^\pi)\toas \eta^*(P(\cdot|s,a), V_r^\pi)$,  leading to $\M_n^\pi\toas \M^\pi$.
\end{proof}

\subsection{Proof of Theorem~\ref{thm: kl_sa_normal}}
Like Lemma~\ref{lem: l1_as}, when the KL uncertainty set is applied, we have $\widehat{V}_r^\pi\toas V_r^\pi$. We denote:
\begin{align*}
    g(\lambda, P, V)=-\lambda\rho-\lambda\log\sum_{s\in\SM}P(s)\exp\left(-\frac{V(s)}{\lambda}\right)
\end{align*}
for any $P\in\Delta(\SM)$ and $V\in\VM$. Besides, we denote $\lambda^*(P, V)=\argmax_{\lambda\ge0}g(\lambda, P, V)$ and $f(P, V)=\max_{\lambda\ge0}g(\lambda, P, V)$. By Danskin's Theorem (Lemma~\ref{lem: danskin}), we can prove that $f(\widehat{P}, V)$ weakly converges to $f(P, V)$ with rate $\sqrt{n}$ (Lemma~\ref{lem: normal_kl_sa}).
\begin{lem}
    \label{lem: normal_kl_sa}
    Fix $V\in\VM$ and assume $V(s_1)<\cdots <V(s_{|\SM|})$. Then we have:
    \begin{align*}
        \sqrt{n}\left(f(\widehat{P}, V)-f(P, V)\right)\tod \NM(0, \sigma^2(P,V)),
    \end{align*}
    where $\sigma^2(P, V)=b^\top \Sigma b$ and $\Sigma\in\RB^{|\SM|\times|\SM|}$ with the $(i,j)$-th element defined as
    \begin{align*}
        \Sigma(i,j)=-P(s_i)\cdot P(s_j)+P(s_i)\mathbf{1}\{i=j\},
    \end{align*}
    and $b=f'(P)\in\RB^{|\SM|}$ with $i$-th element defined as
    \begin{align*}
        b(i)=-\frac{\lambda^*(P,V)\exp(-\frac{V(s_i)}{\lambda^*(P,V)})}{\sum_{s\in\SM}P(s)\exp(-\frac{V(s)}{\lambda^*(P,V)})}.
    \end{align*}
\end{lem}
\begin{proof}[Proof of Lemma~\ref{lem: normal_kl_sa}]
    Noting that $\lim_{\lambda\to0^+}g(\lambda, P, V)=\min_s V(s)$, we define $g(0, P, V)=\min_s V(s)$. For any fixed $P, V$, $g(\lambda, P, V)$ is concave in $\lambda$, which implies $g(\lambda, P, V)$ admits a unique maximizer $\lambda^*(P, V)\in[0,\frac{1}{\rho(1-\gamma)}]$. We further note that $\forall \lambda$, $g(\lambda,P,V)$ is continuous w.r.t.\ $\lambda,P$, and convex in $P$. Thus By Danskin's theorem (Lemma~\ref{lem: danskin}), we know that $f(P,V)$ is differentiable at $P$ with derivative
    \begin{align*}
        \frac{\partial f(P,V)}{\partial P(s_i)}=-\frac{\lambda^*(P,V)\exp(-\frac{V(s_i)}{\lambda^*(P,V)})}{\sum_{s\in\SM}P(s)\exp(-\frac{V(s)}{\lambda^*(P,V)})}.
    \end{align*}
    By the fact that $\sqrt{n}\left(\widehat{P}-P\right)\tod \NM(0, \Sigma)$ via CLT, our result is obtained by the Delta method.
\end{proof}

Thus, we can conclude that $\widehat{\TM}_r^\pi V_r^\pi$ weakly converges to $V_r^\pi$ with rate $\sqrt{n}$ (Corollary~\ref{cor: normal_kl_sa}).

\begin{cor}
    \label{cor: normal_kl_sa}
    Under the $(s,a)$-rectangular assumption and the setting $f(t)=t\log t$ in Example~\ref{eg: f-set}, for any policy $\pi$, we have:
    \begin{align*}
        \sqrt{n}\left(\widehat{\TM}_r^\pi - I\right)V_r^\pi\tod\NM\left(0, \Lambda\right),
    \end{align*}
    where $\Lambda=\diag\{\sigma_1^2, \cdots, \sigma^2_{|\SM|}\}$, and $\forall s\in\SM$,
    \begin{align*}
        \sigma_s^2=\gamma^2\sum_{a\in\AM}\pi(a|s)^2\sigma^2(P^*(\cdot|s,a),V_r^\pi).
    \end{align*}
\end{cor}

Similar with the case of the $\chi^2$ uncertainty set, by Lemma~\ref{lem: kl_diff}, we have:
\begin{align*}
    \widehat{\Phi}_n^\pi V_r^\pi=\widehat{\Phi}_n^\pi V_r^\pi-\widehat{\Phi}_n^\pi \widehat{V}_r^\pi=\M_n^\pi\cdot(V_r^\pi-\widehat{V}_r^\pi)+o_P(V_r^\pi-\widehat{V}_r^\pi),
\end{align*}
where $\M_n^\pi = \widehat{\Phi}'_n(V_r^\pi)$. As we already know that $\widehat{V}_r^\pi\toas V_r^\pi$, and $\widehat{V}_r^\pi$ is $\sqrt{n}$-consistent (Lemma~\ref{lem: consistent}). Thus, as long as $\M_n^\pi$ is consistent (Lemma~\ref{lem: kl_M_converge}), we obtain that:
\begin{align*}
    \sqrt{n}\left(\widehat{V}_r^\pi-V_r^\pi\right)\tod\NM\left(0, (\M^\pi)^{-1}\Lambda(\M^\pi)^{-\top}\right).
\end{align*}

\begin{lem}
    \label{lem: kl_diff}
    For a fixed policy $\pi$, $\Phi^\pi$ is continuous differentiable w.r.t.\ $V\in\VM$.
\end{lem}
\begin{proof}[Proof of Lemma~\ref{lem: kl_diff}]
    W.L.G., we assume $V(s_1)<\cdots <V(s_{|\SM|})$. Firstly, we know that $\lambda^*(P, V)\in[0, \frac{1}{\rho(1-\gamma)}]$ and $g(\lambda, P,V)$ has a continuous partial derivative w.r.t.\ $\lambda$ at $\lambda^*(P, V)$, which implies:
    \begin{align*}
        \frac{\partial f(P,V)}{\partial V}
        &=\frac{\partial g(\lambda^*(P,V),P,V)}{\partial V}+\frac{\partial g(\lambda^*(P,V),P,V)}{\partial \lambda}\frac{\partial \lambda^*(P,V)}{\partial V}\\
        &=\frac{\partial g(\lambda^*(P,V),P,V)}{\partial V},
    \end{align*}
    where $\partial \lambda^*(P,V)/\partial V$ exists by the implicit function theorem. Thus, for each $s'\in\SM$, we have:
    \begin{align*}
        \frac{\partial f(P,V)}{\partial V(s^\prime)}=\frac{P(s^\prime)\exp(-\frac{V(s^\prime)}{\lambda^*(P,V)})}{\sum_{s\in\SM}P(s)\exp(-\frac{V(s)}{\lambda^*(P,V)})}.
    \end{align*}
    
    Next, we prove that $\lambda^*(P,V)$ is continuous w.r.t.\ $V$. Note that
    \begin{align*}
        \frac{\partial g(\lambda,P,V)}{\partial V(s)}=\frac{P(s)\exp(-\frac{V(s)}{\lambda})}{\sum_{s\in\SM} P(s)\exp(-\frac{V(s)}{\lambda})},
    \end{align*}
    which implies $\|\nabla g(\lambda, P, V)\|_1=1$. Thus, the Lipschitz coefficient of $g(\lambda, P, V)$ w.r.t.\ $V$ in $\RB^{|\SM|}$ is bounded by $1$. By Lemma~\ref{lem: domain_converge}, we know that $\lambda^*(P, V)$ is continuous w.r.t.\ V and our result is completed.
\end{proof}

\begin{lem}
    \label{lem: kl_M_converge}
    For a fixed policy $\pi$, we have $\M_n^\pi\toas \M^\pi$.
\end{lem}
\begin{proof}[Proof of Lemma~\ref{lem: kl_M_converge}]
    Noting that $\M_n^\pi=\widehat{\Phi}_n'(V_r^\pi)$ for $s,s'\in\SM$, we have:
    \begin{align*}
        \M_n^\pi(s,s^\prime)=\mathbf{1}\{s=s^\prime\}-\gamma\sum_a \pi(a|s)\frac{\widehat P(s^\prime|s,a)\exp(-\frac{V_r^\pi(s^\prime)}{\lambda^*(\widehat P(\cdot|s,a),V^\pi_r)})}{\sum_{\tilde s\in\SM}\widehat P(\tilde s|s,a)\exp(-\frac{V_r^\pi(\tilde s)}{\lambda^*(\widehat P(\cdot|s,a),V_r^\pi)})}.
    \end{align*}
    By SLLN, we know that for any $(s,a,s')\in\SM\times\AM\times\SM$, 
    \begin{align*}
        \widehat{P}(s'|s,a)\toas P^*(s'|s,a).
    \end{align*}
    Next, we note that when $n$ is large enough s.t. $\max_{s, P(s)\not=0}\left|\frac{\widehat{P}(s)}{P(s)}-1\right|\le 1/2$, we have (details can be referred in Theorem~\ref{thm: kl-fix}):
    \begin{align*}
        \sup_{\lambda\in[0,\frac{1}{\rho(1-\gamma)}]}\left|g(\lambda, \widehat{P}, V)-g(\lambda, P, V)\right|\le \frac{2}{\rho(1-\gamma)}\max_{s,P(s)\not=0}\left|\frac{\widehat{P}(s)}{P(s)}-1\right|.
    \end{align*}
    Thus, by Lemma~\ref{lem: domain_converge}, as $\widehat{P}\toas P$, we have $\lambda^*(\widehat{P}, V)\to \lambda^*(P, V)$. Thus, we conclude that $\M_n^\pi\toas \M^\pi$.
\end{proof}

\subsection{Proof of Corollary~\ref{cor: sup_asymp}}
We let $\widehat{\pi}^*\in\argmax \widehat{V}_r^\pi(\mu)$ and $\pi^*\in\argmax V_r^\pi$. By the $(s,a)$-rectangular assumption, we know that $\widehat{\pi}^*$ and $\pi^*$ both can be deterministic policies (taking values on the vertices of simplex). Next, we prove that $\widehat{\pi}^*$ is exactly $\pi^*$ in probability when $n$ is sufficiently large.

\begin{lem}
    \label{lem: consistent_pi}
    We assume $\min_{s,a_1\not= a_2}\left|Q_r^*(s,a_1)-Q_r^*(s,a_2)\right|>0$. Then  we have $\lim_{n\to\infty}\PB\left(\widehat{\pi}^*\not=\pi^*\right)=0$.
\end{lem}

\begin{proof}[Proof of Lemma~\ref{lem: consistent_pi}]
    By Lemma~\ref{lem: l1_as}, we have already shown $\widehat{V}_r^\pi\toas V_r^\pi$ for any fixed $\pi\in\Pi$. We can also extend the result to a stronger version that $\sup_{\pi}\left\|\widehat{V}_r^\pi - V_r^\pi\right\|_\infty\toas 0$ by a similar proof. By the fact that $\left\|\sup_\pi \widehat{V}_r^\pi - \sup_\pi V_r^\pi\right\|_\infty\le\sup_\pi\left\|\widehat{V}_r^\pi - V_r^\pi\right\|_\infty$, we know that $\widehat{V}_r^*\toas V_r^*$. Thus, if we could also show $\widehat{Q}_r^*\toas Q_r^*$, it's guaranteed that $\widehat{\pi}^*\overset{a.s.}{=}\pi^*$ by the fact that $\widehat{\pi}^*\in\argmax \widehat{Q}_r^*$ and $\min_{s,a_1\not= a_2}|Q_r^*(s,a_1)-Q_r^*(s,a_2)|>0$.

    Indeed, by definition of robust $Q$-value function, we have:
    \begin{align*}
        \left|\widehat{Q}_r^*(s,a) - Q_r^*(s,a)\right|&=\left|\inf_{P\in\widehat{\PM}_{s,a}(\rho)}P^\top \widehat{V}_r^*-\inf_{P\in\PM_{s,a}(\rho)}P^\top V_r^*\right|\notag\\
        &\le \sup_{P\in\widehat{\PM}_{s,a}(\rho)}\left|P^\top\left(\widehat{V}_r^*-V_r^*\right)\right|+\left|\inf_{P\in\widehat{\PM}_{s,a}(\rho)}P^\top V_r^* - \inf_{P\in\PM_{s,a}(\rho)}P^\top V_r^*\right|,
    \end{align*}
    where the first term can be controlled by $\widehat{V}_r^*\toas V_r^*$. And the second term can be controlled by $\widehat{P}_{s,a}\toas P^*_{s,a}$ (refer to Lemmas~\ref{lem: M_converge},~\ref{lem: chi2_M_converge} and~\ref{lem: kl_M_converge}).
\end{proof}

Now by the results in Section~\ref{subsec: asymp_sa}, we have already shown (taking $\pi=\pi^*$):
\begin{align*}
    \sqrt{n}\left(\widehat{V}_r^{\pi^*}(\mu)-V_r^*(\mu)\right)\tod\NM(0,\mu^\top(\M^{\pi^*})^{-1}\Lambda(\M^{\pi^*})^{-\top}\mu).
\end{align*}
Thus, if we could also prove that $\sqrt{n}\left(\widehat{V}_r^*(\mu)-\widehat{V}_r^{\pi^*}(\mu)\right)=o_P(1)$, then our results can be obtained by Slutsky's theorem. Let's denote event $A_n=\{\widehat{\pi}^*=\pi^*\}$. For arbitrary $\delta>0$, we have:
\begin{align*}
    \PB\left(\sqrt{n}\left(\widehat{V}_r^*(\mu)-\widehat{V}_r^{\pi^*}(\mu)\right)>\delta\right)\le\PB\left(\sqrt{n}\left(\widehat{V}_r^{\pi^*}(\mu)-\widehat{V}_r^{\pi^*}(\mu)\right)>\delta, A_n\right)+\PB\left(A_n^c\right)=\PB\left(A_n^c\right).
\end{align*}
Thus, by Lemma~\ref{lem: consistent_pi}, we conclude that $\sqrt{n}\left(\widehat{V}_r^*(\mu)-\widehat{V}_r^{\pi^*}(\mu)\right)=o_P(1)$. 

\subsection{Proof of Theorem~\ref{thm: chi2_s_normal}}
Similar with Lemma~\ref{lem: l1_as}, by Theorem~\ref{thm: chi2-s-union} and the Borel-Cantelli Lemma, we know $\widehat{V}_r^\pi\toas V_r^\pi$. We denote
\begin{align*}
    g(\eta,P,V)=-\sqrt{(\rho+1)|\AM|}\sqrt{\sum_{s\in\SM,a\in\AM}P(s|a)(\eta_a-\pi(a)V(s))_+^2}+\sum_{a\in\AM}\eta_a
\end{align*}
for $P\in\Delta(\SM)^{|\AM|}$ and $V\in\VM$, where we omit the dependence on $\pi$, as it's always fixed. We also denote $\eta^*(P, V)=\argmax_\eta g(\eta, P,V)$ and $f(P, V)=\max_\eta g(\eta, P, V)$. By Danskin's Theorem~\ref{lem: danskin}, we can prove that $f(\widehat{P}, V)$ weakly converges to $f(P, V)$ with rate $\sqrt{n}$ (Lemma~\ref{lem: normal_chi2_s}).
\begin{lem}
    \label{lem: normal_chi2_s}
    Fix $V\in\VM$ and $\pi\in\Delta(\AM)$, and assume $V(s_1)<\cdots <V(s_{|\AM|})$. Then we have:
    \begin{align*}
        \sqrt{n}(f(\widehat P,V)-f(P,V))\tod\NM(0,\sigma^2(P,V)),
    \end{align*}
    where $\sigma^2(P,V)=b^T\Sigma b$ and $\Sigma\in\RB^{|\SM||\AM|\times|\SM||\AM|}$ with the $((s,a), (s',a'))$-th element defined as
    \begin{align*}
        \Sigma((s,a),(s',a'))=\left(-P(s|a)\cdot P(s'|a')+P(s|a)\mathbf{1}\{s=s'\}\right)\mathbf{1}\{a=a'\},
    \end{align*}
    and $b=f'(P)\in\RB^{|\SM||\AM|}$ with the $(s,a)$-th element defined as
    \begin{align*}
        b(s,a)=-\sqrt{(\rho+1)|\AM|}\frac{(\eta_a^*-\pi(a)V(s))_+^2}{2\sqrt{\sum_{s'\in\SM,a'\in\AM}P(s'|a')(\eta^*_{a'}-\pi(a')V(s'))_+^2}}.
    \end{align*}
\end{lem}
\begin{proof}[Proof of Lemma~\ref{lem: normal_chi2_s}]
    Firstly, we note that $g(\eta,P,V)$ is continuously differentiable w.r.t.\ $\eta$. By Lemma~\ref{lem: chi2-value}, there is a unique optimal solution $\eta^*(P,V)$. Besides, $\forall\eta$, $g(\eta,P, V)$ is continuous w.r.t $\eta, P$ and convex in $P$. By Danskin's Theorem~\ref{lem: danskin},  we know that $f(P,V)$ is differentiable at $P$ with derivative
    \begin{align*}
        \frac{\partial f(P,V))}{\partial P(s|a)}=-\sqrt{(\rho+1)|\AM|}\frac{(\eta_a^*(P,V)-\pi(a)V(s))_+^2}{2\sqrt{\sum_{s'\in\SM,a'\in\AM}P(s'|a')(\eta^*_{a'}(P,V)-\pi(a')V(s'))_+^2}}.
    \end{align*}
    Since $\sqrt{n}(\hat{P}-P)\to \NM(0,\Sigma)$ via CLT, we may attain the conclusion by applying the Delta method.
\end{proof}

Thus, we can conclude that $\widehat{\TM}_r^\pi V_r^\pi$ weakly converges to $V_r^\pi$ with rate $\sqrt{n}$ (Corollary~\ref{cor: normal_chi2_s}).
\begin{cor}
    \label{cor: normal_chi2_s}
    Under the $s$-rectangular assumption and the setting $f(t)=(t-1)^2$ in Example~\ref{eg: f-set-s}, for any policy $\pi$, we have:
    \begin{align*}
        \sqrt{n}\left(\widehat{\TM}_r^\pi-I\right)V_r^\pi\tod\NM(0,\Lambda),
    \end{align*}
    where $\Lambda=\diag\{\sigma_1^2, \cdots, \sigma^2_{|\SM|}\}$, and $\forall s\in\SM$,
    \begin{align*}
        \sigma_s^2=\gamma^2\sigma^2(P^*(\cdot|s,\cdot),V_r^\pi).
    \end{align*}
\end{cor}

Similar with the case of $(s,a)$-rectangular set, by Lemma~\ref{lem: chi2_diff_s}, we have:
\begin{align*}
    \widehat{\Phi}_n^\pi V_r^\pi=\widehat{\Phi}_n^\pi V_r^\pi-\widehat{\Phi}_n^\pi \widehat{V}_r^\pi=\M_n^\pi\cdot(V_r^\pi-\widehat{V}_r^\pi)+o_P(V_r^\pi-\widehat{V}_r^\pi),
\end{align*}
where $\M_n^\pi = \widehat{\Phi}'_n(V_r^\pi)$. As we already know that $\widehat{V}_r^\pi\toas V_r^\pi$, next we need to prove that $\widehat{V}_r^\pi$ is $\sqrt{n}$-consistent (Lemma~\ref{lem: consistent}). Thus, as long as $\M_n^\pi$ is consistent (Lemma~\ref{lem: chi2_M_converge_s}), we prove that:
\begin{align*}
    \sqrt{n}\left(\widehat{V}_r^\pi-V_r^\pi\right)\tod\NM\left(0, (\M^\pi)^{-1}\Lambda(\M^\pi)^{-\top}\right).
\end{align*}
\begin{lem}
    \label{lem: chi2_diff_s}
    For a fixed policy $\pi$, $\Phi^\pi$ is continuous differentiable w.r.t.\ $V\in\VM$.
\end{lem}
\begin{proof}[Proof of Lemma~\ref{lem: chi2_diff_s}]
    W.L.O.G., we assume $V(s_1)< \cdots <V(s_{|\SM|})$. Firstly, we know that $g(\eta, P, V)$ has a continuous partial derivative w.r.t.\ $\eta$ at $\eta^*(P, V)$. Thus we have:
    \begin{align*}
        \frac{\partial f(P,V)}{\partial V}
        &=\frac{\partial g(\eta^*(P,V),P,V)}{\partial V}+\frac{\partial g(\eta^*(P,V), P, V)}{\partial \eta}\frac{\partial \eta^*(P,V)}{\partial V}\\
        &=\frac{\partial g(\eta^*(P,V),P,V)}{\partial V},
    \end{align*}
    where $\partial \eta^*(P,V)/\partial V$ exists by implicit function theorem. Thus, for each $s'\in\SM$, we have:
    \begin{align*}
        \frac{\partial f}{\partial V(s')}(P,V)=\sqrt{|\AM|(\rho+1)}\frac{\sum_a\pi(a) P(s'|a)(\eta^*_a(P,V)-\pi(a)V(s'))_+}{\sqrt{\sum_{s,a}P(s|a)(\eta^*_a(P,V)-\pi(a)V(s))_+^2}}.
    \end{align*}
    Next, we prove that $\eta^*(P,V)$ is continuous w.r.t.\ $V$. For every convergent sequence $V_n\to V$, where $V(s_1)<\cdots <V(s_{|\SM|})$, we can choose a large enough $n$ s.t.\ $V_n(s_1)<\cdots <V_n(s_{|\SM|})$. Noting that for every two $V_1, V_2\in\VM$, we have:
    \begin{align*}
        \left|g(\eta, P, V_1)-g(\eta, P, V_2)\right|&\le\sqrt{\sum_{s\in\SM, a\in\AM}P(s|a)\pi(a)^2(V_1(s)-V_2(s))^2}\\
        &=\sqrt{\sum_{a\in\AM}\pi(a)^2\|V_1-V_2\|^2_{2, P_a}}.
    \end{align*}
    Thus, by Lemma~\ref{lem: domain_converge} we have:
    \begin{align*}
        \sup_{\eta\in I_{\chi^2}}\left|g(\eta, P, V_n)-g(\eta, P, V)\right|&\le\sqrt{\sum_{s\in\SM, a\in\AM}P(s|a)\pi(a)^2(V_n(s)-V(s))^2}\\
        &=\sqrt{\sum_{a\in\AM}\pi(a)^2\|V_n-V\|^2_{2, P_a}}\to0,
    \end{align*}
    which implies that $\eta^*(P, V_n)\to \eta^*(P, V)$. Thus, $\eta^*(P, V)$ is continuous w.r.t.\ $V$ and our result is completed.
\end{proof}

\begin{lem}
    \label{lem: chi2_M_converge_s}
    For a fixed policy $\pi$, we have $\M_n^\pi\toas \M^\pi$.
\end{lem}

\begin{proof}[Proof of Lemma~\ref{lem: chi2_M_converge_s}]
    Noting that $\M_n^\pi=\widehat{\Phi}_n'(V_r^\pi)$, for $s,s'\in\SM$, we have:
    \begin{align*}
        \M_n^\pi(s,s^\prime)=\mathbf{1}\{s=s^\prime\}-\gamma\sqrt{|\AM|(\rho+1)}\frac{\sum_a \pi(a|s)\widehat P(s^\prime|s,a)(\eta^*_a(\widehat P(\cdot|s,\cdot),V_r^\pi)-\pi(a|s)V_r^\pi(s^\prime))_+}{\sqrt{\sum_{\tilde s,a}\widehat P(\tilde s|s,a)(\eta^*_a(\widehat P(\cdot|s,\cdot),V_r^\pi)-\pi(a|s)V_r^\pi(\tilde s))_+^2}}.
    \end{align*}
    By SLLN, we know that for any $(s,a,s')\in\SM\times\AM\times\SM$, 
    \begin{align*}
        \widehat{P}(s'|s,a)\toas P^*(s'|s,a).
    \end{align*}
    Next, we note that for any $P_1, P_2\in\Delta(\SM)$, we have:
    \begin{align*}
        \sup_{\eta\in I_{\chi^2}}\left|g(\eta, P_1, V)-g(\eta, P_2, V)\right|&\le \sqrt{\sum_{s\in\SM, a\in\AM}|P_1(s|a)-P_2(s|a)|(\eta_a-\pi(a)V(s))_+^2}\\
        &\le \frac{C(\rho)}{(C(\rho)-1)(1-\gamma)}\sqrt{\sum_{a\in\AM}\|P_1(\cdot|a)-P_2(\cdot|a)\|_1}.
    \end{align*}
    Thus, by Lemma~\ref{lem: domain_converge}, for every convergent sequence $P_n\to P$, we have $\eta^*(P_n, V)\to \eta^*(P, V)$, which implies that $\eta^*(P, V)$ is continuous w.r.t. $P$. Thus, we conclude that $\eta^*(\widehat{P}(\cdot|s,\cdot), V_r^\pi)\toas \eta^*(P(\cdot|s,\cdot), V_r^\pi)$, which leads to $\M_n^\pi\toas \M^\pi$.
\end{proof}

\subsection{Proof of Theorem~\ref{thm: kl_s_normal}}
Firstly, we still have $\widehat{V}_r^\pi\toas V_r^\pi$ in KL case. And, for any fixed policy $\pi$, we denote 
\begin{align*}
    g(\lambda, P, V)=-\lambda|\AM|\rho - \lambda\sum_{a\in\AM}\log\sum_{s\in\SM}P(s|a)\exp\left(-\frac{\pi(a)V(s)}{\lambda}\right)
\end{align*}
for $P\in\Delta(\SM)^{|\AM|}$ and $V\in\VM$. We denote $\lambda^*(P,V)=\argmax_{\lambda\ge0}g(\lambda, P, V)$ and $f(P,V)=\max_{\lambda\ge0}g(\lambda, P, V)$. By Danskin's Theorem (Lemma~\ref{lem: danskin}), we can prove that $f(\widehat{P}, V)$ weakly converges to $f(P,V)$ with rate $\sqrt{n}$ (Lemma~\ref{lem: normal_kl_s}).

\begin{lem}
    \label{lem: normal_kl_s}
    Fix $V\in\VM$, and assume $V(s_1)<\cdots <V(s_{|\SM|})$. Then we have:
    \begin{align*}
        \sqrt{n}\left(f(\widehat{P}, V)-f(P, V)\right)\tod \NM(0, \sigma^2(P,V)),
    \end{align*}
    where $\sigma^2(P, V)=b^\top \Sigma b$ and $\Sigma\in\RB^{|\SM||\AM|\times|\SM||\AM|}$ with the $((s, a), (s', a'))$-th element defined as
    \begin{align*}
        \Sigma((s,a),(s',a'))=\left(-P(s|a)\cdot P(s'|a')+P(s|a)\mathbf{1}\{s=s'\}\right)\mathbf{1}\{a=a'\},
    \end{align*}
    and $b=f'(P)\in\RB^{|\SM||\AM|}$ with the $(s,a)$-th element defined as
    \begin{align*}
        b(s,a)=-\frac{\lambda^*(P,V)\exp\left(-\frac{\pi(a)V(s)}{\lambda^*(P,V)}\right)}{\sum_{s'}P(s'|a)\exp\left(-\frac{\pi(a)V(s')}{\lambda^*(P,V)}\right)}.
    \end{align*}
\end{lem}
\begin{proof}[Proof of Lemma~\ref{lem: normal_kl_sa}]
    Note that $\lambda^*(P, V)\in[0,\frac{1}{|\AM|\rho(1-\gamma)}]$. Furthermore, we note that $\forall \lambda$, $g(\lambda,P,V)$ is continuous w.r.t.\ $\lambda,P$, and convex in $P$. Thus By Danskin's theorem (Lemma~\ref{lem: danskin}), we know that $f(P,V)$ is differentiable at $P$ with derivative
    \begin{align*}
        \frac{\partial f(P,V)}{\partial P(s|a)}=-\frac{\lambda^*(P,V)\exp(-\frac{\pi(a)V(s)}{\lambda^*(P,V)})}{\sum_{s'\in\SM}P(s'|a)\exp(-\frac{\pi(a)V(s')}{\lambda^*(P,V)})}.
    \end{align*}
    By the fact that $\sqrt{n}\left(\widehat{P}-P\right)\tod \NM(0, \Sigma)$ via CLT, our result is concluded by the Delta method.
\end{proof}

Thus, we can conclude that $\widehat{\TM}_r^\pi V_r^\pi$ weakly converges to $V_r^\pi$ with rate $\sqrt{n}$ (Corollary~\ref{cor: normal_kl_s}).

\begin{cor}
    \label{cor: normal_kl_s}
    Under the $s$-rectangular assumption and the setting $f(t)=t\log t$ in Example~\ref{eg: f-set-s}, for any policy $\pi$, we have:
    \begin{align*}
        \sqrt{n}\left(\widehat{\TM}_r^\pi - I\right)V_r^\pi\tod\NM\left(0, \Lambda\right),
    \end{align*}
    where $\Lambda=\diag\{\sigma_1^2, \cdots, \sigma^2_{|\SM|}\}$, and $\forall s\in\SM$,
    \begin{align*}
        \sigma_s^2=\gamma^2\sigma^2(P^*(\cdot|s,\cdot),V_r^\pi).
    \end{align*}
\end{cor}

Similar with the case of $(s,a)$-rectangular set, by Lemma~\ref{lem: kl_diff_s}, we have:
\begin{align*}
    \widehat{\Phi}_n^\pi V_r^\pi=\widehat{\Phi}_n^\pi V_r^\pi-\widehat{\Phi}_n^\pi \widehat{V}_r^\pi=\M_n^\pi\cdot(V_r^\pi-\widehat{V}_r^\pi)+o_P(V_r^\pi-\widehat{V}_r^\pi),
\end{align*}
where $\M_n^\pi = \widehat{\Phi}'_n(V_r^\pi)$. As we have already known that $\widehat{V}_r^\pi\toas V_r^\pi$, next we need to prove that $\widehat{V}_r^\pi$ is $\sqrt{n}$-consistent (Lemma~\ref{lem: consistent}). Thus, as long as $\M_n^\pi$ is consistent (Lemma~\ref{lem: kl_M_converge_s}), we prove that:
\begin{align*}
    \sqrt{n}\left(\widehat{V}_r^\pi-V_r^\pi\right)\tod\NM\left(0, (\M^\pi)^{-1}\Lambda(\M^\pi)^{-\top}\right).
\end{align*}
\begin{lem}
    \label{lem: kl_diff_s}
    For a fixed policy $\pi$, $\Phi^\pi$ is continuous differentiable w.r.t.\ $V\in\VM$.
\end{lem}
\begin{proof}[Proof of Lemma~\ref{lem: kl_diff_s}]
    W.L.G., we assume $V(s_1)<\cdots <V(s_{|\SM|})$. Firstly, we know that $g(\lambda, P, V)$ has a continuous partial derivative w.r.t.\ $\lambda$ at $\lambda^*(P, V)$. Thus we have:
    \begin{align*}
        \frac{\partial f(P,V)}{\partial V}
        &=\frac{\partial g(\lambda^*(P,V),P,V)}{\partial V}+\frac{\partial g(\lambda^*(P,V),P,V)}{\partial \lambda}\frac{\partial \lambda^*(P,V)}{\partial V}\\
        &=\frac{\partial g(\lambda^*(P,V),P,V)}{\partial V},
    \end{align*}
    where $\partial \lambda^*(P,V)/\partial V$ exists by the implicit function theorem. Thus, for each $s\in\SM$, we have:
    \begin{align*}
        \frac{\partial f}{\partial V(s)}(P,V)=\sum_a\frac{P(s|a)\pi(a)\exp(-\pi(a)V(s)/\lambda^*(P,V))}{\sum_{s'}P(s'|a)\exp(-\pi(a)V(s')/\lambda^*(P,V))}.
    \end{align*}
    Next, we prove that $\lambda^*(P,V)$ is continuous w.r.t.\ $V$. Note that $\|\nabla g(\lambda, P, V)\|_1=1$. Thus, the Lipschitz coefficient of $g(\lambda, P, V)$ w.r.t.\ $V$ in $\RB^{|\SM|}$ is bounded by 1. By Lemma~\ref{lem: domain_converge}, our result is obtained.
\end{proof}

\begin{lem}
    \label{lem: kl_M_converge_s}
    For a fixed policy $\pi$, we have $\M_n^\pi\toas \M^\pi$.
\end{lem}
\begin{proof}[Proof of Lemma~\ref{lem: kl_M_converge_s}]
    Noting that $\M_n^\pi=\widehat{\Phi}_n'(V_r^\pi)$ for $s,s'\in\SM$, we have:
    \begin{align*}
        \M_n^\pi(s,s^\prime)=\mathbf{1}\{s=s^\prime\}-\gamma\sum_a \frac{\pi(a|s)\widehat P(s^\prime|s,a)\exp(-\frac{\pi(a|s)V_r^\pi(s^\prime)}{\lambda^*(\widehat P(\cdot|s,\cdot),V^\pi_r)})}{\sum_{\tilde s\in\SM}\widehat P(\tilde s|s,a)\exp(-\frac{\pi(a|s)V_r^\pi(\tilde s)}{\lambda^*(\widehat P(\cdot|s,\cdot),V_r^\pi)})}.
    \end{align*}
    By SLLN, we know that for any $(s,a,s')\in\SM\times\AM\times\SM$, 
    \begin{align*}
        \widehat{P}(s'|s,a)\toas P^*(s'|s,a).
    \end{align*}
    Next, we note that when $n$ is large enough s.t.\ $\max_{s,a, P_a(s)\not=0}\left|\frac{\widehat{P}_a(s)}{P_a(s)}-1\right|\le 1/2$, we have (details can be referred in Theorem~\ref{thm: kl-s-fix}):
    \begin{align*}
        \sup_{\lambda\in[0,\frac{1}{\rho(1-\gamma)}]}\left|g(\lambda, \widehat{P}, V)-g(\lambda, P, V)\right|\le \frac{2}{\rho(1-\gamma)}\max_{s,a,P_a(s)\not=0}\left|\frac{\widehat{P}_a(s)}{P_a(s)}-1\right|.
    \end{align*}
    Thus, by Lemma~\ref{lem: domain_converge}, as $\widehat{P}\toas P$, we have $\lambda^*(\widehat{P}, V)\to \lambda^*(P, V)$. Thus, we conclude that $\M_n^\pi\toas \M^\pi$.
\end{proof}

\subsection{Proof of Theorem~\ref{thm: sup_pi_consistent}}
In fact, the proof of this theorem is exactly Argmax Theorem in Section 3.2.1 of \citet{van1996weak}. Here we give a brief proof.

Firstly, by definition, we have $\widehat{V}_r^*\ge \widehat{V}_r^{\pi^*}$. 

By the results in Section~\ref{sec: gen}, we have $\sup_{\pi}\left\|\widehat{V}_r^\pi - V_r^\pi\right\|\overset{P}{\to}0$ in the $s$-rectangular setting. Thus, $\widehat{V}_r^{\pi^*}\overset{P}{\to}V_r^{\pi^*}=V_r^*$, which leads to $\widehat{V}_r^*\ge V_r^*-o_P(1)$. It follows that:
\begin{align*}
    V_r^* - V_r^{\widehat{\pi}^*}&\le\widehat{V}_r^* - V_r^{\widehat{\pi}^*} +o_P(1)\notag\\
    &\le \sup_\pi \left\|\widehat{V}_r^\pi - V_r^\pi\right\| +o_P(1)=o_P(1).
\end{align*}

Note that we assume $\pi^*\in\argmax_\pi V_r^\pi$ is unique. Thus, for every $\delta>0$ and $\pi$ satisfying $\|\pi-\pi^*\|_1\ge\delta$, there exists $\varepsilon(\delta)>0$ such that $V_r^\pi(s)<V_r^*(s)-\varepsilon(\delta)$ for all $s\in\SM$. It follows that:
\begin{align*}
    \PB\left(\|\widehat{\pi}^*-\pi^*\|_1\ge\delta\right)\le\PB\left(\left\|V_r^*-V_r^{\widehat{\pi}^*}\right\|_\infty\ge\varepsilon(\delta)\right)\to 0,
\end{align*}
which concludes $\widehat{\pi}^*\overset{P}{\to}\pi^*$. Furthermore, noting that $\sum_n \PB\left(\left\|V_r^*-V_r^{\widehat{\pi}^*}\right\|_\infty\ge\varepsilon(\delta)\right)<\infty$, we conclude that $\widehat{\pi}^*\toas\pi^*$ by the Borel-Cantelli lemma.

\subsection{Proof of Corollary~\ref{cor: sup_asymp_s}}
The proof of Corollary~\ref{cor: sup_asymp_s} is similar with Theorems~\ref{thm: chi2_s_normal} and~\ref{thm: kl_s_normal}. We only give detail proofs in the case of the $\chi^2$ uncertainty set. For the case of the KL uncertainty set, the derivation procedure is the same.

Firstly, for any $V\in\VM$, we denote $\Phi V = (\TM_r - I)V$ and $\widehat{\Phi}_n V = (\widehat{\TM}_r - I)V$, where $\TM_r V = \max_\pi \TM_r^\pi V$ and $\widehat{\TM}_r V = \max_\pi \widehat{\TM}_r^\pi V$. We also denote $h(\pi_s,\eta,P_s, V, R_s)= R^{\pi_s}(s)+\gamma g(\eta, P_s, \pi_s, V)$, where $g$ is defined in Proof of Theorem~\ref{thm: chi2_s_normal}. In the next part, we ignore the dependence of $s$, and denote $f(P, V, R):=\max_{\pi,\eta}h(\pi, \eta, P, V, R)$ and $(\pi^*, \eta^*)\in\argmax_{\pi, \eta} h(\pi, \eta, P, V, R)$. 

\begin{lem}
    \label{lem: asymp_opt_1}
    For any fixed $V\in\VM$ and assuming $\pi^*$ and $\eta^*$ are unique,  we have:
    \begin{align*}
        \sqrt{n}\left(f(\widehat{P}, V, R)-f(P, V, R)\right)\tod\NM(0,\sigma^2(P, V)),
    \end{align*}
    where $\sigma^2(P,V)=b^T\Sigma b$ and $\Sigma\in\RB^{|\SM||\AM|\times|\SM||\AM|}$ with the $((s,a),(s',a'))$-th element defined as
    \begin{align*}
        \Sigma((s,a),(s',a'))=\left(-P(s|a)\cdot P(s'|a')+P(s|a)\mathbf{1}\{s=s'\}\right)\mathbf{1}\{a=a'\},
    \end{align*}
    and $b=f'(P)\in\RB^{|\SM||\AM|}$ with the $(s,a)$-th element defined as
    \begin{align*}
        b(s,a)=-\sqrt{(\rho+1)|\AM|}\frac{(\eta_a^*-\pi^*(a)V(s))_+^2}{2\sqrt{\sum_{s'\in\SM,a'\in\AM}P(s'|a')(\eta^*_{a'}-\pi^*(a')V(s'))_+^2}}.
    \end{align*}
\end{lem}

\begin{proof}[Proof of Lemma~\ref{lem: asymp_opt_1}]
    Similar with Theorem~\ref{thm: chi2_s_normal}, we have $\sqrt{n}(\widehat{P}-P)\tod\NM(0,\Sigma)$ via CLT. By Danskin's Theorem~\ref{lem: danskin}, $f(P, V)$ is also differentiable at $P$. Thus, our result is obtained by the Delta method.
\end{proof}

By Lemma~\ref{lem: asymp_opt_1}, we take $V=V_r^*$, which is a fixed point of operator $\TM_r$, and obtain that $\sqrt{n}\left(\widehat{\TM}_r V_r^* - V_r^*\right)$ is asymptotical normal. Notably, if we also have 
\begin{align*}
    \widehat{\Phi}_n V_r^* = \widehat{\Phi}_n V_r^* - \widehat{\Phi}_n \widehat{V}_r^* = -\widehat{\M}_n\cdot(V_r^* - \widehat{V}_r^*) + o_P(V_r^* - \widehat{V}_r^*)
\end{align*}
and $\widehat{\M}_n$ is consistent, we conclude that $\sqrt{n}\left(\widehat{V}_r^* - V_r^*\right)$ is asymptotical normal. In fact, we can choose $\widehat{\M}_n = \widehat{\Phi}_n'(V_r^*)$ by Lemma~\ref{lem: asymp_opt_M_diff}, which is also consistent according to Lemma~\ref{lem: asymp_opt_M_cons}. Thus, our proof of Corollary~\ref{cor: sup_asymp_s} is completed.

\begin{lem}
    \label{lem: asymp_opt_M_diff}
    $\Phi:=\TM_r-I$ is differentiable w.r.t.\ $V\in\VM$.
\end{lem}

\begin{proof}[Proof of Lemma~\ref{lem: asymp_opt_M_diff}]
    For any $V\in\VM$, we have 
    \begin{align*}
        \TM_r V(s)=\max_{\pi_s\in\Delta(\AM), \eta}h(\pi_s, \eta, P_s, V, R_s):=f(P_s, V, R_s).
    \end{align*}
    Noting that $h$ is differentiable w.r.t $V$, we obtain our result by the envelop theorem:
    \begin{align*}
        \frac{\partial f(P_s, V, R_s)}{\partial V} = \frac{\partial h(\pi^*_s, \eta^*, P_s, V, R_s)}{\partial V}.
    \end{align*}
\end{proof}

\begin{lem}
    \label{lem: asymp_opt_M_cons}
    $\widehat{\M}_n$ is consistent.
\end{lem}
\begin{proof}[Proof of Lemma~\ref{lem: asymp_opt_M_cons}]
    It is obtained by a similar proof of Lemma~\ref{lem: chi2_M_converge_s}.
\end{proof}

%% file: tex/apdix/exp.tex
\subsection{Environments}
The environments we use in numerical experiments are random MDPs. With given parameters $|\SM|$, $|\AM|$, $\gamma$, reward function $R$ and transition probability $P^*$ are generated by following mechanism. For each $(s,a)$ pair, $R(s,a)\sim \UM(0,1)$, where $\UM(0,1)$ stands for uniform distribution on interval $(0,1)$, and $P^*(s'|s,a)=\frac{u(s')}{\sum_s u(s)}$, where $u(s)\overset{i.i.d.}{\sim} \UM(0,1)$. Under the $(s,a)$-rectangular assumption, we choose $|\SM|=20$, $|\AM|=10$ and $\gamma=0.9$. Under the $s$-rectangular assumption, we choose $|\SM|=10$, $|\AM|=5$ and $\gamma=0.9$.

\subsection{Algorithms}

Under the $(s,a)$-rectangular assumption, the uncertainty set can be separated by each $(s,a)$ pair, which enables one step of robust Bellman update can be solved at each $(s,a)$ pair. Thus, \citet{iyengar2005robust} proposed Algorithm~\ref{alg: rvi} to obtain the near optimal robust value function and policy, which is similar with Value Iteration algorithm in solving classic MDPs \citep{puterman2014markov}. 

\begin{algorithm}[b!]
    \caption{Robust Value Iteration (RVI) with $(s,a)$-rectangular assumption}
    \label{alg: rvi}
    \begin{algorithmic}
        \STATE {\bfseries Input:} $T$---the number of iterations; $n$---the number of generated sample at each $(s, a)\in\SM\times\AM$; $\rho$---the size of uncertainty set; $f$---choice of uncertainty set; $V_0$---initialized value function.
        \FOR{each $(s,a)\in\SM\times\AM$}
        \STATE Generate i.i.d.\ samples $\{s'_i\}_{i=1}^n$, where $s'\sim P^*(\cdot|s,a)$, and estimate $\widehat{P}(s'|s,a)=\frac{1}{n}\sum_{i=1}^n \mathbf{1}\{s'_i=s'\}$.
        \ENDFOR
        \FOR{iteration $t=0$ {\bfseries to} $T-1$}
        \STATE For each $(s,a)\in\SM\times\AM$, compute
        \begin{align*}
            Q_{t+1}(s,a) = r(s,a) + \gamma \inf_{D_f(P\|\widehat{P}_{s,a})\le\rho}\sum_{s'\in\SM}P(s')V_t(s').
        \end{align*}
        \STATE Then, let $V_{t+1}(s)=\max_{a} Q_{t+1}(s,a)$
        \ENDFOR
        \STATE {\bfseries Return:} $V_{T}$ and $\pi_T = \argmax_{a}Q_{T}(s,a)$
    \end{algorithmic}
\end{algorithm}
However, when the $s$-rectangular is assumed, the uncertainty set cannot be separated at each $(s,a)$ pair and solving the primal problem of $\widehat{\TM}_r V$ (given $V$) could be difficult. \citet{ho2018fast} proposed an efficient algorithm to obtain the near optimal robust value function called Bisection algorithm, which is detailed in Algorithm~\ref{alg: bisec}. Here we give a brief description of Bisection algorithm. For each state $s\in\SM$, the following optimization problem is equivalent to $\widehat{\TM}_r V (s)$ (given $V$):
\begin{align*}
    &\min_{u\in\RB} u, \\
    \text{s.t.} &\sum_{a}q_{s,a}^{-1}(u, V)\le|\AM|\rho,
\end{align*}
where $q_{s,a}^{-1}(u, V)$ is defined by another optimization problem:
\begin{align*}
    q_{s,a}^{-1}(u, V):=\min_{P\in\Delta(\SM)\atop R(s,a)+\gamma P^\top V\le u} D_f(P\|\widehat{P}_{s,a}).
\end{align*}
To get the near optimal value of this problem $\widehat{\TM}_r V$, a bisection method is applied in \cite{ho2018fast}. Finally, recursively applying the bisection method, the near optimal robust value function can be obtained by the fact that $\widehat{\TM}_r$ is a $\gamma$-contraction operator. Notably, we could also obtain a near optimal robust policy from the bisection algorithm, and we refer to readers as \citet{ho2018fast} directly.
\begin{algorithm}[htbp!]
    \caption{Bisection algorithm with the $s$-rectangular assumption}
    \label{alg: bisec}
    \begin{algorithmic}
        \STATE {\bfseries Input:} $T$--the number of iterations; $n$--the number of generated sample at each $(s,a)\in\SM\times\AM$; $\rho$--the size of uncertainty set; $f$--choice of uncertainty set; $V_0$--initialized value function; $\varepsilon$--desired precision.
        \FOR{each $(s,a)\in\SM\times\AM$}
        \STATE Generate i.i.d.\ samples $\{s'_i\}_{i=1}^n$, where $s'\sim P^*(\cdot|s,a)$, and estimate $\widehat{P}(s'|s,a)=\frac{1}{n}\sum_{i=1}^n \mathbf{1}\{s'_i=s'\}$.
        \ENDFOR
        \FOR{iteration $t=0$ {\bfseries to} $T-1$}
        \FOR{each $s\in\SM$}
        \STATE Based on $V_t$, set $u_{\min}(s)$: minimum known $u$, and $u_{\max}(s)$: maximum known $u$
        \WHILE{$u_{\max}(s)-u_{\min}(s)>2\varepsilon$}
        \STATE Let $u = \frac{u_{\max}(s)+u_{\min}(s)}{2}$, and compute $m = \sum_{a\in\AM} q_a^{-1}(u, V_t)$;
        \STATE If $m\le|\AM|\rho$, let $u_{\max}(s)=u$ if $u$ is feasible; else, let $u_{\min}(s)=u$ if $u$ is infeasible.
        \ENDWHILE
        \STATE Let $V_{t+1}(s) = \frac{u_{\max}(s)+u_{\min}(s)}{2}$;
        \ENDFOR
        \ENDFOR
        \STATE {\bfseries Return:} $V_{T}$
    \end{algorithmic}
\end{algorithm}